\documentclass[Afour,sageh,times]{sagej}

\usepackage{moreverb,url}
\usepackage{bm}
\usepackage{stfloats}

\usepackage[colorlinks,bookmarksopen,bookmarksnumbered,citecolor=red,urlcolor=red]{hyperref}

\usepackage{algorithm}
\usepackage{algpseudocode}

\newtheorem{theorem}{Theorem}
\newtheorem{discussion}{Discussion}
\newtheorem{corollary}{Corollary}
\usepackage{pdflscape}
\newtheorem{question}{Question}
\newtheorem{assumption}{Assumption}
\newtheorem{lemma}{Lemma}
\usepackage{graphicx}
\usepackage{multirow}
\usepackage{subcaption}
\usepackage{centernot} 

\usepackage[most]{tcolorbox}
\usepackage{xcolor}
\usepackage{setspace}
\usepackage{geometry}
\usepackage{listings}
\usepackage{lmodern}
\usepackage{enumitem}
\definecolor{mygray}{gray}{0.6}
\definecolor{mypink}{rgb}{1.0, 0.75, 0.8}

\newcommand\BibTeX{{\rmfamily B\kern-.05em \textsc{i\kern-.025em b}\kern-.08em
T\kern-.1667em\lower.7ex\hbox{E}\kern-.125emX}}

\def\volumeyear{2016}

\begin{document}

\runninghead{Smith and Wittkopf}

\runninghead{Yongbo Chen et al.}

\title{POMDP-based object search with growing state space and hybrid action domain}

\author{Yongbo Chen\affilnum{1,~2,~3}, Hesheng Wang\affilnum{1,~3}, Shoudong Huang\affilnum{4}, and Hanna Kurniawati\affilnum{2}}

\affiliation{\affilnum{1}School of Automation and Intelligent Sensing, Shanghai Jiao Tong University, Shanghai, 200240, People's Republic of China\\ \affilnum{2}School of Computing, Australian National University (ANU), Canberra, ACT, 2601, Australia (The main part of this work is completed at ANU)\\ \affilnum{3}Key Laboratory of System Control and Information Processing, Ministry of Education of China, Shanghai, 200240, People's Republic of China\\ \affilnum{4}Robotics Institute, University of Technology Sydney, Australia}

\corrauth{Yongbo Chen, School of Automation and Intelligent Sensing, Shanghai Jiao Tong University, Shanghai, 200240, People's Republic of China}

\email{shjtdx\_cyb@sjtu.edu.cn, chengyongbobobo@gmail.com}

\begin{abstract}
Efficiently locating target objects in complex indoor environments with diverse furniture, such as shelves, tables, and beds, is a significant challenge for mobile robots. This difficulty arises from factors like localization errors, limited fields of view, and visual occlusion. We address this by framing the object-search task as a high-dimensional Partially Observable Markov Decision Process (POMDP) with a growing state space and hybrid (continuous and discrete) action spaces in 3D environments. Based on a meticulously designed perception module, a novel online POMDP solver named the growing neural process filtered $k$-center clustering tree (GNPF-$k$CT) is proposed to tackle this problem. Optimal actions are selected using Monte Carlo Tree Search (MCTS) with belief tree reuse for growing state space, a neural process network to filter useless primitive actions, and $k$-center clustering hypersphere discretization for efficient refinement of high-dimensional action spaces. A modified upper-confidence bound (UCB), informed by belief differences and action value functions within cells of estimated diameters, guides MCTS expansion. Theoretical analysis validates the convergence and performance potential of our method. To address scenarios with limited information or rewards, we also introduce a guessed target object with a grid-world model as a key strategy to enhance search efficiency. Extensive Gazebo simulations with Fetch and Stretch robots demonstrate faster and more reliable target localization than POMDP-based baselines and state-of-the-art (SOTA) non-POMDP-based solvers, especially large language model (LLM) based methods, in object search under the same computational constraints and perception systems. Real-world tests in office environments confirm the practical applicability of our approach. Project page: \url{https://sites.google.com/view/gnpfkct}.
\end{abstract}


\maketitle

\section{Introduction}
\label{sec:intro}
 
Object search in partially familiar, cluttered environments is crucial for mobile robots. For example, a robot retrieving a customer's order from a warehouse or fetching a Coke bottle from a kitchen must often locate a target object that is partially or fully obscured by other items. Known as the `Mechanical Search' problem~\cite{R1}, it typically involves leveraging partial environmental knowledge, like maps of large furniture (e.g., shelves, tables), while smaller objects (e.g., mugs, glasses) remain unlocalized due to frequent movement. Humans typically find smaller objects by relying on prior knowledge, assuming familiarity with the target, and removing obstructions as needed~\cite{aydemir2013active}. Similarly, robots must adjust their configurations, strategically observe from multiple viewpoints, and interact with the environment by removing obstructions and retrieving the target object.  This paper focuses on object search in a multi-room home environment with predefined workspaces. Equipped with 3D point cloud and 2D occupancy grid maps and provided with target photos, the robot aims to efficiently locate objects using onboard sensors, as shown in Fig.~\ref{f3}.

\begin{figure}[!th]
	\centering
	\includegraphics[width=\linewidth]{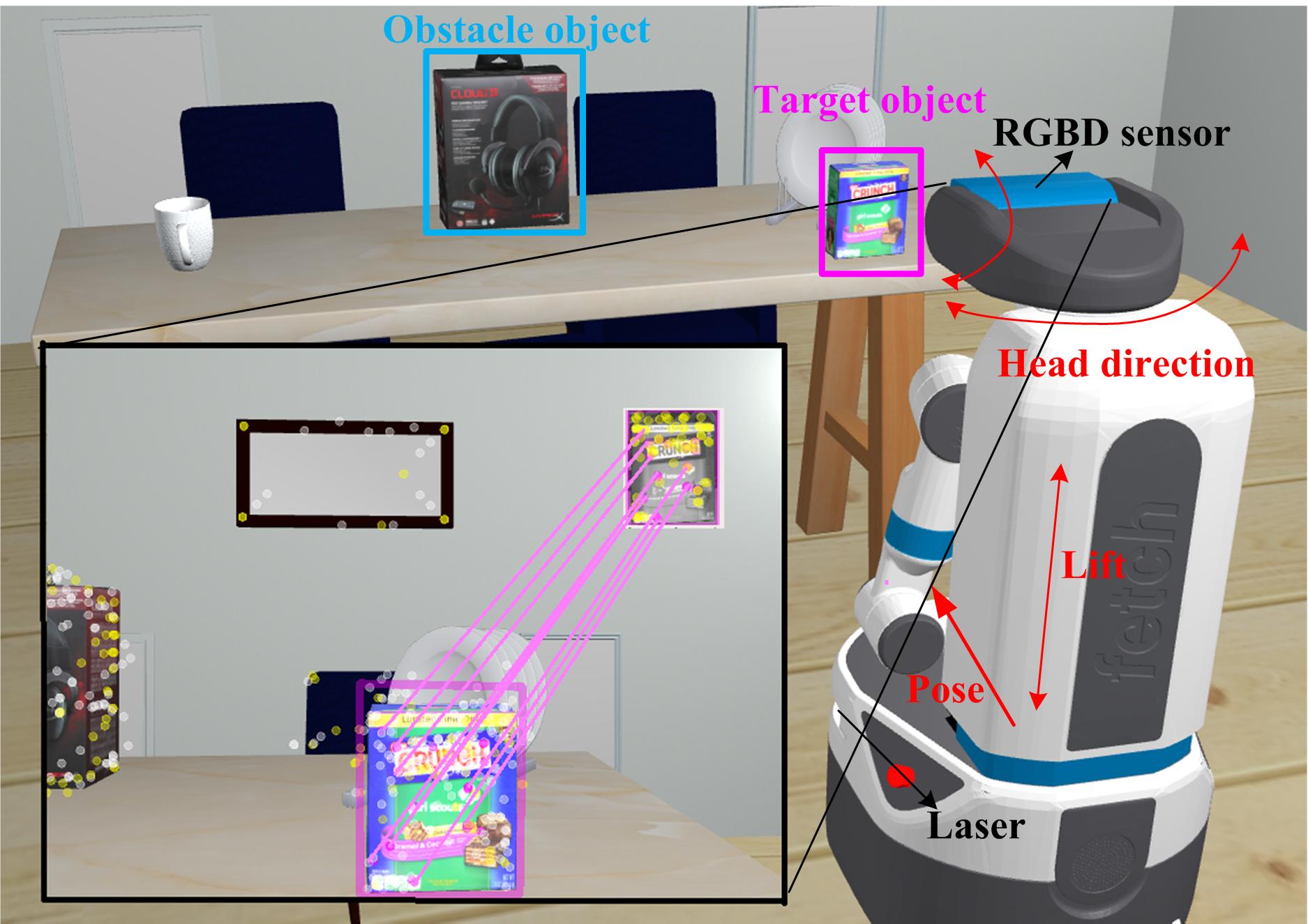}
	\caption{A scenario of object search involves a Fetch robot locating and removing a pink snack box from a table using feature matching. The robot actively adjusts its base, lift, and head (red) to change its view, with no prior knowledge of the size, number, or color of objects on the table.} 
	\label{f3}
\end{figure}

The object search method depends on a comprehensive support system that includes modules for object segmentation, object detection, pose estimation, task-level planning, and action-level planning. Significant advancements have been made in these supporting areas, particularly through advanced learning technologies. For example, an impressive approach presented in~\cite{R3} showcases an efficient one-shot learning method utilizing a fully convolutional network for pixel-level image segmentation. Similarly, other methods leverage architectures like neural radiance fields and convolutional neural networks (CNNs) to address challenges in optimal grasp pose generation~\cite{R4} and manipulation planning~\cite{R5}.  In this study, we focus on task-level planning, aiming to select primitive actions and achieve long-term sequential goals for locating and retrieving target objects in complex, unlabeled environments.

In this paper, we propose a partially observable Markov decision process (POMDP) formulation with a hybrid action domain for object search, integrating well-defined perception and action modules. To solve this formulation, we introduce a novel solver that leverages belief tree reuse, neural process filtering, and $k$-center clustering. The main contributions of this work are as follows:

\begin{itemize}
    \item{\textbf{Realistic formulation}: A complete POMDP formulation for object search task considering a realistic perception process.}
	\item{\textbf{Better exploration trick}: Introduction of the guessed target object with odds update for better exploration.}
	\item{\textbf{Novel efficient POMDP solver}: A novel solver for POMDP problem with growing state space, neural process filtering, and hybrid (continuous and discrete) action domain.}
	\item{\textbf{Demo}: Realistic gazebo simulations and real-world demos using two robots, showcasing the proposed framework with publicly available open-source code.}
\end{itemize}

This paper extends our previous work~\cite{chen2024pomdp} and an on-line document~\cite{anonymous2024npfkct}. The paper is organized as follows: Section~\ref{s2} reviews related work on POMDP models, existing solvers, and mechanical search. Section~\ref{s3} outlines our problem formulation and solvers for the POMDP-based object search task. Section~\ref{s4} details the POMDP formulation, including state, action, and observation spaces, and key functions. 
Section~\ref{s6} introduces a novel online POMDP solver with theoretical discussions. In Section~\ref{s5}, we explain our action execution and perception parts using the Fetch robot with onboard sensors to support our POMDP formulation. Section~\ref{s7} provides simulation and experimental results validating the practicality and high performance of the proposed method. Finally, we summarize our conclusions and present our future work in Section~\ref{s8}.

\section{Related work}
\label{s2}
\subsection{Mechanical search: Planning}
Robots frequently encounter challenges navigating cluttered environments to locate target objects. Many contemporary approaches leverage advancements in learning technologies, including deep reinforcement learning (RL)~\cite{kurenkov2020visuomotor}, deep Q-learning~\cite{yang2020deep}, and deep-geometric inference systems~\cite{huang2021mechanical}. For instance, ~\cite{kurenkov2020visuomotor} proposes an innovative learning procedure that combines an asymmetric architecture to exploit privileged information, guidance from suboptimal teachers, and mid-level representations for training deep RL agents to discover target objects hidden within piles of unknown items. However, these methods are tailored to specific environments, such as shelves or boxes, limiting their applicability in diverse, unstructured settings.

Beyond learning-based methods, the POMDP framework offers a compact and effective approach for object search. As an example, in \cite{huang2022mechanical}, a Monte-Carlo tree search method featuring a 1D occupancy distribution for objects is employed to swiftly identify and extract the target object. Due to its natural compatibility with scenarios involving multiple detected objects, the object-oriented POMDP (OO-POMDP) framework is widely adopted. This framework factors objects into classes with associated attributes, treating the beliefs for different objects as independent, which enables linear scaling of belief size and reduces computational costs~\cite{wandzel2019multi}. To simplify computation, prior work often restricts the workspace to a 2D plane with simple transition and observation models. For example, \cite{aydemir2013active} proposed a 2D POMDP formulation for object search in a room by evaluating candidate viewpoints. Recent efforts have extended this to 3D environments with object-size-friendly properties using multi-resolution planning algorithms~\cite{zheng2021multi}. The commonly used POMDP solver in these frameworks is partially observable Monte Carlo planning (POMCP)~\cite{silver2010monte}. We build on the OO-POMDP framework for real-world applicability, but go forward one step to consider the belief relationship between different objects in a complicated 3D environment including some challenging cases with partially and fully occluded objects. We also formulate the object-level POMDP formulation with a growing state space, a hybrid action domain, and a guessed target object.  

\subsection{Mechanical search: Reasoning} Earlier robotics approaches integrated reasoning directly into the planning process, such as updating probabilities, rather than treating it as a distinct module. As robotic systems became more sophisticated, reasoning evolved into a separate layer, addressing high-level cognitive tasks like commonsense inference, contextual understanding, and hypothesis generation. This reasoning layer then informs a planning module to execute detailed action strategies.  For instance, ~\cite{giuliari2023leveraging} demonstrates a reasoning module that uses environmental context and commonsense knowledge to infer plausible object locations, aiding localization in partially observed scenes. Similarly, ~\cite{ge2024commonsense} employs commonsense knowledge from large language models to construct scene graphs for improving object search.  However, these methods struggle with unconventional object arrangements, such as randomized setups. In our scenario, objects are deliberately placed in configurations that defy typical human habits, increasing search complexity and requiring active robot-environment interactions, such as occlusion removal. Our reasoning idea is based on the probability update of the grid world. By representing the belief over the pose state of each object in the planning environment using particle filtering, the authors in \cite{garrett2020online} incorporate probabilistic reasoning into a deterministic planner and then perform replanning when facing the base movement failure. Recent work~\cite{huang2024out} utilizes a video tracking-based memory model with integrated reasoning and planning capabilities. This allows the system to remember potential locations of occluded target objects and complete tasks using appropriate actions.

\subsection{POMDP formulation and solvers}
For a POMDP agent, the process commences with a brief, where the agent deduces and executes the optimal action to transition to a new state. This new state is concealed within an updated belief, refined through Bayesian inference based on observed data. Each step yields an immediate reward and contributes to a discounted cumulative reward, fostering a long-term objective. This iterative process continues until terminal conditions are met.  The POMDP framework is a powerful tool for planning under uncertainty, excelling in non-deterministic and partially observable scenarios. Widely applied in robotics, it addresses tasks like autonomous underwater vehicle navigation~\cite{Hou_belief}, robot manipulation~\cite{pajarinen2022pomdp}, and Human-Robot collaboration~\cite{burks2023harps}.  Our work adopts this approach, developing an object-oriented POMDP model with hybrid action domains.

Finding the exact optimal strategy for a POMDP problem is widely recognized as computationally intractable~\cite{papadimitriou1987complexity}. Over the past two decades, POMDP solvers have been categorized into offline and online methods. Offline solvers, such as point-based POMDP methods~\cite{Shani_2013, Kurniawati_2008}, can produce optimal solutions for small-scale problems. However, their significant memory requirements often make them impractical for use on standard computational platforms in real-world applications. In contrast, online solvers, especially sampling-based methods, offer a practical trade-off by achieving approximate optimality with greater computational efficiency, making them suitable for various applications.

Sampling-based solvers, such as POMCP~\cite{silver2010monte}, adaptive belief tree (ABT)~\cite{kurniawati2016online}, and determinized sparse partially observable tree (DESPOT)~\cite{somani2013despot}, share a common strategy: representing belief as particles and utilizing Monte Carlo tree search (MCTS) to expand the belief tree within computational limits. Despite significant advancements, solving POMDPs with high-dimensional continuous action spaces remains a substantial challenge. The key point in existing continuous-action POMDPs is to refine the action subset incrementally to improve the possibility that the selected subset of actions contains the best action.  Methods like partially observable Monte Carlo planning with observation widening (POMCPOW) ~\cite{sunberg2018online} use Progressive Widening to add new sampled actions after sufficient exploration, while others integrate techniques like Voronoi optimistic optimization~\cite{Michael2021Voronoi} or Bayesian optimization~\cite{mern2021bayesian}. These methods commonly apply the UCB1~\cite{auer2002finite} algorithm for exploration and Monte Carlo backups for value estimation, differing mainly in how they refine the action domain. Despite advances, the challenges posed by growing state spaces and continuous state, action, and observation spaces remain formidable.

\subsection{Neural process in robotics}
As a more efficient alternative to Gaussian processes, neural processes excel in function regression by capturing uncertainty in predictions as a stochastic process. In numerous robotics tasks, the ability to account for uncertainty in predictions is crucial for enhancing system robustness and practicality.  For instance, \cite{chen2022meta} introduces a conditional neural process (CNP)-based meta-learning algorithm to estimate grasp points from limited depth image trials.  The versatility of CNPs is further demonstrated in applications such as 6D pose estimation~\cite{Li2022Category} and understanding social aspects of navigation~\cite{yildirim2022learning}. In our work, we utilize neural processes to address uncertain scoring function regression and primitive action filtering, creating a robust framework for decision-making in complex, dynamic environments. This idea was publicly released in our earlier work~\cite{anonymous2024npfkct}, preceding several concurrent studies exploring neural process-based planning~\cite{Jain2025Learning}.

\section{Overall framework}
\label{s3}
In a cluttered indoor environment with numerous unknown movable and immovable obstacles, a robot must locate and pick up a movable target object using onboard sensors and reference photos. To solve this, we use a POMDP-based approach, GNPF-$k$CT, with 4 alternating stages: planning, execution, observation, and filtering. In the planning stage, a look-ahead MCTS exploration uses an initial belief, transition, and reward functions, observation model, rollout strategy, and terminal function to sample particles and determine the best action. To address challenges from a growing state space and continuous action domain, we enhance the POMDP solver with belief tree reuse, neural process filtering, and $k$-center clustering, improving efficiency and performance.

At the execution and observation stages, the algorithm performs the selected action on the robot, which gathers real-world measurements from onboard sensors, including 2D lidar scans, 3D point clouds, and RGBD camera images, without relying on markers like QR codes. These measurements are processed into state-space information through steps like robot pose estimation, point cloud filtering and segmentation, principal component analysis, object matching, and move-ability estimation. In the filtering stage, a particle filter updates the belief and prunes the belief tree using processed sensor data. The guessed target object and its grid world are updated via log-odds, similar to occupancy grid mapping. Robot pose estimation uses iterative closest point (ICP), initialized by adaptive Monte Carlo localization (AMCL). If the terminal condition is met, the task ends with the cumulative reward; otherwise, the updated belief reuses the existing belief tree and resamples particles to proceed with the MCTS. These steps, as shown in Fig.~\ref{f2xx}, collectively achieve object search and retrieval.


\begin{figure}[!th]
	\centering
	\includegraphics[width=\linewidth]{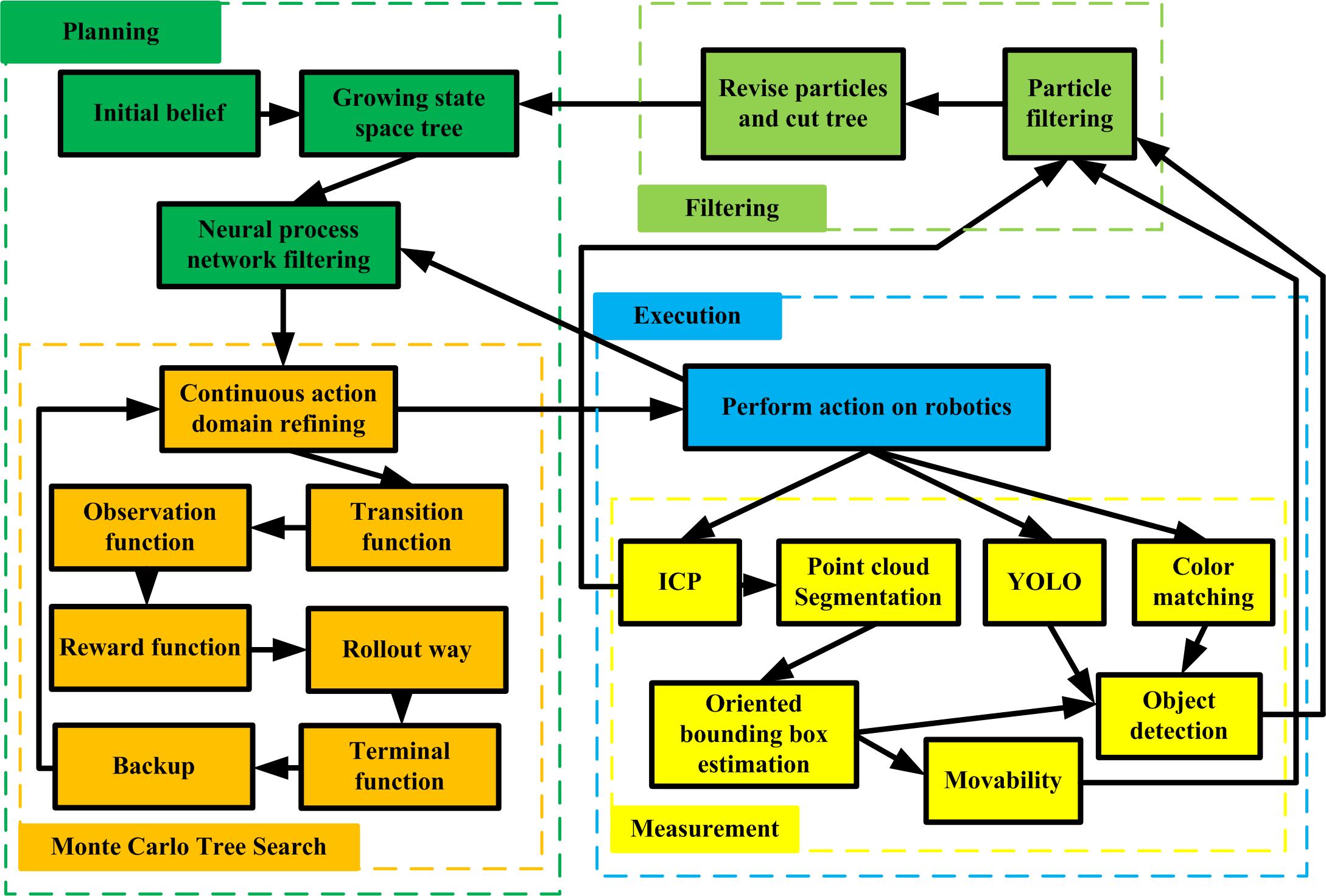}
	\caption{The main steps in our proposed approach, GNPF-$k$CT} 
	\label{f2xx}
\end{figure}

\section{POMDP formulation for object search}
\label{s4}
\subsection{General hybrid and growing POMDP}
In this paper, we consider a POMDP formulation $\mathcal{P}$ with a hybrid action domain. Formally, it is defined as an 8-tuple $<\mathcal{S},~\mathcal{A},~\mathcal{O},~{T},~{Z},~{R},~\bm{b}_0,~\gamma>$, where the state space $\mathcal{S}$ at the $k$-th time step $\mathcal{S}_k$ denotes the set of all possible states of the robot and the environment and is assumed to be continuous or discrete with growing property along: $\mathcal{S}_k = \mathcal{S}_{k-1} \times \triangle \mathcal{S}_{k-1}$, where  $\star \times \bullet$ denotes the direct product group of the group $\star$ and the group $\bullet$, $\triangle \mathcal{S}_{k-1}$ is the new added orthogonal subspace; the action space $\mathcal{A}\triangleq \mathcal{A}_c\times \mathcal{A}_d$ denotes the set of all actions the robot can perform, where sub-domain $\mathcal{A}_c$ is assumed to be continuous and embedded in a bounded metric space with distance metric function (Typically, the metric space to be a $D$-dimensional bounded Euclidean space) and sub-domain $\mathcal{A}_d$ is discrete; the observation space $\mathcal{O}$ means the set of all observations the robot can perceive and is assumed to be discrete; the transition function $T(\bm{s},~\bm{a},~\bm{s}') = Pr(\bm{s}'|\bm{s},~\bm{a})$ represents the nondeterministic effects of actions $\bm{a}\in \mathcal{A}$ working from the current state $\bm{s} \in \mathcal{S}$ to the resulting state $\bm{s}'\in \mathcal{S}$; the observation function $Z(\bm{s}',~\bm{a},~\bm{o}) = Pr(\bm{o}|\bm{s}',~\bm{a})$ is commonly a conditional probability function that represents the observation $\bm{o}$ the robot may perceive after performing action $\bm{a}\in\mathcal{A}$ in state $\bm{s}'\in \mathcal{S}$; the immediate reward function $R(\bm{s},~\bm{a},~\bm{s}')$ shows a function mapping from a state, an action, a state–action pair, or a tuple of state, action, and subsequent state to a value;  the state $\bm{s}\in \mathcal{S}$ is initially hidden in an initial belief $\bm{b}_0$, which is a probability distribution on the state space $\mathcal{S}$; $\gamma$ is a discount factor following $0<\gamma<1$, set as $0.9$ in this paper. 


The goal of solving a POMDP problem is to find an optimal policy $\Pi^*(\bm{b}) =\textbf{argmax}_{\bm{a}\in \mathcal{A}}Q(\bm{b}, \bm{a})$ for belief $\bm{b}$, where the $Q(\bm{b}, \bm{a})$-value is the value of executing action $\bm{a}$ when the agent is at belief $\bm{b}$ and continuing optimally afterwards. 

\subsection{Focused object search problem}\label{s42}
Our object search task involves locating a movable target object in clustered workspaces with unknown numbers ($n>1$) of obstacles (e.g., cups, laptops, or books) at unknown locations. We also have the point cloud and occupancy grid maps for a planning environment with furniture for robot navigation. While point clouds for some known obstacle objects are available, the target object and many other obstacles lack prior point cloud data, requiring detection via onboard sensors during execution. Unlike the assumption in OO-POMDP, because we measure the information about the object using real-world sensors, the object beliefs in our scenario are not independent and may have effects on each other. Additionally, the planning environment is continuous, with parameters such as object poses represented as continuous variables rather than discrete ones.

\textbf{State space $\mathcal{S}$.} A state $\bm{s} = \{\bm{s}_{r},~\bm{s}_{o^0},~\bm{s}_{o^1},\cdots,~\bm{s}_{o^n}\}\in \mathcal{S}$ consists of robot configuration and object-oriented factors, where $\bm{s}_{r}$ is the robot state and $\bm{s}_{o^i}$ denotes the object state of the $i$-th object. A robot state is $\bm{s}_{r}=(\bm{p}^\top, ~l_h,~l_p,~l_t)^\top$, where $\bm{p}$ is the 6D base pose, $l_h$ is the lift height, $l_p$ and $l_t$ are the pan and tilt angles of the robot head~\footnote{Unless otherwise specified, all statements are based on the Fetch robot. Our framework is adaptable to similar mobile robots with minimal changes. For example, applying it to the Stretch robot excludes $l_h$ due to its fixed camera height.}. It is easy to know that, at the beginning stage, the robot may not get into the workspace, so no detected object is available.  To guide the robot toward meaningful actions, we introduce a guessed target object with state $\bm{s}_{o^0}=(\bm{p}_{o^0}^\top,~s^x_{o^0},~s^y_{o^0},~s^z_{o^0},~\bm{g}_{o^0}^\top,~m_{o^0},~u_{o^0})^\top\in \mathbb{R}^{20}$, where $\bm{p}_{o^0}=({\bm{p}^p_{o^0}}^\top,~{\bm{p}^o_{o^0}}^\top)$ is the 6D object pose; $s^x_{o^0}$, $s^y_{o^0}$, and $s^z_{o^0}$ mean the sizes along the principal 3D axis of the object; $\bm{g}_{o^0}\in\mathbb{R}^{8}$ indicates 8 grid odds for identifying target/obstacle status, updated through object matching. The positions of these 8 grids are linked to the object’s pose, indicating that each visual observation can only capture certain parts (surfaces or grids) of the object. This helps us to identify similar objects with the same visual surfaces. $m_{o^0}\in\mathbb{R}$ is the odds value to evaluate the object’s move-ability; $u_{o^0}\in\{-2,~-1,~0,~1\}$ marks the object’s status~\footnote{-2 indicates the object has been removed; -1 means it is still updating without being declared or removed; 0 and 1 signify it has been declared as an obstacle or target object, respectively, and is no longer updating.}. The belief of the position $\bm{p}^p_{o^0}$ of the guessed target object is saved in a grid world $\mathcal{G}_f$ generated by all workspaces ${\mathcal{W}}$ using many odds values $Odd(\mathcal{G}_f)$ with a given resolution. When sampling, within each grid cell, the object’s position is uniformly sampled, and the probability is determined by the corresponding odds value. It is noted that the grids for manually generated workspaces, which are the surfaces of the furniture, differ from 8 grids associated with objects. The odds values are updated based on the field of view (FOV) of the camera using real-world measurement in the excursion process, which is similar to the update of the occupancy grid map~\cite{Zhao2024Occupancy, Yongbo2020Active}, following:
\begin{equation}\label{grid}
	\begin{aligned}
&\log Odd(g_j|z_{1:t})=\log Odd(g_j | z_t)+\\
&\log Odd(g_j |z_{1:t-1}),\\
& Odd(g_j|z_{1:t}) = {Pr(g_j|z_{1:t})}/{Pr(\neg g_j|z_{1:t})},
	\end{aligned}
\end{equation}
where $Pr(g_j|z_{1:t})$ and $Pr(\neg g_j|z_{1:t})$ means the probability of the object belonging to and not belonging to the $j$-th grid $g_j\in \mathcal{G}_f$ based on multiple observations $z_{1:t}$; $Odd(g_j|z_{1:t})$ is the corresponding odd value. However, during belief tree search, the odds values $Odd(\mathcal{G}_f)$ remain unchanged and are used to sample the guessed target object for MCTS. In the planning stage, this grid world is just used to sample the potential position of the guessed target object in the root node. The guessed target object is special with a constant (no need to estimate) orientation (set as (0,~0,~0,~1)), size (0.1,~0.1,~0.1), and move-ability value (set as 100, movable). The grid values $\bm{g}_{o^0}$ and the declared value $u_{o^0}$ are update-able in the belief tree search but need to be reinitialized as the given value after each excursion. The guessed target object is not the really detected objects. Fig.~\ref{1aa} shows the scenario about grid world for the guessed target object. The other object $\bm{s}_{o^i},~i\neq 0 \in \mathbb{R}^{1\times20}$ follows the same format and but all parameters should be updated both in belief tree search and real-world excursion. It is noted that the pose of the other objects $\bm{s}_{o^i}$ is estimated based on the point cloud which is independent of the grid world ${\mathcal{G}_f}$. The newly detected objects after each real-world excursion will be appended to the state vector and form a growing state space $(\bm{s},~\bm{s}_{add})\in  \mathcal{S}\times\triangle\mathcal{S}$. 

\begin{figure}[!th]
	\centering
	\includegraphics[width=\linewidth]{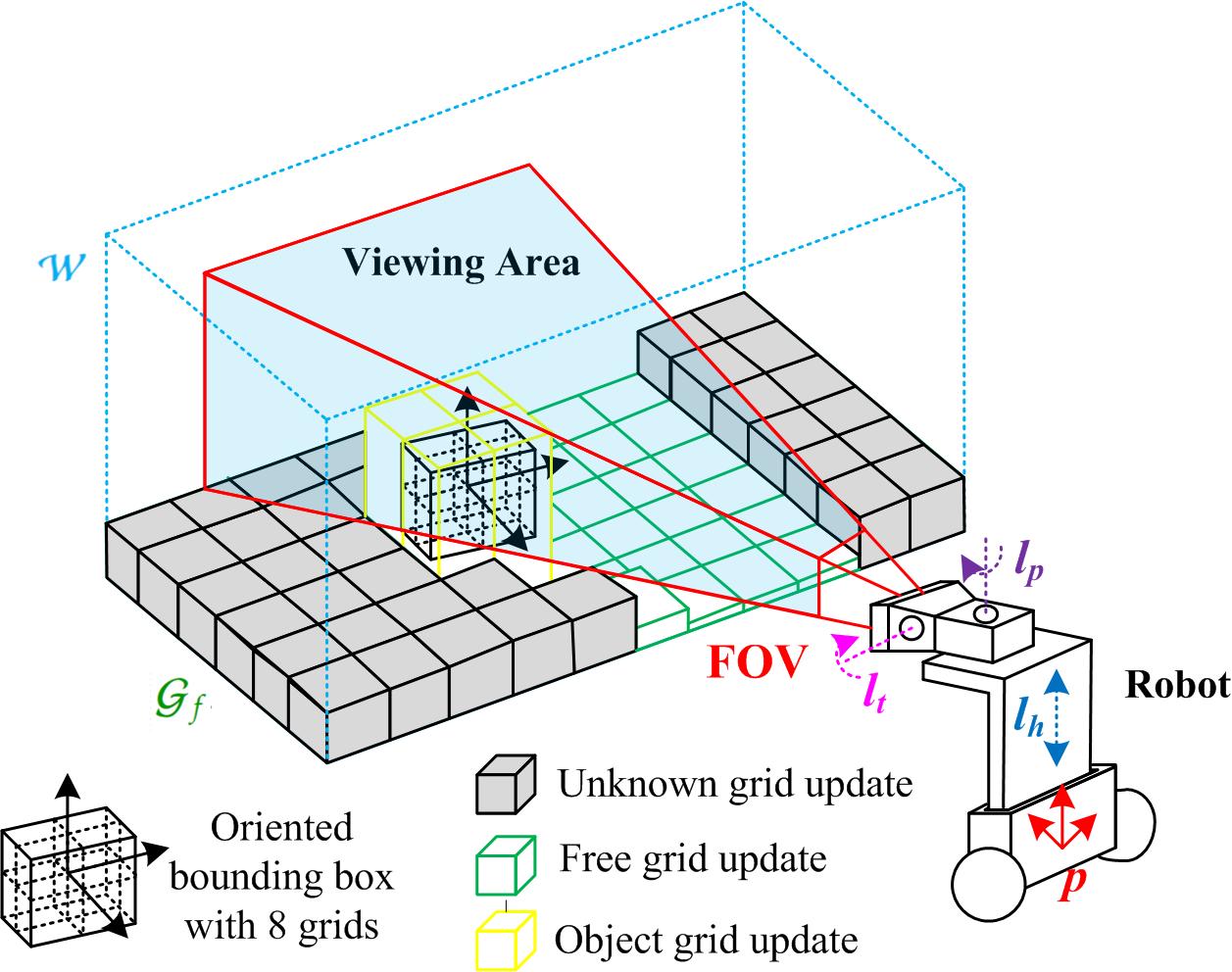}
	\caption{Grid world update in one frame} 
	\label{1aa}
\end{figure}

\textbf{Action space $\mathcal{A}$.} Our method is a task-level planning framework designed with multiple primitive actions, rather than control-level planning. These actions adjust the robot camera's FOV to cover the workspace and enable object manipulation. The action space $\mathcal{A}$ comprises 3 primitive actions, including changing the robot configuration $\bm{a}_r$, which belongs to a continuous action domain $\bm{a}_r\in\mathbb{R}^9$ , declaring the updating object ($u_{o^i}=-1$) to be the target ($u_{o^i}=1$) or obstacle object ($u_{o^i}=0$), and removing $i$-th object using robot manipulator. Changing robot configuration $\bm{a}_r$ contains the actions to move the robot base to a given target, to change the angles of the robot head pan (left/right) and tilt (up/down), and to move the height. Because the declaring and removing actions are discrete, our action space follows the previous setting with both continuous and discrete actions $\mathcal{A}_c\times \mathcal{A}_d$. For simplicity, we assume all actions succeed with a 100\% success rate.\footnote{In practice, achieving 100\% success is challenging, even in a Gazebo simulator. For object removal, we use the set\_model\_state function with Gaussian noise to adjust the model's state. This compromise will be improved in future work. \label{f1}}

\textbf{Observation space $\mathcal{O}$.} The observation focuses on the FOV frustum $V$ from the robot's head camera. Each object is modeled as an oriented box with 8 grids. For every observation containing visible objects, we update the 4 nearest grids\footnote{The parameter 4 is manually set to simulate visual surfaces, but can be adjusted, e.g., based on the number of the visible grids in this frame.} of the observed object with noisy log-odds values: 
$-c_o + \eta$ (negative), $c_o + \eta$ (positive), and $\eta$ (near-zero), where $c_o>0$ is a constant value and $\eta$ is Gaussian noise. Due to occlusions caused by other objects, the robot's camera may not detect all four nearest grids, but we still update 4 grids. We set the observation space to be discrete. By aggregating observations of  $n_o$ detected objects, the final observation for this step is represented as $\{\{i, \cdots, j\},~\{\bm{o}(i),\cdots,\bm{o}(j)\}\},~|\{i, \cdots, j\}|=n_o$, where $\{i, \cdots, j\}$ indexes the observed objects with at least one updated grid, and $\{\bm{o}(i),\cdots,\bm{o}(j)\}$ contains their updated log-odds values, $|\star|$ represents the set's cardinality.

\textbf{Transition function $T$.} The transition function $T$ primarily accounts for changes in camera motion induced by the robot's configuration  $\bm{a}_r$ and updates to object states caused by declaring or removing actions. Camera motion follows a rigid transformation and the robot's static structural information is obtained from the \emph{tf} message. A Declaring actions update an object's state $u_{o^i}$ by comparing the mean of the two smallest log-odds values ($\nu=2$) among its 8 grids to thresholds $C^o_d$ (obstacle) and $C^o_t$ (target). Objects remain static and rigid unless removed. When a removing action is performed on the $i$-th object, the object is moved outside the workspace, and its status $u_{o^i}$ is set to -2, indicating that it will no longer obstruct the observation of other objects in the workspace.

\textbf{Observation function $Z$.} 
The coordinates of the eight grids for the $i$-th object are computed using the object's center point $\bm{p}^p_{o^i}$, orientation $\bm{p}^o_{o^i}$, and size dimensions $s^x_{o^0}$, $s^y_{o^0}$, and $s^z_{o^0}$, estimated from point cloud data.   After an action, the visible grids within the four nearest grids of different objects in the current frame are updated. The updated value depends on the mean log-odds of the grid: If the original mean log-odds value of this grid is positive and exceeds a positive threshold $\nu_p>0$, the updated value will be set to be positive $c_0+\eta$. If below a negative threshold $\nu_n<0$, the updated value will be set to be negative $-c_0+\eta$. Otherwise, the value between the positive and negative thresholds will generate a near-zero $\eta$. To encourage exploration, the initial grid values $\bm{g}_{o^0}$ of a fake object are reinitialized to values greater than the positive threshold $\nu_p>0$ after each real-world excursion.

\textbf{Reward function $R$.} The reward function is a crucial component of the POMDP formulation, guiding the agent toward successfully completing the task.  The robot's expected sequence of actions is that the robot begins by moving to the workspace for initial observation, iteratively adjusts its viewpoint to update log-odds in object grids, declares objects as obstacles or targets, and finally removes obstacles to improve visibility or directly removes the target to complete the task. Following these expected actions, successfully removing the target yields the highest reward ($R_{max} \gg 0$), while correct declarations of obstacles ($R_{c_o}$) and targets ($R_{c_t}$) offer moderate rewards, with $R_{max} \gg R_{c_t} > R_{c_o} \gg 0$. All actions will receive a negative step cost $R_{min}<0$. Due to the complexity of the removing action, its cost is set as $2R_{min}$. Illegal actions result in a significant penalty $R_{ill}\ll 0$. Illegal actions include: (a)  Moving to a target pose that is within an occupied grid of the occupancy map (violating collision avoidance); (b) Attempting to remove an object that is estimated to be non-movable; (c) Only the updating object is allowed to be declared and others are illegal; (d) Declaring an object that does not satisfy the declaration conditions; (e) Attempting to remove objects that have not been declared or objects that have already been removed.

\section{Online POMDP solver: GNPF-$k$CT}
\label{s6}
\subsection{Overview}
GNPF-$k$CT is an anytime online POMDP solver that assumes the Q-value of the POMDP problem follows Lipschitz continuity in the action space\footnote{Definition: A value function $f$ mapping $\mathcal{S}\subset\mathbb{R}^n$ to $\mathbb{R}^m$ is Lipschitz continuous at $\bm{x} \in \mathcal{S}$ if there exists a constant $\eta$ such that $|f(\bm{y}) - f(\bm{x})| \leq \eta|\bm{y} - \bm{x}|$ for all $\bm{y}\in \mathcal{S}$ near $\bm{x}$.}. The solver alternates between planning, execution, observation, and filtering, focusing on planning to identify the optimal action from the current belief $\bm{b}_0$. To reduce the complexity of the high-dimensional continuous action domain, a pre-trained neural process predicts the feasibility of sampled actions with associated uncertainties, filtering out irrelevant actions. The remaining actions are clustered into hyperspheres using $k$-center clustering. A belief tree $\mathcal{T}$ is then constructed, where nodes represent beliefs and actions. Each belief node $\bm{node}_o \in \mathcal{T}$ is linked to a dynamic list $\mathcal{L}(\bm{node}_o)$, initialized from the previous tree and augmented with newly generated hyperspheres to include connected action nodes. During episode simulations, $\mathcal{L}(\bm{node}_o)$ expands as $k$-center clustering refines the action space further. In POMDP problems with growing state spaces (e.g., detecting new objects), the observation introduces unpredicted states, requiring classical solvers to rebuild the entire belief tree, which is computationally expensive. GNPF-$k$CT avoids this by saving a history tuple $\bar{\mathcal{H}} = <H_s, H_r, H, dim_r(H_s)>$. The list $H_s$ tracks all state histories from simulations, where rows represent particle IDs and columns track their state histories.  Two additional lists, $H$ (action-observation history) and $H_r$ (immediate rewards), as well as a vector $dim_r(H_s)$ (depth vector) are associated with $H_s$. These saved histories allow the solver to approximate and reuse the previous belief tree by updating $H_s$, avoiding a full rebuild. Algorithm~\ref{alg:cap11} outlines the complete method, with details provided in subsequent sections\footnote{Gray text provides explanations for the corresponding line.}. 


\begin{algorithm}
\small
\caption{GNPF-$k$CT framework}\label{alg:cap11}
\begin{algorithmic}[1]
\Require Initial belief $\bm{b}_0$
\Ensure Task is complete or not
\State $\bm{b}\gets{\bm{b}_0}$
\State isTerminal = False
\While{isTerminal is False}
\State \textcolor{red}{-----------------Planning stage----------------}
\If{$\bm{s}_{add}\neq \emptyset$}
\State $\bar{\mathcal{H}}\gets$Episode\_simulation\_along\_list($\bar{\mathcal{H}},~\bm{s}_{add}$) \Comment{\textcolor{mygray}{Update states $\bm{s}$ saved in $H_s$ by transition and observation functions considering new states $\bm{s}_{add}$; See Algorithm~\ref{alg:cap9}.}}
\State $\mathcal{T} \gets$Grow\_tree($\mathcal{T}$,~$\bar{\mathcal{H}}$)\Comment{\textcolor{mygray}{Grow the belief tree using these updated particles and decide terminal nodes by saved depth vector $dim_r(H_s)$; See Algorithm~\ref{alg:euclid101}.}}
\EndIf
\State $\mathcal{A}_r =\{\bm{a}_i\}\gets$~Network$\_$sampling($Odd(\mathcal{G}_f)$,~$\mathcal{P}$,~$\bm{b}$,\\~$C_{threshold}$)\Comment{\textcolor{black}{\textcolor{mygray}{Sample a lot of candidate actions and filter them using neural network; See Algorithm~\ref{alg:cap4x}.}}}
\State $\mathcal{C}_a=\{\bm{center}_i\},~\mathcal{R}_a=\{{range}_i\}$~$\gets k$-Clustering($\mathcal{A}_r$)\Comment{\textcolor{mygray}{Generate high-dimension hyperspheres by partitioning and covering the candidate actions set using $k$-Clustering over some discrete actions, satisfying $|\mathcal{C}_a|=|\mathcal{R}_a|=k$.}}
\State $\mathcal{L}(\bm{node}_o)$ = \text{Update}($\mathcal{T}$,~$\mathcal{C}_a$,~$\mathcal{R}_a$)  \Comment{\textcolor{mygray}{Extend the list of each belief node based on the cluster centers and ranges recursively.}}
\While{planning budget not exceeded}
\State $\bm{s}\gets \text{Sampling}(\bm{b})$ 
\State $\mathcal{T},~\bar{\mathcal{H}}\gets$Episode\_simulation($\mathcal{T}$, $\bm{s}$, $\bm{h}$, $\bar{\mathcal{H}}$) \Comment{\textcolor{mygray}{Go through the  MCTS sampling by refining the action domain, then save the tuple $\bar{\mathcal{H}}$ for each particle; See Algorithm~\ref{alg:cap4}.}}
\EndWhile
\State \textcolor{red}{------Execution and observation stages--------}
\State $\bm{a}^*, \bm{center}^*, {range}^* \gets$ Get the best action in $\mathcal{T}$ from $\bm{b}$
\While{planning budget not exceeded}
\State $\bm{a}_{implement}~\gets$ Action\_sampling($\bm{a}^*$,~$\bm{center}^*$ $, {range}^*$) \Comment{\textcolor{mygray}{Get a discrete action from the selected domain.}}
\If{\text{Reasonability\_check}$(\bm{a}_{implement})$}
\Comment{\textcolor{mygray}{Only select the reasonable discrete action, such as preventing a robot from moving
into obstacles.}}
\State break
\EndIf
\EndWhile
\State ($\bm{o}$, \text{isTerminal}) ~$\gets$~Execute $\bm{a}_{implement}$~\Comment{\textcolor{mygray}{Execute the selected action and get real observation; See Section~\ref{s3}}}
\State \textcolor{red}{-----------------Filtering stage-----------------}
\If{$\bm{s}_{add}\neq \emptyset$}
\State $\bm{o}~\gets$ Cut\_observation($\bm{o},~\bm{s},~\bm{s}_{add}$)~\Comment{\textcolor{mygray}{Get the sub-observation available in the old tree~$\mathcal{T}$ by removing sub-part related to $\bm{s}_{add}$.}}
\EndIf
\State $\bm{b},~\mathcal{T}~\gets$\text{Filtering}($\bm{b},~\mathcal{T},~\bm{a}_{implement},~\bm{o}) $\Comment{\textcolor{mygray}{Particle filtering to get a new belief and the corresponding sub-tree.}}
\State $\bar{\mathcal{H}} \gets$ \text{Cutting}$(\bm{a}^*, \bm{o}, \bar{\mathcal{H}})$\Comment{\textcolor{mygray}{Cut lists based on observation $\bm{o}$; See Algorithm~\ref{alg:cap8}.}}
\EndWhile
\end{algorithmic}
\end{algorithm}

\subsection{GNPF-kCT: NP network for filtering}
\subsubsection{Motivation about network filtering.}\label{s3B1}
In complex tasks, many primitive actions with continuous control parameters ($\mathcal{A}c$) are often inefficient or impractical for achieving long-term goals, such as the optimal policy goal of POMDPs, or short-term goals, like completing a primitive action.  Specifically, if the optimal action $\bm{a}^* = \textbf{argmax}_{\bm{a}\in\mathcal{A}_c} Q(\bm{b}, \bm{a})$ lies within the continuous action domain $\mathcal{A}_c$, there exists a smaller feasible region $\mathcal{X} \subseteq \mathcal{A}_c$ that satisfies $\bm{a}^* \in \mathcal{X}$. An intuitive idea is to use a score function $g$, with some inherent uncertainty, to quickly identify this region $\mathcal{X}$ given $\bm{s}$ and $\alpha(\mathcal{P})$, where $\alpha(\mathcal{P})$ represents all configuration settings of the POMDP problem $\mathcal{P}$. Gaussian Processes (GPs) are a natural choice for modeling such functions, leveraging probabilistic inference to update predictions, but they become computationally expensive in high-dimensional tasks. To address this, we propose to use Neural Processes (NPs), which also define distributions over functions while being computationally efficient and adaptable to various inputs leveraging the advantages of neural networks. Unlike traditional approaches like Q-learning~\cite{hausknecht2015deep}, which train complex networks with large minimum description lengths~\cite{Zirui2024Large} to pursue $\mathcal{X}=\bm{a}^*$, we adopt a simpler scoring network. This network prioritizes short-term, physically meaningful goals (e.g., single-step objectives), making training and implementation more practical. The NP-based scoring network efficiently filters out ineffective actions, enabling a more compact POMDP framework to select optimal actions based on the filtered actions and current beliefs. We employ attentive NPs $nn(\mu(\bm{a}), \sigma(\bm{a}))$ to learn the scoring function, where $\mu(\bm{a})$ is the mean function and $\sigma(\bm{a})$ is the kernel function. Importantly, our aim is not to mandate the use of the same network for all POMDP tasks but to highlight the action-filtering concept. Users are encouraged to train domain-specific networks tailored to their problems, integrating short-term primitive action learning with long-term POMDP frameworks in continuous action domains:
\begin{equation}\label{smooth-constraint}
	\begin{aligned}
	g(\bm{a},~\bm{s},~\alpha(\mathcal{P})) \sim nn(\mu(\bm{a}),~\sigma(\bm{a})).
 \end{aligned}
\end{equation} 

To preserve the optimal action $\bm{a}^*$, eliminate irrelevant actions, and accurately represent the complex feasible region $\mathcal{X}$, we generate a set of actions ${\bm{a}_i} \subset \mathcal{A}_c$ such that, with high probability, $g(\bm{a}_i,~\bm{s},~\alpha(\mathcal{P})) \geq 0$.  These samples are then encapsulated using high-dimensional hyperspheres to represent the feasible region. To ensure a high likelihood of selecting effective robot actions, we establish a predictive score bound for these samples, derived from the theorem:


\begin{theorem}\label{t_add_new}
Let $\delta \in (0, 1)$, $g(\bm{a},~\bm{s},~\alpha(\mathcal{P})) \sim nn(\mu(\bm{a}),~\sigma(\bm{a}))$ and set $\beta^{*} =
(2\ln(1/\delta)) ^{\frac{1}{2}}$. If $\mu(\bm{a}_i)>\beta^{*}\sigma(\bm{a}_i)$, $\forall i = 1,\cdots,T$, then $Pr[g(\bm{a}_i,~\bm{s},~\alpha(\mathcal{P}))> 0]\geq 1 - \delta$. 
\end{theorem}

\begin{proof}
Please refer to Appendix~\ref{a1}.
\end{proof}

 This theorem provides a condition for actions $\bm{a}_i$: If all the sampled actions using $nn(\mu(\bm{a}_i),~\sigma(\bm{a}_i))$ satisfying $\mu(\bm{a}_i)>\beta^{*}_{i}\sigma(\bm{a}_i)$, then all samples will satisfy the constraint $g(\bm{a}_i,~\bm{s},~\alpha(\mathcal{P}))  > 0$ with probability at least $1 - \delta$. This simple conclusion offers a good way to sample good robot actions using their predictive scores.

\subsubsection{Neural process filtering implemented for object searching task.}
 The goal of the object search process is to locate the target object (assumed to be the $i$-th object) and update its grid belief $\bm{b}(\bm{g}{o^i})$ from an initial belief $\bm{b}_{0}(\bm{g}_{o^i})$ to a target belief $\bm{b}_T(\bm{g}{o^i})$ by adjusting the robot’s FOV. This process involves passing a certain threshold to enable subsequent declaration and removal actions. Among all primitive actions $\mathcal{A}$, those deemed \emph{efficient} $\{\bm{a}_i,\cdots\}\in \mathcal{X}$ are identified if, at step $j$, they can update the belief $\bm{b}_{j}(\bm{g}_{o^i})$ to move closer to the target belief $\bm{b}_{T}(\bm{g}_{o^i})$ within a bounded distance: $\omega_{\triangle} \triangle \leq \|\bm{b}_{j-1}(\bm{g}_{o^i})-\bm{b}_T(\bm{g}_{o^i})\|_1 - \|\bm{b}_{j}(\bm{g}_{o^i})-\bm{b}_T(\bm{g}_{o^i})\|_1 \leq \triangle $, where $\triangle$ serves as a natural upper bound for grid updating, like updating one surface of the nontransparent object, since it is impossible to observe all surfaces of an opaque object simultaneously. $\omega_{\triangle},~0<\omega_{\triangle}<1$ means at least one grid is observed and updated correctly,  $\|\star\|_1$ means the L1 norm for $\star$. According to Theorem~\ref{t_add_new}, if the action $\bm{a}_i$ satisfies $\mu(\bm{a}_i) > \beta^{*}$ $\sigma(\bm{a}_i)$, the probability of the grid belief moving closer to $\bm{b}_T(\bm{g}_{o^i})$ is at least $1 - \delta$. Given this, a successful action sequence that reaches the target belief and completes the task selected by the POMDP solver must include at least $N_l = [\|\bm{b}_{0}(\bm{g}_{o^i})-\bm{b}_{T}(\bm{g}_{o^i})\|_1/\triangle]$ efficient actions within $N_p\geq N_l$ potential primitive actions, where $[\star]$ means to round to an integer smaller than the value $\star$. Equality holds only if every primitive action is efficient. Faster task completion with higher rewards is achieved by prioritizing actions with larger $\beta^*$, satisfying $\mu(\bm{a})>\beta^{*}\sigma(\bm{a})$. All these ideas rely on accurate learning of the scoring function.  The scoring function $g(\bm{a},~\bm{s},~\alpha(\mathcal{P}))$ predicts the likelihood of the robot observing updated grids for the target object. This probability depends on the robot's configuration $\bm{s}_r$, obstacle data from the fused point cloud of detected objects $\{o_0,~o_1,\cdots,~o_n\}\in\mathcal{M}'_c$, the grid world status $\mathcal{G}_f$ (represented as a 2D grayscale image), and the 8 grid odds $\bm{g}_{o^i}$ of the target object (assume to be the $i$-th object), satisfying $g(\bm{a},~\bm{s},~\alpha(\mathcal{P})) \triangleq g(\bm{s}_r,~\mathcal{G}_f,~\mathcal{M}'_c,~\bm{g}_{o^i})$.

\subsubsection{Simulation data generator.}
To learn the scoring function, we use an iterative process to autonomously generate simulation data in a Gazebo environment with two simulators, as detailed in Algorithm~\ref{alg:cap}. Initially, predefined action sequences generate point clouds of detected objects and the corresponding grid world for the guessed target object. Subsequently, we repeatedly modify the robot's configuration and object grid arrangement, using object detection to calculate the probability of detecting the target object. This process operates fully autonomously once candidate actions and the workspace $\mathcal{W}$ are provided. Fig.~\ref{fx1} illustrates a scenario involving a sampled target object, a specific robot configuration, and successful online color-based object detection. 

\begin{figure}[!htb]
	\centering
\includegraphics[width=0.9\linewidth]{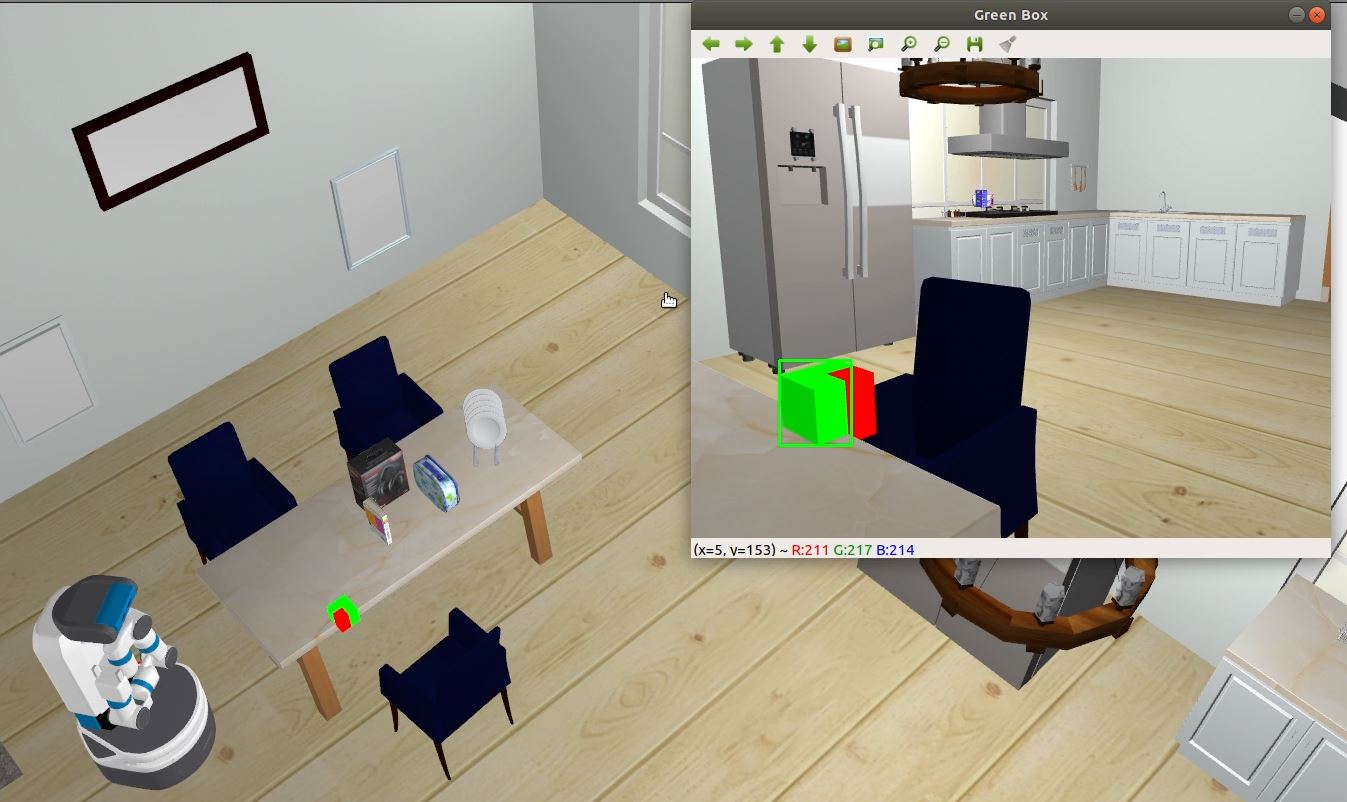}
\caption{Robot successfully detects the green grids in this scene and its variable "$score_d$" will add 1 in Algorithm~\ref{alg:cap}.}
\label{fx1}
 \end{figure}

\begin{algorithm*}
\small
\caption{Simulation dataset generator}\label{alg:cap}
\begin{algorithmic}[1]
\Require 2D occupancy grid map, corresponding 3D point cloud map, Fetch robot simulator in Gazebo environment, a set of manually selected candidate actions $\mathcal{A}_{manual}$ 
\Ensure  Data mappings $(\bm{s}_r,~\mathcal{G}_f,~\mathcal{M}'_c,~\bm{g}_{o^i})\rightarrow g(\bm{a},~\bm{s},~$ $\alpha(\mathcal{P}))$
\While{Dataset size is not enough}
\State Re-initialize the robot simulator and gazebo environment with different object numbers and poses. \Comment{\textcolor{mygray}{Different scenarios.}}
\State Randomly generate $I$ classes of action sequences $\{\bm{a}^1_i,~\bm{a}^2_i,\dots,~\bm{a}^J_i\},~\bm{a}_i^j \in  \mathcal{A}_{manual} \subseteq  \mathcal{A},~i=1,~2,\cdots,~I$ from $\mathcal{A}_{manual}$ with limited $J$ steps. \Comment{\textcolor{mygray}{Different action sequences generate different observed point clouds and grid worlds.}}
\For{$i = 1$  to $ I$}  \Comment{\textcolor{mygray}{Different action sequences for different scenarios.}}
\For{$j = 1$  to $ J$} \Comment{\textcolor{mygray}{Action steps for growing obstacle environment.}}
\State Set robot status based on the action $\bm{a}_i^j$ with some noise.  \Comment{\textcolor{mygray}{Apply the actions to generate different observed environment.}}
\State Collect and fuse point cloud $\mathcal{M}'_c = f_F(\bm{P}\bigcup\bm{P}_i^j)$ after ICP and filtering operations. \Comment{\textcolor{mygray}{Generate point cloud.}}
\State Compute the odds update for the grid world $\mathcal{G}_f$ based on FOV and object detection.  \Comment{\textcolor{mygray}{Generate grid world with odds.}}
\State Remove the target object (identified) and undetected objects in gazebo environment. \Comment{\textcolor{mygray}{Use only the information from the detected obstacle environment for visual occlusion testing.}}
\State Collect and save $\mathcal{M}'_c$ and $Odd(\mathcal{G}_f)$.
\For{$k = 1$  to $ K$}
\State Uniformly sample and apply configurations $\bm{s}_r$ in continuous action domain $\mathcal{A}_c$. \Comment{\textcolor{mygray}{Generate robot configuration $\bm{s}_r$.}}
\State Build a cube with 8 red and green grids. The color distribution is decided by $\bm{g}_{o^i}\in\{0,~1\}^{8}$. \Comment{\textcolor{mygray}{Avoid observed grids, marked as 1 and red color, to be re-observed assuming that $i$-th object has the highest probability to be target object.}}
\State Collect and save $\bm{s}_r$ and $\bm{g}_{o^0}$.
\State $score_d \gets  0$   \Comment{\textcolor{mygray}{Index to record the observed times under given parameters $(\bm{s}_r,~\mathcal{G}_f,~\mathcal{M}'_c,~\bm{g}_{o^0})$.}}
\For{$l = 1$  to $ L$} \Comment{\textcolor{mygray}{Repeat multiple times to compute probability.}}
\State Sample the position of the cube with 8 grids in different positions based on odds value $Odd(\mathcal{G}_f)$.
\If{Object detection finds the green area based on the collected RGBD image is True}
\State $score_d \gets  score_d+ 1$ \Comment{\textcolor{mygray}{Based on given parameters, record the times to update grids in $L$ times attempts.}}
\EndIf
\EndFor 
\State $g(\bm{a},~\bm{s},~$ $\alpha(\mathcal{P})) \gets  \frac{score_d}{L}\times 100\%$.
\State {Collect and save} $g(\bm{a},~\bm{s},~\alpha(\mathcal{P}))$. \Comment{\textcolor{mygray}{Compute probability.}}
\EndFor 
\State Rearrange all objects based on their original poses before removing them. \Comment{\textcolor{mygray}{Recover the object setting to generate the grown point cloud and the updated grid world.}}
\EndFor 
\EndFor
\For{$m = 1$  to $M = I\times J\times K \times L$}
\State Normalize the following data mappings: $(\bm{s}_r, Odd(\mathcal{G}_f), \mathcal{M}'_c, \bm{g}_{o^i}) \rightarrow g(\bm{a},~\bm{s},~\alpha(\mathcal{P}))$. \Comment{\textcolor{mygray}{Get learning dataset.}}
\EndFor 
\EndWhile
\end{algorithmic}
\end{algorithm*}

The inputs to Algorithm~\ref{alg:cap} include: 1) a 2D occupancy grid map for initializing 3D ICP matching; 2) A 3D point cloud map, employed to create a fused object point cloud by performing scan matching with the current 3D camera's point cloud and filtering out points outside the workspace; 3) The Fetch simulator for real-time sensor data; 4) A set of manually selected candidate actions, $\mathcal{A}_{manual}$, chosen from randomly generated actions to ensure good diversity. The outputs are training pairs for the neural network mapping, which connect the robot state $\bm{s}_r$, the generated grid world $\mathcal{G}_f$, the detected objects in the workspace (represented as point clouds $\mathcal{M}'_c$), and the status of the target object's 8 grids $\bm{g}_{o^i}$ to a scoring value. This scoring value represents the probability of updating the grids of the target object.


\subsubsection{Network structure and training.}
Considering all the input formats, we employ various encoder networks, including partial pointNet~\cite{Qi2017Pointnet}, ResNet18~\cite{He2016Deep}, and Multilayer perception (MLP) to process different inputs. The point cloud $\mathcal{M}'_c$ is encoded into the global feature vector with 1024 elements. The odds value for the grid world about the guessed target object is transformed into a 2D grayscale image and further encoded as a 1000-dimension vector using ResNet-18. The robot configuration and the 8-grid odds are processed by MLP to generate three 10-dimensional global features. The concatenated global features form a tensor of size $M\times 2054$. For NPs, the whole tensor is divided into training data $M_1\times 2054$, including the context data $(\bm{x_C},~\bm{y_C})$ and the target data $(\bm{x_T},~\bm{y_T})$, and test data $M_2\times 2054$, satisfying $M_1 + M_2 = M$. With these data, the latent variable version of the NP model models these conditional distributions as:
\begin{equation}\label{smooth-constraint}
	\begin{aligned}
	&p(\bm{y_T}|\bm{x_T},~\bm{x_C},~\bm{y_C})\triangleq\int p(\bm{y_T}|\bm{x_T},~\bm{z})\\&q(\bm{z}|r(\bm{x_C},~\bm{y_C})) d\bm{z}
 \end{aligned}
\end{equation} 
where $r(\bm{x_C},~\bm{y_C})$ is an encoder function that shows a representation of the context data using an MLP, $p(\star|\bullet)$ denotes the conditional prior for $\star$ given $\bullet$, and $q(\star|\bullet)$ means the variational posterior for $\star$ given $\bullet$. Then, with the latent variable $\bm{z}$ generated by the Gaussian sampling of the representation $r(\bm{x_C},~\bm{y_C})$, the MLP-based decoder process is applied for the latent variable $\bm{z}$ and the test data to model a final Gaussian distribution for prediction. The whole network structure without training based on the test data and the context data is shown in Fig.~\ref{fig:landscape} in Appendix~\ref{ab2}. For training, the network about the encoder part to the latent variable needs to work on both the context data and the target data to get Kullback–Leibler (KL) divergence $D_{KL}(\bullet||\bullet)$ between prior and posterior. The parameters of the whole network are learned by maximizing the evidence lower bound (ELBO) $\log p(\bm{y_T}|\bm{x_T},~\bm{x_C},~\bm{y_C}) \geq u_{ELBO}$~\cite{Garnelo2018Neural}:
\begin{equation}\label{smooth-constraint}
	\begin{aligned}
	&u_{ELBO}=\mathbb{E}_{q(\bm{z}|\bm{x_C},~\bm{y_C})}\left[\log p(\bm{y_T}|\bm{x_T},~\bm{z})\right]\\&+D_{KL}(q(\bm{z}|\bm{x_T},~\bm{y_T})||q(\bm{z}|\bm{x_C},~\bm{y_C})),
     \end{aligned}
\end{equation}
where $\mathbb{E}_{q(\bm{z}|\bm{x_C},~\bm{y_C})}\left[\log p(\bm{y_T}|\bm{x_T},~\bm{z})\right]=\int q(\bm{z}|\bm{x_C},$ $~\bm{y_C})\log p(\bm{y_T}|\bm{x_T},~\bm{z}) d\bm{z}$. Building on the findings in~\cite{Le2018Empirical} and to enhance learning performance, the context data are selected as a subset of the target data, and the observation variance is modeled as a latent variable constrained within a range of 0.1 to 1. Additionally, the NP model incorporates a self-attention module to preprocess context and test data tensors, reducing predictive uncertainty near context points, as shown in~\cite{Kim2019Attentive}.

\subsection{GNPF-$k$CT: Belief Tree Construction}
The belief tree $\mathcal{T}$ is a tree of which the nodes include the belief nodes $\bm{node}_o$ and action nodes $\bm{node}_a$ and the edges are action $\bm{a}\in \mathcal{A}$ and observation $\bm{o}\in \mathcal{O}$. The belief nodes $\bm{node}_o$ save a set of states $\bm{s}\in \bm{b}$  forming belief $\bm{b}$. An action node $\bm{node}_a$ is a child of the belief node $\bm{node}_o$ via edge $\bm{a}$ if and only if transition function $T(\bm{s},~\bm{a},~\bm{s}')$ satisfies. Belief nodes are linked by the action–observation pairs $(\bm{a},~\bm{o})$ based on  transition function $T(\bm{s},~\bm{a},~\bm{s}')$ and observation function $Z(\bm{s}',~\bm{a},~\bm{o})$. The episode simulation framework is in Algorithm~\ref{alg:cap4}.

\begin{algorithm}
\small
\caption{Episode\_simulation($\mathcal{T}$, $\bm{s}$, $\bm{h}$, $\bar{\mathcal{H}}$) }\label{alg:cap4}
\begin{algorithmic}[1]
\Require Belief tree $\mathcal{T}$, sampled state $\bm{s}$, history in episode $\bm{h}$, the history tuple $\bar{\mathcal{H}}$
\Ensure  Discounted reward value $r$, updated history tuple $\bar{\mathcal{H}}$, and updated belief tree $\mathcal{T}$.
\State $H,~H_s,~H_r,~dim_r(H_s)\gets{\bar{\mathcal{H}}}$
\State belief node $\bm{node}_o \gets (N(\bm{node}_o), V(\bm{node}_o))$ based on history $\bm{h}$
\If{$\bm{node}_a$ is False}
\State Generate all action nodes $\bm{node}_a \gets (N(\bm{node}_o,~\bm{a}), V(\bm{node}_o,~\bm{a}))$ based on list $\mathcal{L}(\bm{node}_o)$;
\State $r,~\bm{h}_s,~{\bm{h}}_r,~\bar{\bm{h}},\gets$RolloutHeuristic\_Random($\bm{s}$, $\bm{h}+\{\bm{a}, \bm{o}\}$) \Comment{\textcolor{mygray}{A rollout policy with the random action strategy. $\bm{h}_s$ and $\bar{\bm{h}}$ are the record states and action-observation pairs from tree root until to the terminal of rollout.}}
\State $H_s \gets H_s\cup \{\bm{h}_s\}$, $H_r \gets H_r\cup \{\bm{h}_r\}$, $H \gets H\cup \{\bar{\bm{h}}\}$, $dim_r(H_s) \gets dim_r(H_s) $ $\cup\{|\bm{h}|\}$ \Comment{\textcolor{mygray}{Terminal node before rollout, so save lists and depth vector $dim_r(H_s)$, until terminal node, by Hash table format sharing the particle ID as keys.}}
\State \Return $r$ \Comment{\textcolor{mygray}{Get the random rollout reward.}}
\Else 
\State $\bm{a}^* = \text{argmax}_{\bm{a}\in \mathcal{L}(\bm{node}_o)} U (\bm{node}_o,~\bm{a})$ \Comment{\textcolor{mygray}{Eq.~\eqref{list1}}}
\State $\mathcal{T},~H\gets$Refine($\mathcal{T}$,~$\bm{node}_o$,~$\bm{a}^*$,~$H$)
\Comment{\textcolor{mygray}{Grow list, revise tree, and update history list. See Algorithm~\ref{alg:cap5}.}}
\State $\bm{a}_u^*\gets$Action\_sampling($\bm{a}^*$,~$\bm{center}^*$,~$range^*$)
\Comment{\textcolor{mygray}{Sample a discrete action $\bm{a}_u^*$ from range $\mathcal{H}(\bm{center}^*,~range^*)$; Refer to Section~\ref{s634}.}}
\State Get $\bm{s}'$ and $\bm{o}$ based on  $\bm{a}_u^*$, $T(\bm{s},~\bm{a}^*,~\bm{s}')$, and $Z(\bm{s}',~\bm{a}^*,~\bm{o})$
\If{$\bm{s}$ is not terminal state}
\State $r \gets \gamma \text{Episode\_simulation}(\mathcal{T}$,~$\bm{s}'$,~$\bm{h}+\{\bm{a}^*, \bm{o}\}, H_s,$ $H_r, H, dim_r(H_s))  + R(\bm{s},~\bm{a}^*,~\bm{s}')$
\Comment{\textcolor{mygray}{Recursively}}
\State Backup($\mathcal{T}$,~$\bm{node}_o$,~$\bm{a}^*$,~$r$,~$R(\bm{s},~\bm{a}^*,~\bm{s}')$) \Comment{\textcolor{mygray}{See Algorithm~\ref{alg:cap7}.}}
\Else
\State $H_s \gets H_s\cup \{\bm{h}_s\}$, $H_r \gets H_r\cup \{{\bm{h}_r}\}$, $H \gets H\cup \{{\bm{h}}\}$,  $dim_r(H_s) \gets dim_r(H_s)\cup \{|\bm{h}|\}$ \Comment{\textcolor{mygray}{Terminal node, so save state lists $H_s$, $H_r$ and $H$ as well as depth vector $dim_r(H_s)$ by Hash table format.}}
\EndIf
\EndIf
\State \Return $\bar{\mathcal{H}}\gets<H$, $H_s$, $H_r$, $dim_r(H_s)>$, $\mathcal{T}$.
\end{algorithmic}
\end{algorithm}

\subsubsection{Construction overview.} To construct the belief tree $\mathcal{T}$, our GNPF-$k$CT framework follows the standard select-expand-simulate-backup strategy used in many MCTS algorithms, with an adaptive discretization mechanism based on the $k$-center clustering method for controllable discretization rates. As mentioned before, each belief node maintains a dynamic list $\mathcal{L}(\bm{node}_o)$. If resources allow, episodes are continuously sampled to grow the belief tree $\mathcal{T}$, at each episode, we select a path from the tree root with the format: $\bm{s}_0$, $\bm{a}_0$, $\bm{o}_0$, $\bm{r}_0$, $\bm{s}_1$, $\bm{a}_1$, $\bm{o}_1$, $\bm{r}_1$, $\cdots$. Starting from the root belief, an action $\bm{a}_i \in \mathcal{L}(\bm{node}_o)$ is selected using a revised UCB1 strategy (based on HOO, detailed in Section~\ref{s433}). The selected action is executed from state $\bm{s}_i$ to achieve the next state $\bm{s}_{i+1}$ and obtain observation $\bm{o}_i$ as well as reward $\bm{r}_i$ based on transition $T(\bm{s}_i,~\bm{a}_i,~\bm{s}_{i+1})$, observation $Z(\bm{s}_{i+1},~\bm{a}_i,~\bm{o}_i)$, and reward functions $R(\bm{s}_i,~\bm{a}_i,~\bm{s}_{i+1})$. Finally, we can update the belief $\bm{b}_i=\{\bm{s}_i\}$ to $\bm{b}_{i+1}=\{\bm{s}_{i+1}\}$. If the terminal condition is met or no child node exists, the tree expands by adding a new belief node and associated action nodes based on $\mathcal{L}(\bm{node}_o)$. A rollout policy, typically RolloutHeuristic\_Random($\bullet$), is simulated to provide an estimated value for new (terminal) nodes. States, action-observation, immediate rewards, and depths along the episode are recorded in lists $H_s$, $H$, $H_r$, and $dim_r(H_s)$ when the node is terminal due to reaching rollout or terminal state. Finally, backup operations update the estimated values for all visited nodes in the episode.

 \begin{algorithm}
\small
\caption{Network$\_$sampling($Odd(\mathcal{G}_f)$,~$\mathcal{P}_j$,~$\bm{b}$,\\~$~ C_{threshold})$ }\label{alg:cap4x}
\begin{algorithmic}[1]
\Require The trained network $nn(\mu(\bm{a}),~\sigma(\bm{a}))$, the odds update for the whole grid world about the guessed target object $Odd(\mathcal{G}_f)$,~belief $\bm{b}=\{\bm{s}_i\}$, fused point cloud $\mathcal{M}'_c$, throushold$~ C_{threshold}$
\Ensure  A set of potential actions $\{\bm{a}_i\}$ satisfy the following condition  $Pr[g(\bm{a}_i,~\bm{s}_i,~\alpha_i(\mathcal{P}))  > 0, \forall i] \geq C_{threshold}$
\State Transform odds value for the grid world as an image and save as a repeated tensor $T_{odd}$.
\State Repeats the fused point cloud and saves it as a tensor $T_{point}$.
\State Samples a class of states $\{\bm{s}_i\}$.
\State $\{\bm{s}_r\},~\{\bm{g}_{o^0}\}\gets\{\bm{s}_i\}$
and gets the robot state tensor $T_{r}$ and 8 grids odds tensor $T_{g}$.
\State Gets the predicted mean $T_{\mu}$ and variance tensors $T_{\sigma}$ based on $T_{r}$,~$T_{odd}$,~$T_{point}$,~$T_{g}$ and the trained network $nn(\mu(\bm{a}),~\sigma(\bm{a}))$.
\For{$\mu(\bm{a}_i)\gets T_{\mu}$, $\sigma(\bm{a}_i)\gets T_{\sigma}$}
\State $\beta^{*} \gets
(2\log(1/(1-C_{threshold})) ^{\frac{1}{2}}$
\State Check $\mu(\bm{a}_{i})>\beta^{*}\sigma(\bm{a}_{i})$ as well as $\mu(\bm{a})>10\%$ and collect the ones satisfying these two conditions to $\{\bm{a}_i\}$.
\If{satisfying the number limitation}
\State \Return selected action set $\{\bm{a}_i\}$.
\EndIf
\EndFor
\State \Return action set $\{\bm{a}_i\}$ with the highest score.
\end{algorithmic}
\end{algorithm}

\textbf{Action clustering and list initial update before MCTS:}
Our key idea is to identify the feasible region $\mathcal{X}$ using Algorithm~\ref{alg:cap4x} and update the action list associated with belief nodes (Update$(\cdot)$). Due to the complexity of real-world problems, $\mathcal{X}$ often has intricate shapes. A practical and efficient method is to get a lot of potential actions $\{\bm{a}_i\}$ with a high probability of satisfying $g(\bm{a}_i,~\bm{s},~\alpha(\mathcal{P}))> 0$ based on NP filtering.  Firstly, we quickly predict mean and variance values based on the highly-dimensional tensor format with only one-step network evaluation and then test values based on Theorem~\ref{t_add_new}. Typically, efficient actions in $\mathcal{X}$  form scattered clusters across the continuous action domain $\mathcal{A}_c$, as shown by the purple, red, and green regions in Fig.~\ref{2}, reflecting the natural continuity of primitive action physical properties.  We use the Elbow method~\cite{thorndike1953belongs} to determine the optimal number of clusters and each cluster is enclosed within high-dimensional hyperspheres with centers  $\mathcal{C}_a$ and radii $\mathcal{R}_a$. Hope the feasible region $\mathcal{X}$ can be a subset of the space covered by high-dimensional hyperspheres, such that $\mathcal{X} \subseteq \bigcup_i \mathcal{H}(\bm{center}_i,~range_i)$, where $\bm{center}_i \in \mathcal{C}_a$ and $range_i \in \mathcal{R}_a$. The centers and radii are recursively added to the action list $\mathcal{L}(\bm{node}_o)$ for observation nodes by traversing the entire tree.

\begin{figure}[!htb]
	\centering
\includegraphics[width=\linewidth]{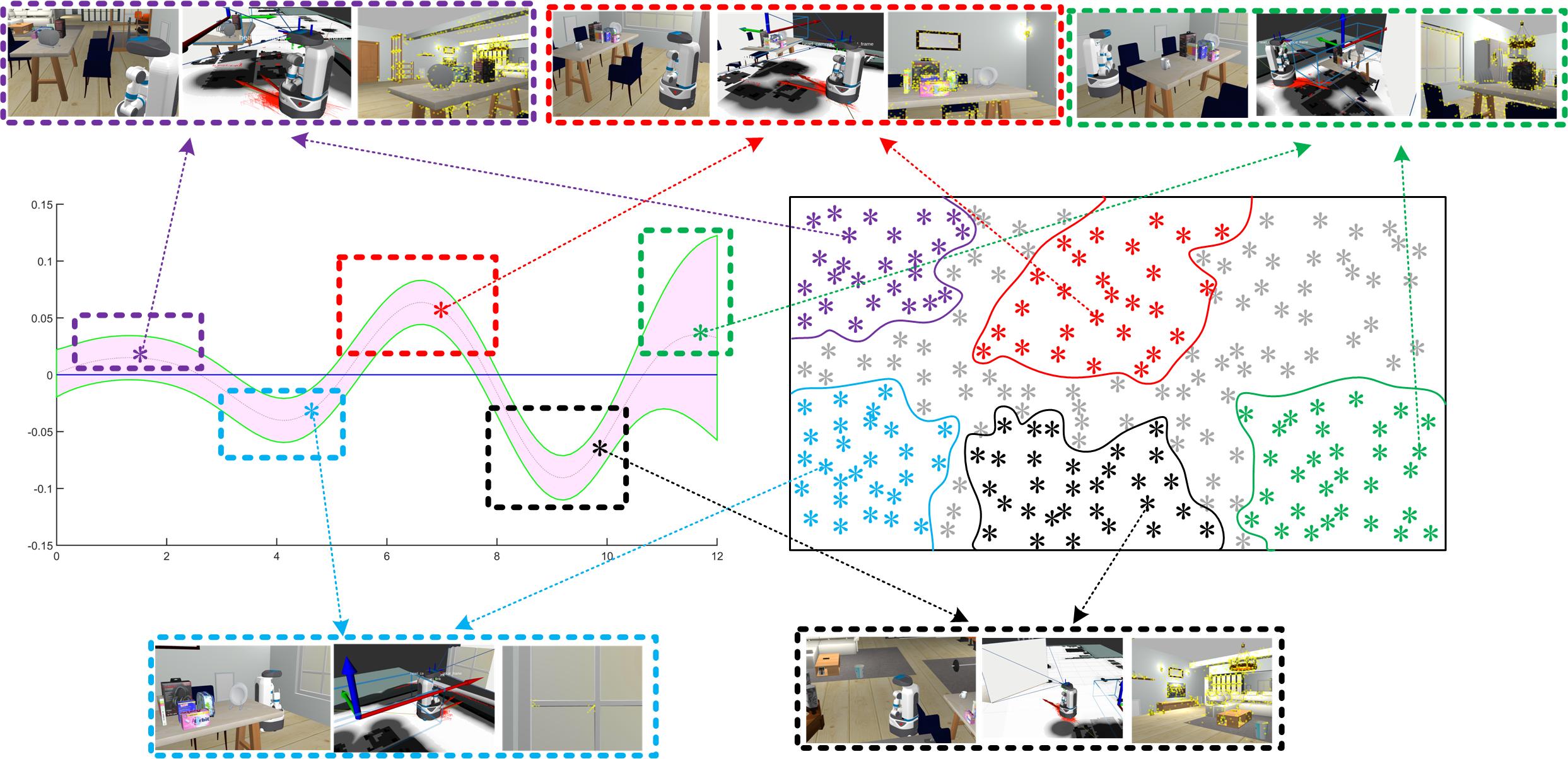}
\caption{The neural process-based scoring function, some scattered clusters, and some samples.}
\label{2}
 \end{figure}

\subsubsection{Action selection strategy and list growing in MCTS.}\label{s433}
Inspired by the classical hierarchical optimistic optimization (HOO) idea, which comes from the continuous-arm bandit problem, we select an action from the set of candidate actions $\mathcal{L}(\bm{node}_o)$ according to:
\begin{equation}\label{list1}
	\begin{aligned}
	&\bm{a}^* = \text{argmax}_{\bm{a}\in \mathcal{L}(\bm{node}_o)} U (\bm{node}_o, \bm{a}), \\
&U (\bm{node}_o, \bm{a}) = \hat Q (\bm{node}_o, \bm{a})  \\&+\omega_1 \sqrt{\log \frac{N (\bm{node}_o)}{N (\bm{node}_o,~\bm{a})}}+ \omega_2 range_i \\&+ f_{\text{option}}(\omega_3 BD), 
 \end{aligned}
\end{equation}
where $\hat Q (\bm{node}_o, \bm{a})$ represents the average reward received during rounds when the corresponding action node was selected; $U (\bm{node}_o, \bm{a})$ denotes the upper-confidence bound for the maximum possible Q-value within the hypersphere region $\mathcal{H}(\bm{center}_i,~range_i)$, akin to the UCB1 bound. This bound incorporates the effect of $\text{range}_i$ for the $i$-th hypersphere, adhering to the Lipschitz assumption. $N (\bm{node}_o)$ and $N (\bm{node}_o,~\bm{a})$ represent the visitation counts for the observation node $\bm{node}_o$ and its corresponding action node, respectively. The coefficients $\omega_1$, $\omega_2$, and $\omega_3$ adjust the relative contributions of different components. The $BD$ term refers to the belief distance associated with the growing state space, which is detailed in Section~\ref{s642}. This term is zero until a new object is detected. The belief difference heavily depends on the specific problem settings. It may disrupt the regret bounds derived for policies using UCB1 in multi-armed bandit problems~\cite{auer2002finite} and the bias correction for estimated expected payoffs in MDPs~\cite{Kocsis2006Bandit}. Hence, the $BD$ term is optional for users and denoted as $f_{\text{option}}(\star)$. We recommend setting $f_{\text{option}}(\star) = \star$ when the problem exhibits significant changes and the belief distance is large. For more common scenarios~\footnote{In our experiments, adding or removing the $BD$ term does not yield significant differences in the results.}, users may choose $f_{\text{option}}(\star) = 0$, simplifying the action selection to a UCB1-like term. This case is discussed further regarding convergence in Section~\ref{s65}. It is noted that we cannot guarantee its convergence to the optimal solution if selecting $f_{\text{option}}(\star)=\star$. Within the planning budget, episodes begin at the root node, using the action selection strategy in Eq.~\eqref{list1} to select an action node or expand the belief tree $\mathcal{T}$. If an action node associated with a high-dimensional hypersphere $\mathcal{H}(\bm{center}_i,~range_i)$ is visited more than an adaptive threshold, it is refined and divided into smaller clusters and hyperspheres:

\begin{equation}\label{list2}
	\begin{aligned}
	&N (\bm{node}_o, \bm{a}) \geq \frac{1}{C_r range_i^2}, 
 \end{aligned}
\end{equation}
where $C_r$ is a self-defined exploration constant and $N (\bm{node}_o, \bm{a})$ here provides a rough estimate of the quality of the reward estimation $\hat Q (\bm{node}_o, \bm{a})$, which follows the adaptive refining strategy in~\cite{hoerger2022adaptive} to limit the growth of
the dynamic list $\mathcal{L}(\bm{node}_o)$ and ensures that a hypersphere is only refined when this action node has been visited sufficiently often. We also constraint the refining accuracy and limit the node number corresponding to the dynamic list $|\mathcal{L}(\bm{node}_o)|$ by $range_i\geq D_{lim}$, where $D_{lim}$ is the minimum radius for partitioning.

\setcounter{footnote}{0}

Assuming the action node $\bm{node}_{a^*}$ containing $N (\bm{node}_o, \bm{a}^*)$ episodes with the selected action $\bm{a}^*$ and hypersphere $\mathcal{H}(\bm{center}^*,~range^*)$ needs refinement, these actions in this node are divided into $k$ clusters\footnote{The optimal cluster number depends on the areas filtered by the neural process. By default, $k$ is set to 3 unless otherwise specified.} and then the hyperspheres are obtained with centers $\{\triangle \bm{center}_i\}$ and radii  $\{\triangle range_i\},~i=1,~2,\cdots,~k$ based on the KMeans algorithm. We then update the action $\bm{a}^*$ of dynamic list $\mathcal{L}(\bm{node}_o)$ and its corresponding hypersphere $\mathcal{H}(\bm{center}^*,~range^*)$ by the alternative actions set $\{\bm{a}^*,~\bm{a}_{|\mathcal{L}(\bm{node}_o)|+2,},\cdots,~\bm{a}_{|\mathcal{L}(\bm{node}_o)|+k}\}$ and new hypersphere set $\{\mathcal{H}(\triangle \bm{center}_1,~f_{lim}($ $\triangle range_1)),~\cdots,~\mathcal{H}(\triangle \bm{center}_k,~f_{lim}(\triangle range_k))\}$, where $f_{lim}(\star)$ controls the refinement rate, ensuring convergence and planning performance.
\begin{equation}\label{list2}
	\begin{aligned}
	&f_{lim}(\star)= \max(f'_{lim}(\star),~D_{lim}),~\\&f'_{lim}(\star)= \\&
 \begin{cases}
        \bar{\omega}_1range^* & \text{if } \star\geq\bar{\omega}_1range^*\\
        \star & \text{if } \bar{\omega}_2range^*<\star<\bar{\omega}_1range^*\\
       \bar{\omega}_2range^* & \text{if } \star \leq \bar{\omega}_2range^* \\
    \end{cases},
 \end{aligned}
\end{equation}
where $\bar{\omega}_1$ and $\bar{\omega}_2$ are coefficients controlling the refining velocity, with $0 \leq \bar{\omega}_2 < \bar{\omega}_1 \leq 1$. The original sub-tree with root node $\bm{node}_{a^*}$ is copied and connected to the observation node $\bm{node}_o$ as an additional child node based on actions ${\bm{a}_{|\mathcal{L}(\bm{node}_o)|+2}, \cdots, \bm{a}_{|\mathcal{L}(\bm{node}_o)|+k}}$. All nodes generated from $\bm{node}_{a^*}$ are updated based on the clustered episode IDs, shown in Algorithm~\ref{alg:cap5} and Fig.~\ref{3x}. 

\begin{figure}[!htb]
	\centering
\includegraphics[width=\linewidth]{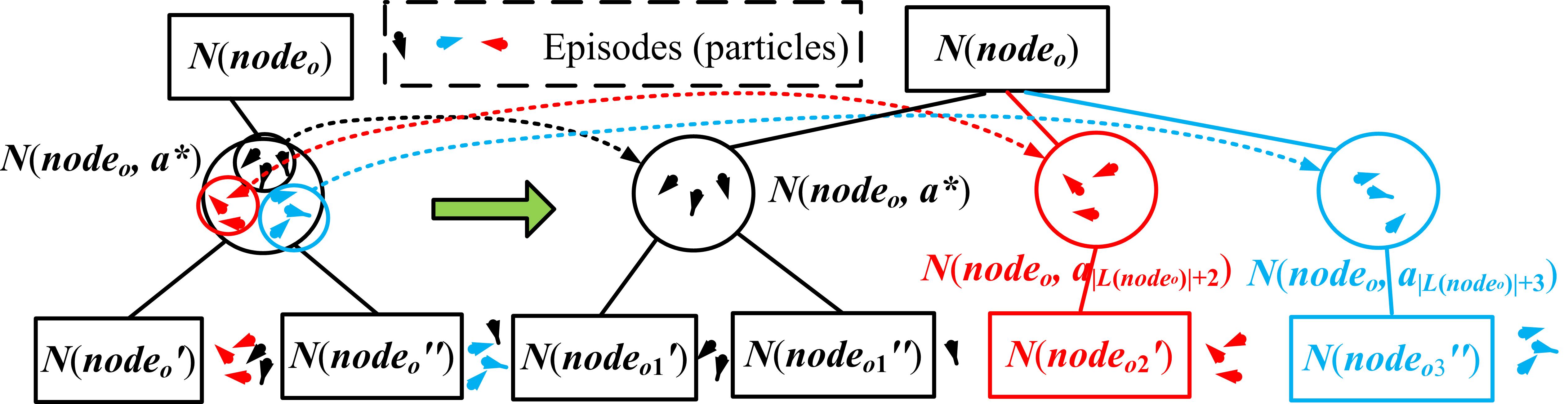}
\caption{The refining process using clustered episode IDs.}
\label{3x}
 \end{figure}

\begin{algorithm}
\small
\caption{Refine($\mathcal{T}$,~$\bm{node}_o$,~$\bm{a}^*$,~$H$)}\label{alg:cap5}
\begin{algorithmic}[1]
\Require The belief tree $\mathcal{T}$, the belief node $\bm{node}_o$, the selected action $\bm{a}^*$, the recorded action-observation history list $H$
\Ensure  The updated history list $H$ and the new tree $\mathcal{T}$ with the refined nodes
\State Collects all applied actions $\mathcal{S}_r'=\{\bm{a}_i\}$ in previous episodes passed leaf node of belief node $\bm{node}_o$ with its action $\bm{a}^*$;
\If{$N (\bm{node}_o, \bm{a}^*) \geq \frac{1}{C_r range_i^2}$ \text{and} $range_i > D_{lim}$} 
\State $\{\triangle \bm{center}_i\},~\{\triangle {range}_i\}$, clustered episode IDs~$\gets k$-Clustering$(\mathcal{S}_r')$
\State Update the dynamic list $\mathcal{L}(\bm{node}_o)\gets \mathcal{L}(\bm{node}_o)\bigcup$ $\{\bm{a}_{|\mathcal{L}(\bm{node}_o)|+2,},~\cdots,~\bm{a}_{|\mathcal{L}(\bm{node}_o)|+k}\}$
\State $\mathcal{H}(\bm{center}_i,~f_{lim}( range_i)) \gets \mathcal{H}(\triangle \bm{center}_1,$ $f_{lim}(\triangle range_1))$
\State $\{\mathcal{H}(\bm{center}_i,~f_{lim}(range_i))\} \gets \{\mathcal{H}(\bm{center}_i,$ $f_{lim}(range_i))\}\bigcup\{\mathcal{H}(\triangle \bm{center}_2,~f_{lim}(\triangle range_2)),\cdots,$ $\mathcal{H}(\triangle \bm{center}_k,~f_{lim}(\triangle range_k))\}$
\State Pick out sub-tree $\mathcal{T}_{sub}(\bm{node}_o,~\bm{a}^*)$ from $\mathcal{T}$ corresponding to the belief node $\bm{node}_o$ with its action $\bm{a}^*$;
\State Copy and generate new sub-trees based on sub-trees $\mathcal{T}_{sub}(\bm{node}_o,~\bm{a}^*)$ by list $\mathcal{L}(\bm{node}_o)$.
\State Revise all nodes of new sub-trees based on clustered episode IDs, as shown in Fig.~\ref{3x}
\State Attach generated sub-trees to the belief node $\bm{node}_o$.
\For{Find $\bm{h}$ in $H$ based on key: clustered episode ID}
\State Reach the observation $\bm{o}$ based on belief node $\bm{node}_o$
\State Update the next action $\bm{a}^*$ in $\bm{h}$ by the corresponding element $\bar{\bm{a}}$ in dynamic list $\mathcal{L}(\bm{node}_o)$, satisfying $\bm{a}^*\gets \bar{\bm{a}}\in \{\bm{a}^*,~\bm{a}_{|\mathcal{L}(\bm{node}_o)|+2,},~\cdots,~\bm{a}_{|\mathcal{L}(\bm{node}_o)|+k}\}$
\EndFor
\State \Return updated tree $\mathcal{T}$ and the updated history list $H$.
\EndIf 
\end{algorithmic}
\end{algorithm}

\subsubsection{Action sampling.}\label{s634}
Since the selected action $\bm{a}^*$ corresponds to the hypersphere $\mathcal{H}(\bm{center}^*,~range^*)$, in order to execute the action, we assume that the sampled action is uniformly distributed in the hypersphere $\mathcal{H}(\bm{center}^*,~range^*)$ and then sample a discrete action in this hypersphere, similar to ellipsoid sampling.

\subsubsection{Backup.}\label{s334}
When each episode reaches the terminal state, our GNPF-$k$CT framework updates the estimation reward $\hat{Q}(\bm{node}_o,~\bm{a})$ as well as the visited numbers $N(\bm{node}_o)$ and $N(\bm{node}_o,~\bm{a})$ of all nodes visited by this episode. Here, we present two classical stochastic backup methods including the Bellman backup (Algorithm~\ref{alg:cap7}), which is used
in the ABT method and similar to the rule used in Q-learning, and the Monte-Carlo backup (Algorithm~\ref{alg:cap7}), which is widely used in many outstanding POMDP solvers, like POMCP, POMCPOW, and Voronoi Optimistic Monte Carlo Planning with Observation Weighting (VOMCPOW)~\cite{Michael2021Voronoi}. The Bellman update naturally follows the objective function of the POMDP formulation that aims to pursue optimal action in each step of the long-term planning. It helps the solver to explore deeper by focusing its search on promising parts of the belief tree. The main challenge for the Bellman backup is when facing unexpected observations, a lot of deeply explored belief trees will be frequently cut and this case causes poor planning performance. Hence, the Bellman backup gets a better performance when the good rewards are sparse in the belief tree, but it is not stable enough for the poor observation prediction. In contrast to selecting the reward with optimal action, the Monte-Carlo backup computes the average reward along with different action episodes, which means that the generated belief tree will be more robust when facing unexpected uncertainty in received observation. In our application for object search, the real visual observation is not well predictable for the observation model in the POMDP formulation and the robot camera will frequently receive unexpected measurements, which breaks the advantage of the Bellman update. We use the Monte-Carlo backup in our problem but the Bellman backup is commonly superior in other applications, so we present both two backup ways here.
\begin{algorithm}
\small
\caption{Backup($\mathcal{T}$, $\bm{node}_o$, $\bm{a}$, $r$, $R$)}\label{alg:cap7}
\begin{algorithmic}[1]
\Require The belief tree $\mathcal{T}$, the observation mode $\bm{node}_o$, the selected action $\bm{a}$, the accumulated reward $r$, and the immediate reward $R$, optional type for backup: Bellman or Monte-Carlo.
\Ensure  The updated tree $\mathcal{T}$ with the refined nodes
\State $N(\bm{node}_o)\gets N(\bm{node}_o)+1$ and $N(\bm{node}_o,~\bm{a}^*)\gets N(\bm{node}_o,~\bm{a}^*)+1$
\If{Use Monte-Carlo backup}
\State $\hat Q (\bm{node}_o, \bm{a}^*) \gets \hat Q (\bm{node}_o, \bm{a}^*) + \frac{r - \hat Q (\bm{node}_o, \bm{a}^*)}{N(\bm{node}_o,~\bm{a}^*)}$
\Else \Comment{\textcolor{mygray}{Bellman backup.}}
\State $\bm{node}_o'$ is the child of $\bm{node}_o$
\State $\hat Q (\bm{node}_o, \bm{a}^*) \gets \hat Q (\bm{node}_o, \bm{a}^*) + \frac{R + \gamma \hat V^*(\bm{node}_o') - \hat Q (\bm{node}_o, \bm{a}^*)}{N(\bm{node}_o,~\bm{a}^*)}$
\State $\hat V^*(\bm{node}_o) \gets \max_{\bm{a}\in \mathcal{L}(\bm{node}_o)} \hat Q (\bm{node}_o, \bm{a})$
\EndIf
\end{algorithmic}
\end{algorithm}

\subsection{GNPF-$k$CT: Reusing the belief tree}\label{s64}

In constructing the belief tree, three history lists are maintained: state history $H_s$, immediate rewards $H_r$, and action-observation pairs $H$. The depth vector $dim_r(H_s)$ for terminal nodes from the previous tree is also tracked (lines 5 and 16, Algorithm~\ref{alg:cap4}). When the robot observes new objects, leading to non-empty states $s_{add}$, the new observation cannot directly map to a belief node. Sub-observations of newly detected objects are removed (line 28, Algorithm~\ref{alg:cap4}), and the remaining observation vector $\bm{o}$ with action $\bm{a}^*$ identifies the root node among child nodes of the prior tree $\mathcal{T}$. This process also trims $H_s$, $H_r$, $H$, and $dim_r(H_s)$ (Algorithm~\ref{alg:cap8}). Using the reduced histories, for each particle, we can update all states in the state history list $H_s$ and the action-observation history $H$ by adding sub-state vectors and adjusting observations for newly detected objects based on the transition and observation functions, as shown in Algorithm~\ref{alg:cap9}. Finally, the updated histories grow the tree structure (Algorithm~\ref{alg:euclid101}).

\subsubsection{Cutting and updating state list and depth vector.} 

The key to reusing the previous belief tree is maintaining and updating stored information, including $H_s$, $H_r$, $H$, and $dim_r(H_s)$, in the sampled particles. The first step refines action nodes (Algorithm~\ref{alg:cap5}). When an action, $\bm{a}^*$, is refined into multiple new actions (e.g., $\bm{a}^*$ to $\{\bm{a}^*,~\bm{a}_{|\mathcal{L}(\bm{node}_o)|+2,},~\cdots,~\bm{a}_{|\mathcal{L}(\bm{node}_o)|+k}\}$), the history list $H$ is updated accordingly. The second step prunes branches after executing $\bm{a}^*$ and receiving observation $\bm{o}$ (Algorithm~\ref{alg:cap8}). For the new state $(\bm{s}, \bm{s}_{add})$, branches corresponding to the old state $\bm{s}$ in the previous tree $\mathcal{T}$ are assigned as the root of the new tree. The last significant operation is to update the state $\bm{s}'$ by the added state vector $\bm{s}_{add}'$ corresponding to new objects based on the transition function $T((\bm{s},~\bm{s}_{add}), \bm{a}, \bm{s}_{add}')$ and observation function $Z((\bm{s},~\bm{s}_{add}'), \bm{a}, \bm{o}')$ for each particle\footnote{It is noted that because our current considered actions $\bm{a}_1$ and $\bm{a}_2$ just perform on the old objects, the reward function does not need to update. Our reward function follows $R((\bm{s},~\bm{s}_{add}), \bm{a}_1, (\bm{s}',~\bm{s}_{add}')) = R(\bm{s}, \bm{a}_1, \bm{s}')$, because the old action $\bm{a}_1$ and $\bm{a}_2$ just operate one object each step.}, where $\bm{o}'$ is the new observation including both the old objects and newly detected objects.  In the state update, we consider the effect $\bm{s}' \rightarrow \bm{s}_{add}'$ from the old state vector $\bm{s}$, corresponding to old objects, to the added state vector $\bm{s}_{add}'$, which is related to the newly detected objects. However, the old state vector $\bm{s}'$ itself is not updated to preserve and reuse existing information. In other words, the belief of the old state \textbf{does not} consider the effect from the new object $\bm{s}_{add}' \centernot\rightarrow \bm{s}'$. This process ensures updated histories for particles (Algorithm~\ref{alg:cap9}) while adhering to the allocated planning budget.

\begin{algorithm}
\small
\caption{\text{Cutting}($\bm{a}^*$,~$\bm{o},~\bar{\mathcal{H}}$)}\label{alg:cap8}
\begin{algorithmic}[1]
\Require The selected action $\bm{a}^*$, the received observation $\bm{o}$, the history tuple $\bar{\mathcal{H}}$
\Ensure  The updated history tuple $\bar{\mathcal{H}}$
\State $H,~H_s,~H_r,~dim_r(H_s)\gets{\bar{\mathcal{H}}}$
\State $H'\gets{\emptyset}$, $H_s'\gets{\emptyset}$, $H_r'\gets{\emptyset}$, $dim_r(H_s)'\gets{\emptyset}$
\For{$\bm{h}$~in~$H$}
\If{$\bm{h}[0]==\bm{a}^*$ and $\bm{h}[1]==\bm{o}$}
\State Get $key$ for $\bm{h}$ in $H$
\State Find states history $\bm{h}_s$ and corresponding depth value $dim_r(H_s)[key]$ based on hash table key $key$
\State $\bm{h} \gets \bm{h}/\{\bm{h}[0],\bm{h}[1]\}$, $\bm{h}_s \gets \bm{h}_s/\{\bm{h}_s[0]\}$, $\bm{h}_r \gets \bm{h}_r/$ $\{\bm{h}_r[0]\}$, $dim_r(H_s)$ $ [key]\gets dim_r(H_s)[key]-1$ \Comment{\textcolor{mygray}{$\star/\bullet$ means to remove sets or elements $\bullet$ from a set $\star$.}}
\State $H'\gets H'\bigcup\{\bm{h}\}$, $H_s'\gets H_s'\bigcup\{\bm{h}_s\}$, $H_r'\gets H_r'\bigcup$ $\{\bm{h}_r\}$, $dim_r(H_s)'\gets $ $dim_r(H_s)'\bigcup dim_r(H_s)[key]$
\EndIf
\EndFor
\State $H\gets H$, $H_s\gets H_s'$, $H_r\gets H_r'$, $dim_r(H_s)\gets dim_r(H_s)'$
\State \Return $\bar{\mathcal{H}}\gets<H$, $H_s$, $H_r$, $dim_r(H_s)>$.
\end{algorithmic}
\end{algorithm}

\begin{algorithm}
\small
\caption{Episode\_simulation\_along\_list($\bar{\mathcal{H}},~\bm{s}_{add}$)}\label{alg:cap9}
\begin{algorithmic}[1]
\Require The history tuple $\bar{\mathcal{H}}$, the initial state for newly detected objects $\bm{s}_{add}$
\Ensure  The updated tuple $\bar{\mathcal{H}}$.
\State $H,~H_s,~H_r,~dim_r(H_s)\gets{\bar{\mathcal{H}}}$
\State $H'\gets{\emptyset}$, $H_s'\gets{\emptyset}$, $H_r'\gets{\emptyset}$, $dim_r(H_s)'\gets{\emptyset}$
\For{$\bm{h}_s$~in~$H_s$}
\State Get $key$ for $\bm{h}_s$ in $H_s$ \Comment{\textcolor{mygray}{Get the hash table key.}}
\State $\bm{h}\gets H[key]$ \Comment{\textcolor{mygray}{Get the action-observation history based on hash table key $key$.}}
\State $i\gets 0$, $\bm{h}_s'\gets \emptyset$, $\bm{h}'\gets \emptyset$
\For{$i<|\bm{h}_s|$}
\State $\bm{s}\gets \bm{h}_s[i]$, $\bm{a}\gets \bm{h}[2*i]$
\State $\bm{s}_{add}'\gets T((\bm{s},~\bm{s}_{add}), \bm{a}, \bm{s}_{add}')$ \Comment{\textcolor{mygray}{Transition function to get the states about newly detected object.}}
\State $\bm{o}'\gets Z((\bm{s},~\bm{s}_{add}'), \bm{a}, \bm{o}')$ \Comment{\textcolor{mygray}{Observation function to get the new observations $\bm{o}'$ considering the extended state.}}
\State {$\bm{h}_s' \leftarrow \bm{h}_s' \bigcup \{(\bm{s}', \bm{s}_{add}')$\}}
\State {$\bm{h}' \leftarrow \bm{h}' \bigcup \{(\bm{a}, \bm{o}')$\}}
\EndFor
\State $H_s'\gets H_s'\bigcup\{\bm{h}_s'\}$, $H'\gets H'\bigcup\{\bm{h}'\}$, $H_r'\gets H_r'\bigcup\{H_r$ $[key]\}$, $dim_r(H_s)'\gets $ $dim_r(H_s)'\bigcup dim_r(H_s)[key]$  \Comment{\textcolor{mygray}{The recorded immediate reward and depth will not be changed.}}
\If{planning budget exceeded} 
\State break \Comment{\textcolor{mygray}{limitation. Left some time for resampling.}{This planning budget is commonly related to the whole planning limitation. Left some time for resampling.}}
\EndIf
\EndFor
\State $H\gets H'$, $H_s\gets H_s'$, $H_r\gets H_r'$, $dim_r(H_s)\gets dim_r(H_s)'$
\State \Return $\bar{\mathcal{H}}\gets<H$, $H_s$, $H_r$, $dim_r(H_s)>$.
\end{algorithmic}
\end{algorithm}

\subsubsection{Rebuild the belief tree and resampling.}\label{s642}
Starting from the identified root node, the new belief tree $\mathcal{T}'$ is grown using the updated particles. This growth process follows the same exploration principles as the tree structure in MCTS sampling, leveraging the pre-updated particles. The main steps are outlined in Algorithm~\ref{alg:euclid101} and Algorithm~\ref{alg:euclid102}. Notably, new belief nodes may be introduced while constructing the new belief tree $\mathcal{T}'$. Consequently, $\mathcal{T}'$ may not be a strict sub-tree of the previous belief tree $\mathcal{T}$. Based on $\mathcal{T}'$, MCTS is then performed for all branches, including those with new actions and observations. The action selection strategy adheres to the UCB principle, incorporating the hypersphere range $\omega_2 range_i$, and the belief distance $\omega_3 BD$, as detailed in Section~\ref{s433}. Let $\bm{s}_{new}= (\bm{s},~\varnothing)$ and $\bm{s}_{new}' = (\bm{s}',~\bm{s}_{add}')$, we know that the current optimal action should be selected following equation Eq.~\eqref{smooth-1}. Assuming our belief approximation is perfect outside the sub-tree with $depth\in \{1,~2,~3,\cdots\}$ layer~\footnote{This assumption is reasonable in our object search problem. Because, after they reach the threshold of the declaring action, different state vectors corresponding this object are considered the same in its physical meaning, the belief differences will become into 0. }, of which one layer has action nodes and observation nodes, in the sub-tree of the new belief tree shown in Fig.~\ref{f2}, for $depth=1$, the only difference between two ways (ours and the direct resampling way) is from the last terms $V^*(\bm{b}(\bm{s}_{new}': \bm{s}'\rightleftarrows \bm{s}_{add}'))$ (correct by building tree) and $V^*(\bm{b}(\bm{s}_{new}': \bm{s}'\rightarrow \bm{s}_{add}'))$ (ours, without considering partial effect). Hence, we add the belief distance term, which is defined as $BD \triangleq \sum_{o\in O}Z(\bm{s}_{new}', \bm{a}, \bm{o}') \|\bm{b}((\bm{s}',~\bm{s}_{add}'): s'\rightarrow \bm{s}_{add}') -  \bm{b}((\bm{s}',~\bm{s}_{add}'): s'\rightleftarrows \bm{s}_{add}')\|_1$, to get a better estimation for the upper confidence bound.

\begin{figure*}[!htb]
   \begin{equation}\label{smooth-1}
   \small{
	\begin{aligned}
	\max_{\bm{a}\in A} \left(\sum_{\bm{s}_{new}\in S} R(\bm{s}_{new}, \bm{a}, \bm{s}_{new}') \bm{b}(\bm{s}_{new})+  \gamma\sum_{o\in O}\sum_{\bm{s}_{new}'\in S}Z(\bm{s}_{new}', \bm{a}, \bm{o}')T(\bm{s}_{new}, \bm{a}, \bm{s}_{new}')V^*(\bm{b}(\bm{s}_{new}': \bm{s}'\rightleftarrows \bm{s}_{add}'))\right).
	\end{aligned}}
\end{equation} 
\end{figure*}

\begin{figure}[!htb]
	\centering
	\includegraphics[width=\linewidth]{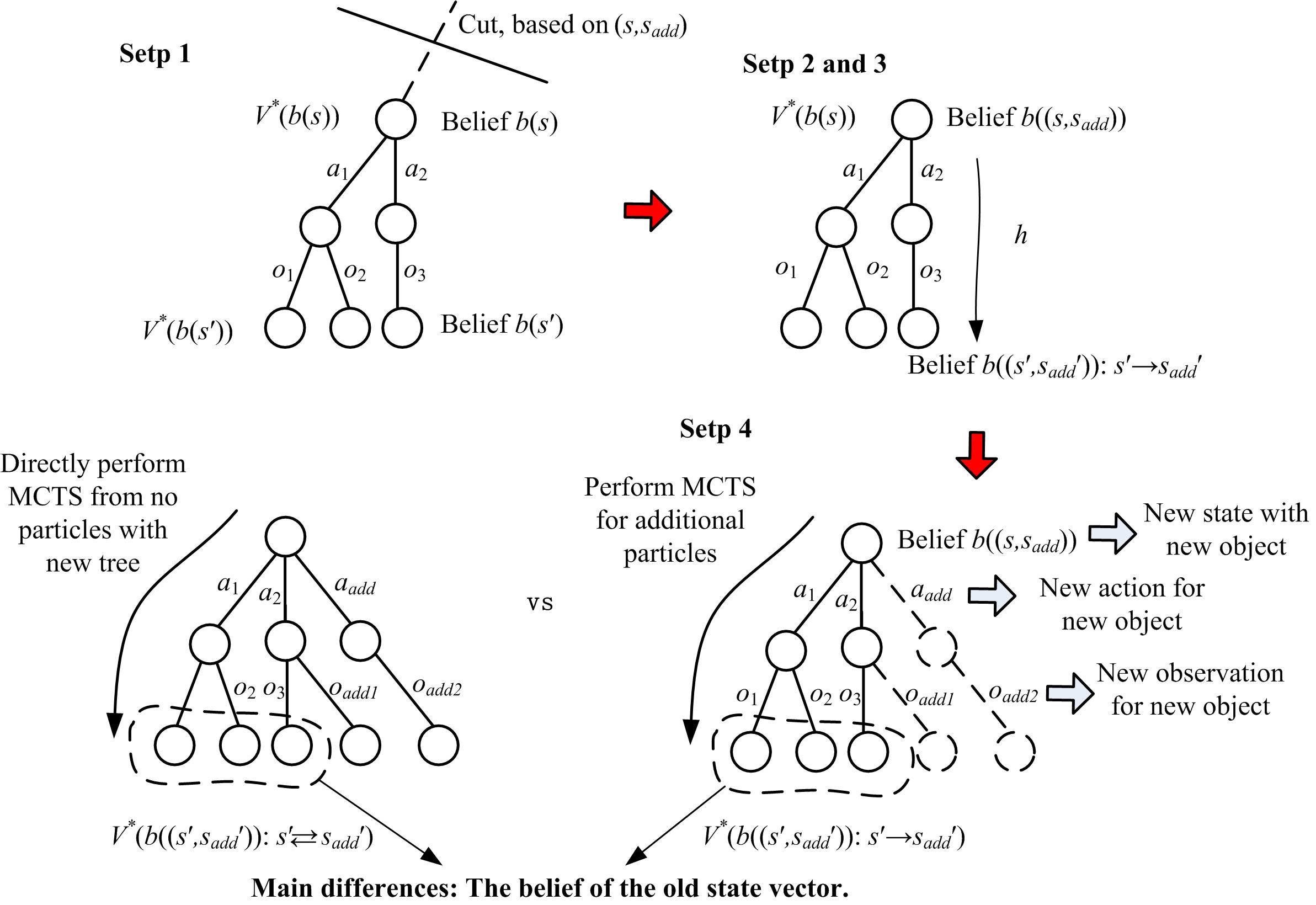}
	\caption{Main steps for growing state space} 
	\label{f2}
\end{figure}

\begin{algorithm}[!ht]  
	\small
	\caption{Grow\_tree$(\mathcal{T},~\bar{\mathcal{H}})$}  
	\label{alg:euclid101}  
 \begin{algorithmic}[1]  
			\Require{The previous belief tree $\mathcal{T}$, the history tuple $\bar{\mathcal{H}}$}
			\Ensure{The updated tree $\mathcal{T}$}
            \State $H,~H_s,~H_r,~dim_r(H_s)\gets{\bar{\mathcal{H}}}$
            \State Get the root node $\bm{node}_o$ from $\mathcal{T}$
            \For{$\bm{h}_s$~in~$H_s$}
            \State Get $key$ for $\bm{h}_s$ in $H_s$ \Comment{\textcolor{mygray}{Get the hash table key.}}
            \State $\bm{h}\gets H[key]$, $\bm{h}_r\gets H_r[key]$ \Comment{\textcolor{mygray}{Get the action-observation history $\bm{h}$ and reward hisotry $\bm{h}_r$ based on hash table key $key$.}}
            \State $i \gets 0$
            \State $\mathcal{T},r\gets$Simulation\_new\_objects($\bm{node}_o$, $\bm{h}_s$,~$\bm{h}_r$,~$\bm{h}$, $dim_r$ $(H_s)[key]$,~$i$)
            \EndFor
	\end{algorithmic}
\end{algorithm}

\begin{algorithm}[!ht]  
	\small
	\caption{Simulation\_new\_objects($\bm{node}_o$, $\bm{h}_s$,~$\bm{h}_r$,~$\bm{h}$,~$i$)}  
	\label{alg:euclid102}  
 \begin{algorithmic}[1]  
			\Require{The previous belief node $\bm{node}_o$, the updated histories in one particle $\bm{h}_s,~\bm{h}_r,~\bm{h}$, the depth value $i$}
			\Ensure{The updating tree $\mathcal{T}$, the discounted accumulated reward $r$}
            \State {$\bm{a}^*\gets \bm{h}[2*i]$}
            \If {$i == dim_r(H_s)[key]$} \Comment{\textcolor{mygray}{Reach the terminal node in tree structure and the later part in this history is corresponding to rollout.}}
            \For{\textbf{all} $\bm{a} \in \mathcal{A}$}
            \State {Action node $\bm{node}_a\leftarrow(N(\bm{node}_o, \bm{a}), $ $V(\bm{node}_o,$ $\bm{a})$, $\varnothing)$} \Comment{\textcolor{mygray}{Generate the empty action nodes.}}
            \EndFor
            \State \Return {Rollout\_new\_objects$(\bm{h}_r, i+1)$} \Comment{\textcolor{mygray}{See Algorithm~\ref{alg:euclid11}.}}
            \Else 
            \If {no action nodes from $\bm{node}_o$ and }
            \For{\textbf{all} $\bm{a} \in \mathcal{A}$}
            \State {$\bm{node}_a \leftarrow (N(\bm{node}_o, \bm{a}), V(\bm{node}_o,$ $\bm{a}), \varnothing)$} \Comment{\textcolor{mygray}{Generate the empty action nodes.}}
            \EndFor
            \EndIf
            \State  $\bm{s}_{new}\gets \bm{h}_s[i]$, $\bm{s}_{new}'\gets \bm{h}_s[i+1]$, $\bm{o}'\gets \bm{h}[2*i+1]$, $R(\bm{s}_{new},~\bm{a}^*,~\bm{s}_{new}')\gets \bm{h}_r[i]$
            \State belief node $\bm{node}_{o'} \gets (N(\bm{node}_{o'}), V(\bm{node}_{o'}))$ based on $\bm{node}_a$ and $\bm{o}'$
            \If {$i < \left|h'\right|_{rollout}$}
            \State  {$r\leftarrow  \gamma$ Simulation\_new\_objects($\bm{node}_{o'}$, $\bm{h}_s$, $\bm{h}_r$, $\bm{h}$, $i+1$)} \Comment{\textcolor{mygray}{Recursively}}
            \State  {Collect $\{\bm{s}_{new}\}$ in the belief node $\bm{node}_{o}$} 
            \State Backup($\mathcal{T}$,~$\bm{node}_o$,~$\bm{a}^*$,~$r$,~$R(\bm{s}_{new},~\bm{a}^*,~\bm{s}_{new}')$)  \Comment{\textcolor{mygray}{See Algorithm~\ref{alg:cap7}.}}
            \EndIf
            \EndIf
	\end{algorithmic}
\end{algorithm} 

\begin{algorithm}[!ht]  
    \small
	\caption{Rollout\_new\_objects$(\bm{h}_r, j)$}
	\label{alg:euclid11}  
	\begin{algorithmic}[1]  
			\Require{The reward history $\bm{h}_r$, the depth $j$}
			\Ensure{Discounted total roll-out reward $r$}
            \If  {$j == \left|\bm{h}_r\right|$}
            \State \Return $\bm{h}_r[j]$
            \Else
            \State \Return $\bm{h}_r[j] + \gamma $Rollout\_new\_objects$(\bm{h}_r, j+1)$
            \EndIf
	\end{algorithmic} 
\end{algorithm} 
\subsection{GNPF-$k$CT: Theoretical analysis}\label{s65}
In this section, we aim to analyze the convergence of the proposed solver with some assumptions to improve the compactness of this work\footnote{In all these discusses, the belief distance term $f_{\text{option}}(BD)$ is set to be 0.}. The focus is to answer:
\begin{question}\label{Q1}
Does the GNPF-$k$CT algorithm converge in probability to the optimal value function in POMDPs?
\end{question}
Let's recall our key steps in GNPF-$k$CT algorithm related to the convergence, including network filtering, $k$ clustering for hierarchical partition, and the revised UCB strategy. It is easy to know that the prediction accuracy of the neural network will directly affect the performance of the method. 

\begin{assumption}\label{a11}
The neural network used for action filtering does not filter the optimal action. The obtained feasible region $\mathcal{X}$ satisfies $\bm{a}^*\in \mathcal{X},~\bm{a}^* = \textbf{argmax}_{\bm{a}\in \mathcal{A}}Q(\bm{b},~\bm{a})$.
\end{assumption}

Satisfying assumption~\ref{a11}, the network filtering will not affect the convergence of the GNPF-$k$CT algorithm to the optimal result. The problem just has a smaller action domain with the same configuration. To answer Question~\ref{Q1} is equal to answer Question~\ref{Q2}:

\begin{question}\label{Q2}
Without considering network filtering, does the GNPF-$k$CT algorithm with $k$-center clustering and the revised UCB strategy converge in probability to the optimal value function in POMDPs?
\end{question}

Based on \cite{silver2010monte}, we can answer 
Question~\ref{Q2} by considering POMDPs as a derived MDP. Let's consider Lemma~\ref{l1} about the value function and  Lemma~\ref{l2} about the rollout distribution~\cite{silver2010monte}:

\begin{lemma}\label{l1}
Given a POMDP $\mathcal{M}=<\mathcal{S},~\mathcal{A},~\mathcal{O},~T,~Z,~R>$ consider the derived MDP with histories $\mathcal{H}$ as states, $\widetilde{\mathcal{M}}=<\mathcal{H},~\mathcal{A},~\widetilde{T},~\widetilde{R}>$, where $\widetilde{T}^{\bm{a}}_{h,hao}=\sum_{\bm{s}\in\mathcal{S}}\sum_{\bm{s}'\in\mathcal{S}}\bm{b}(\bm{s},h)T(\bm{s},~\bm{a},~\bm{s}')Z(\bm{s}',~\bm{a},~\bm{o})$, $\bm{b}(\bm{s},h)$ means the belief over states given history $h$, $hao$ means the updated history pruning the tree by $\bm{a}$ and $\bm{o}$, and $\widetilde{R}({\bm{a}},~$h$)=\sum_{\bm{s}\in\mathcal{S}}\bm{b}(\bm{s},h)R(\bm{s}, \bm{a}, \bm{s}')$. Then the value function $\widetilde{V}^\pi(h)$ of the derived MDP is equal to the value function ${V}^\pi(h)$ of the POMDP, $\forall \pi, \widetilde{V}^\pi(h)= {V}^\pi(h)$, where value function ${V}^\pi(h)$ is the expected return from state
 when following policy $\pi$.~\cite{silver2010monte}
\end{lemma}

\begin{lemma}\label{l2}
For any rollout policy $\pi$, the POMDP rollout distribution is equal to the derived MDP rollout distribution, $\forall \pi$, $D^\pi(h_T ) = \widetilde{D}^\pi (h_T )$.~\cite{silver2010monte}
\end{lemma}


From Lemmas~\ref{l1} and \ref{l2}, proving POMDP solver convergence reduces to proving convergence for the corresponding MDP solver of the driver MDP. Thus, we reformulate Question~\ref{Q2} as:

\begin{question}\label{Q3}
Does the GNPF-$k$CT algorithm converge in probability to the optimal value function in MDPs?
\end{question}

In order to connect our GNPF-$k$CT algorithm with some existing MDP solvers, we present the following assumption:
\begin{assumption}\label{assumption2}
The action partitioning results $\mathcal{X}_{d,i_j},~j=1,\cdots,k$ generated by all $k$ center clustering operations for the action domain $\mathcal{X}_{d-1,i}$ follow the properties of the hierarchical partitioning, satisfying $\mathcal{X}_{d,i_j}\bigcap\mathcal{X}_{d,i_j'}=\emptyset$ for $\forall i_j,i_j' \in \{i_1,\cdots,i_k\}$ and $\bigcup_{j=1,\cdots,k}\mathcal{X}_{d,i_j}=\mathcal{X}_{d-1,i}$.
\end{assumption}

Under Assumption~\ref{assumption2}, the $k$-center clustering can be viewed as hierarchical partitioning. With hierarchical partitioning,  our method follows the same search strategy as the HOO method but with different partitioning ways. Because we limit the refining accuracy $range_i$ and the corresponding list dimension $|\mathcal{L}(\bm{node}_o)|$, we can consider it as the same problem with the finite discrete actions and each action $\bm{a}$ is a range instead of a value. We have the following new question:

\begin{question}\label{Q3}
Does the action selection strategy~Eq.~\eqref{list1} in the GNPF-$k$CT algorithm converge in probability to the optimal value function in MDPs with discrete action domains?
\end{question}

When the number of visits $N(\bm{node}_o)$ approaches infinity, the action range $range_i$ will be the constant limitation $D_{lim}\in\mathbb{R}$. For the action selection strategy~\eqref{list1}, with a given coefficient $\omega_2$, we can ignore the region-related terms $\omega_2 range_i$ due to the same constant value for all candidate action ranges. In this way, the action selection strategy becomes the standard UCB1 bound: $\hat Q (\bm{node}_o, \bm{a}) + \omega_1 \sqrt{\frac{\log N (\bm{node}_o)}{N (\bm{node}_o,~\bm{a})}} + \omega_2 range_i\rightarrow \hat Q (\bm{node}_o, \bm{a}) + \omega_1 \sqrt{\frac{\log N (\bm{node}_o)}{N (\bm{node}_o,~\bm{a})}} + \omega_2 D_{lim}$. So it follows the convergence analysis for the UCB1 in~\cite{Kocsis2006Bandit} and~\cite{silver2010monte}, following:

\begin{lemma}\label{l3}
For a suitable choice of $\omega_1$, the value function constructed by UCT converges
in probability to the optimal value function. As the number of visits $N (node_o)$ approaches infinity, the bias of the value
function is $O(\log~N (node_o)/N (node_o))$.
\end{lemma}

This convergence result means that the method can find the optimal range action $\mathcal{A}_{opt}$ that has the largest mean value for all refined ranges with some probability. 

\begin{assumption}\label{a3x}
The obtained  mean values $\widetilde{{Q}}(\bm{b},~\mathcal{A}_{opt})$ and $\widetilde{{Q}}(\bm{b},~\mathcal{A}_{sub})$ corresponding to the optimal action range $\mathcal{A}_{opt}$ and any sub-optimal action range $\mathcal{A}_{sub}$ satisfy:

\begin{equation}\label{list2}
	\begin{aligned}
	\widetilde{{Q}}(\bm{b},~\mathcal{A}_{opt})-\widetilde{{Q}}(\bm{b},~\mathcal{A}_{sub})\geq \eta D_{lim}.
 \end{aligned}
\end{equation}
\end{assumption}

Based on Lipschitz continuous, we have any action $\bm{a}\in \mathcal{A}_{sub}$ in the sub-optimal range $\mathcal{A}_{sub}$ satisfies:  ${Q}(\bm{b},~\bm{a})\leq \widetilde{{Q}}(\bm{b},~\mathcal{A}_{sub}) + \eta D_{lim}$. Then, considering Assumption~\ref{a3x}, we will have: $\widetilde{{Q}}(\bm{b},~\mathcal{A}_{opt})\geq \eta D_{lim}+\widetilde{{Q}}(\bm{b},~\mathcal{A}_{sub})\geq {Q}(\bm{b},~\bm{a})$. Because the best action $\widetilde{\bm{a}}^*\in \mathcal{A}_{opt}$ in the optimal action range $\mathcal{A}_{opt}$ satisfy ${Q}(\bm{b},~\widetilde{\bm{a}}^*) \geq {{Q}}(\bm{b},~\mathcal{A}_{opt})$, finally, for any action $\bm{a}\in \mathcal{A}_{sub}\bigcup \mathcal{A}_{opt}$ in both  sub-optimal range $\mathcal{A}_{sub}$ and optimal range $\mathcal{A}_{opt}$, we have: $\widetilde{\bm{a}}^* = \bm{a}^*$ and ${Q}(\bm{b},~\widetilde{\bm{a}}^*)={Q}(\bm{b},~{\bm{a}}^*)\geq {{Q}}(\bm{b},~\bm{a})$, which means the obtained optimal range $\mathcal{A}_{opt}$ definitely includes the optimal action $\bm{a}^*$. In short, under given assumptions, GNPF-$k$CT probabilistically converges to a small range including the optimal solution for POMDPs with continuous action domains.

The above discussion focuses solely on the case without belief tree reuse, assuming it perfectly approximates traditional MCTS. However, belief tree reuse can be seen as a mis-estimation of action and observation node values in the correct belief tree. To evaluate its approximation performance, we present the following results:

\begin{theorem}\label{t_add_new1a}
We are considering the unweighted particles for approximating the extended full-correct belief $\bm{b}(\bm{s}_{new}': \bm{s}'\rightleftarrows \bm{s}_{add}')=\frac{1}{N}\sum^N_{i=1}\mathbb{I}(\bm{s}_{new}'={\bm{s}_{new}^i}':$ $\bm{s}'\rightleftarrows \bm{s}_{add}')$ and the reused approximated belief $\bm{b}(\bm{s}_{new}': \bm{s}'\rightarrow \bm{s}_{add}')=\frac{1}{N}\sum^N_{i=1}\mathbb{I}(\bm{s}_{new}'={\bm{s}_{new}^i}': \bm{s}'\rightarrow \bm{s}_{add}')$. Assuming $depth=1$ as well as the reward function is Borel measurable and bounded, denoted $R_{max}= \|R\|_\infty$, the belief distance L1 is limited in $\|\bm{b}(\bm{s}_{new}': \bm{s}'\rightleftarrows \bm{s}_{add}')-\bm{b}(\bm{s}_{new}': \bm{s}'\rightarrow \bm{s}_{add}')\|_1\leq \delta$, and the optimal action selected by building the whole tree is same as the one selected by approximated value $\max_{\bm{a}} {V}(\bm{b}(\bm{s}_{new}), \bm{a}) = \max_{\bm{a}} \hat{V}(\bm{b}(\bm{s}_{new}), \bm{a})=\bm{a}^*$, the optimal value function of POMDP problem using our method $V^*(\bm{b}(\bm{s}_{new}))$ and the optimal value function using the direct resampling way $\hat{V}^*(\bm{b}(\bm{s}_{new}))$ will satisfy the following bounding equation:
\begin{equation}\label{EQ_28}
\|V^*(\bm{b}(\bm{s}_{new})) - \hat{V}^*(\bm{b}(\bm{s}_{new}))\|_1 \leq \frac{\gamma R_{max}}{1- \gamma}\delta.
\end{equation}
\end{theorem}

\begin{proof}
Please refer to Appendix~\ref{a3}.
\end{proof}

\begin{corollary}\label{t_add_new3}
Considering multiple layers $depth\geq 1$ case, the extended full-correct belief $\bm{b}(\bm{s}_{new}^{depth}: \bm{s}'\rightleftarrows \bm{s}_{add}')$ and the reused approximated belief $\bm{b}(\bm{s}_{new}^{depth}: \bm{s}'\rightarrow \bm{s}_{add}')$ are defined for terminal nodes in $depth$ layer. Assuming the reward function is Borel measurable and bounded, the belief L1 distance  for each belief node is limited in $\|\bm{b}(\bm{s}_{new}^{(depth)}: \bm{s}'\rightleftarrows \bm{s}_{add}')-\bm{b}(\bm{s}_{new}^{(depth)}: \bm{s}'\rightarrow \bm{s}_{add}')\|_1\leq \delta$, and the optimal action for each observation node selected by building the whole tree are same as the ones selected by approximated value, the optimal value function of POMDP problem using our method $V^*(\bm{b}(\bm{s}_{new}))$ and the optimal value function using the direct resampling way $\hat{V}^*(\bm{b}(\bm{s}_{new}))$ will satisfy the following bounding equation:
\begin{equation}\label{EQ_28}
\|V^*(\bm{b}(\bm{s}_{new})) - \hat{V}^*(\bm{b}(\bm{s}_{new}))\|_1 \leq \frac{\gamma R_{max}}{1- \gamma}\delta.
\end{equation}
\end{corollary}

\begin{proof}
Please refer to Appendix~\ref{a4}.
\end{proof}

It is noted that Theorem~\ref{t_add_new1a} and Controllary~\ref{t_add_new3} are correct only under the assumption that each optimal action selected for each node using the approximation way is the same as the one using directly fully rebuilding the belief tree, which means that there are some bounds for the performance guarantee shown in Theorem~\ref{t_add_new1a} and Controllary~\ref{t_add_new3} using the reuse approximation.

\begin{discussion}\label{t_add_new1}
As mentioned before, our belief tree reuse can be considered as the wrong estimation $\triangle$ of the action and observation node values in the correct belief tree. For all observation nodes with belief $\bm{b}$, when the distance between the value function for the optimal action ${V}(\bm{b}, \bm{a}^*)$ and the value function for any sub-optimal action ${V}(\bm{b}, \bm{a}_{sub}^*)$ is relatively large satisfying ${V}(\bm{b}, \bm{a}^*)-{V}(\bm{b}, \bm{a}_{sub}^*) >> \triangle$,the estimation error $\triangle$ will not affect the final solution and the convergence of our solver, which follows the method directly rebuilding the whole belief tree. 
\end{discussion}

Following Discussion 1, we would like to get the exact performance bound to ensure that the obtained solution remains the same as the optimal solution using the fully-rebuild belief tree. Firstly, we have the following corollary:

\begin{corollary}\label{t_add_new2s}
With the same assumption as Corollary~\ref{t_add_new3}, given the action-observation sequence $\{\bm{a}_1,~\bm{o}_1,~\bm{a}_2,~\bm{o}_2,$ $\cdots,~\bm{o}_i\},~i< depth$ and reach the observation node with belief $\bm{b}$, the optimal action-value function of POMDP problem using our method $V^*(\bm{b},~\bm{a})$ and the optimal value function using the direct resampling way $\hat{V}^*(\bm{b},~\bm{a})$ for any given action $\bm{a}$ will satisfy the following bounding equation:
\begin{equation}\label{EQ_28}
\|V^*(\bm{b},~\bm{a}) - \hat{V}^*(\bm{b},~\bm{a})\|_1 \leq \frac{\gamma R_{max}}{1- \gamma}\delta.
\end{equation}
\end{corollary}

\begin{proof}
Just follows the similar proof process shown in Corollary~\ref{t_add_new1}, so we ignore it here.
\end{proof}

\begin{corollary}\label{t_add_new3z}
With the same assumption as Corollary~\ref{t_add_new3}, for any observation node with belief $\bm{b}$ within $depth$ layer, if the distance between the value function ${V}(\bm{b}, \bm{a}^*)$ for the optimal action $\bm{a}^*$ and the value function ${V}(\bm{b}, \bm{a}_{sub}^*)$ for any sub-optimal action $\bm{a}_{sub}^*$ is larger than $\frac{2\gamma R_{max}}{1-\gamma} \delta$, the optimal action obtained by the tree reuse approximation way will be same as the optimal action using the direct resampling way $\max_{\bm{a}} {V}(\bm{b}, \bm{a}) = \max_{\bm{a}} \hat{V}(\bm{b}, \bm{a})=\bm{a}^*$.
\end{corollary}

\begin{proof}
Please refer to Appendix~\ref{a5}.
\end{proof}

\section{Action execution and perception}
\label{s5}
Unlike many state-of-the-art methods~\cite{zheng2021multi, Zheng2022Towards} that assume static objects and no robot-object interactions, our approach enables the robotic arm to actively clear obstacles, uncovering hidden spaces. Our perception module combines point cloud and image data to estimate object poses, sizes, movability, and detection, eliminating the need for manual markers like QR codes.

\subsection{Action execution}
Our object search framework supports mobile robots with 2D LiDAR and RGB-D cameras, focusing on the Fetch robot. Navigation relies on the ROS move\_base interface with an AMCL-based stack~\cite{Quigley2009ROS}. Lift height and head joints are controlled via ros\_control~\cite{Chitta2017ros}, while object removal uses moveit~\cite{Sucan2012The} to pick and place items outside the workspace. The basic actions are shown in Fig.~\ref{fig:graph5and8}.

\begin{figure}[!htb]
	\minipage{\columnwidth}
	\centering
	\hspace{.1cm}
     \subfloat[Changing robot configuration (Move base) ]{\includegraphics[width=0.30\columnwidth, height=1.1in]{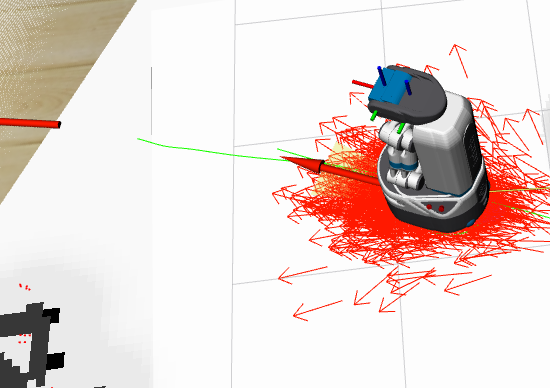}%
    \label{fig_msg_img}}
 	\hspace{.1cm}
     \subfloat[Changing robot configuration (Joints controller)]{\includegraphics[width=0.27\columnwidth, height=1.1in]{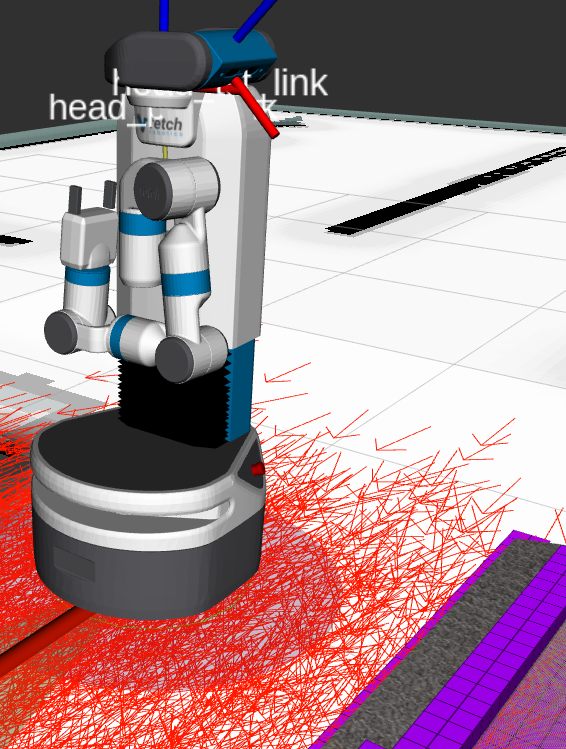}%
    \label{fig_msg_img}}
 	\hspace{.1cm}
      \subfloat[Removing action]{ \includegraphics[width=0.3\columnwidth, height=1.1in]{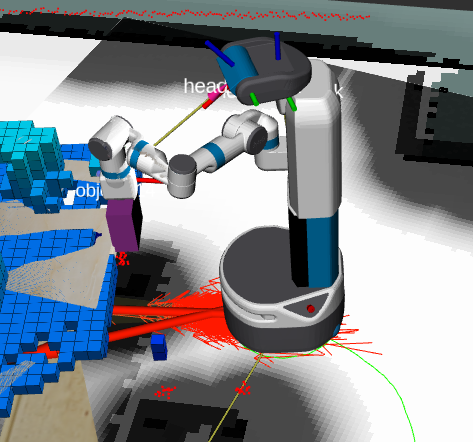}%
    \label{fig_msg_img}}
 	\hspace{.1cm}
	\caption{{Actions execution}}
	\label{fig:graph5and8}
	\endminipage
\end{figure}

\subsection{Sensor data operations}

To simplify the scenario with limited resources—a 2D LiDAR, a 3D RGB-D camera, and a 7-DoF robotic arm on the Fetch robot—we assume the availability of a pre-built point cloud map $\mathcal{M}$ and occupancy grid map of the environment~\footnote{For the Stretch robot, which also has a camera and LiDAR but only a 3-DoF manipulator, we adapt to its constraints by using the IKPy tool~\cite{Manceron2017IKPy} for inverse kinematics combined with base motion to achieve the required grasp poses.}. These maps include furniture and known objects but exclude unknown objects with unknown poses. This assumption is realistic, as the maps can be created beforehand and reused during planning. During online planning and execution, the depth camera’s point cloud $\mathcal{F}_j$ from the $j$-th frame is fused with $\mathcal{M}$ using Iterative Closest Point (ICP), initialized by 2D LiDAR matching~\cite{Rusu20113d}. ICP provides a noisy robot pose measurement $\bm{Z}_{icp}\sim \mathcal{N}(\bar{\bm{Z}}_{icp},\bm{\Sigma}_{icp})$, fused with AMCL localization $\bm{Z}_{amcl}\sim \mathcal{N}(\bar{\bm{Z}}_{amcl},\bm{\Sigma}_{amcl})$ in the filter part of the POMDP framework, where $\mathcal{N}(\star,~\bullet)$ means the one-dimension or multi-dimension normal distribution with mean value/vector $\star$ and variance value/covariance matrix $\bullet$. As frames accumulate, the combined point cloud $\mathcal{P}_j=\mathcal{M}\bigcup\mathcal{F}_0\bigcup\cdots\bigcup\mathcal{F}_j$ grows more complete. Subtracting the original map yields $\mathcal{M}'_j = \mathcal{P}_j / \mathcal{M}$, isolating $n$ newly detected objects $\{o_0,~o_1, \dots,~o_n\} \in \mathcal{M}'_j$ via Euclidean cluster extraction, as shown in Fig.~\ref{1xxxada}. Object poses and sizes are estimated by fitting minimum bounding boxes using principal component analysis. Data association matches local and global objects based on Mahalanobis and point-wise mean distances.


\begin{figure}[!th]
	\centering
	\includegraphics[width=8cm]{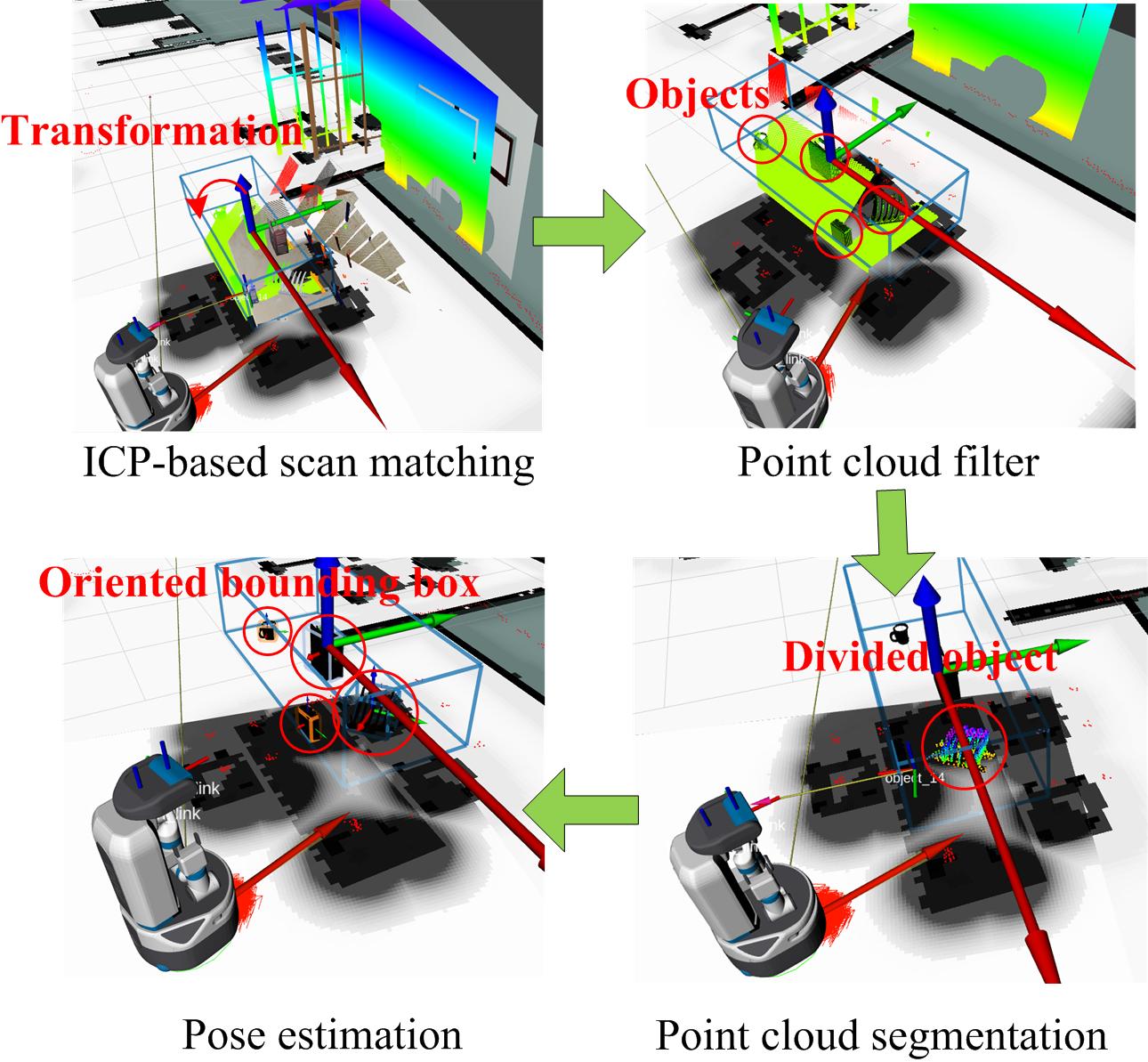}
	\caption{Measurement from point cloud} 
	\label{1xxxada}
\end{figure}

\subsection{Object detector}

State-of-the-art real-time object detection systems like YOLO classify objects into predefined categories, which may not match specific targets. Meanwhile, 3D point clouds provide valuable data for image segmentation. To tackle object detection with limited images and optional semantic cues, we combine traditional feature matching with YOLO, leveraging their strengths for accurate and efficient detection.

Building on point cloud segmentation, we project segmented local frame point clouds onto the image using the camera configuration, creating sub-images. YOLOv5~\cite{Jocher2020ultralytics} generates sub-images with semantic scores $\{s^y_i,~i=~1,~2,\cdots,~m\}$, where $m\neq n$ typically. A nearest-neighbor-based method aligns sub-images by matching those with sufficient overlap. Matched pairs link local sub-images to objects, and SIFT descriptors match these to the target object. Detection probability is defined as the ratio of matched features to total features, $\{s^d_i,~i=~1,~2,\cdots,~n\}$. For specified target types (e.g., a cup), semantic scores $s^y_i$ and detection probabilities $s^d_i$ are averaged. As real-world object matching can lack robustness and perception is not our main focus, especially for uniform-colored objects (e.g., a red bottle), we add color matching. The main object detection process is shown in Fig.\ref{7}.

\begin{figure}[!th]
	\centering
	\includegraphics[width=\columnwidth]{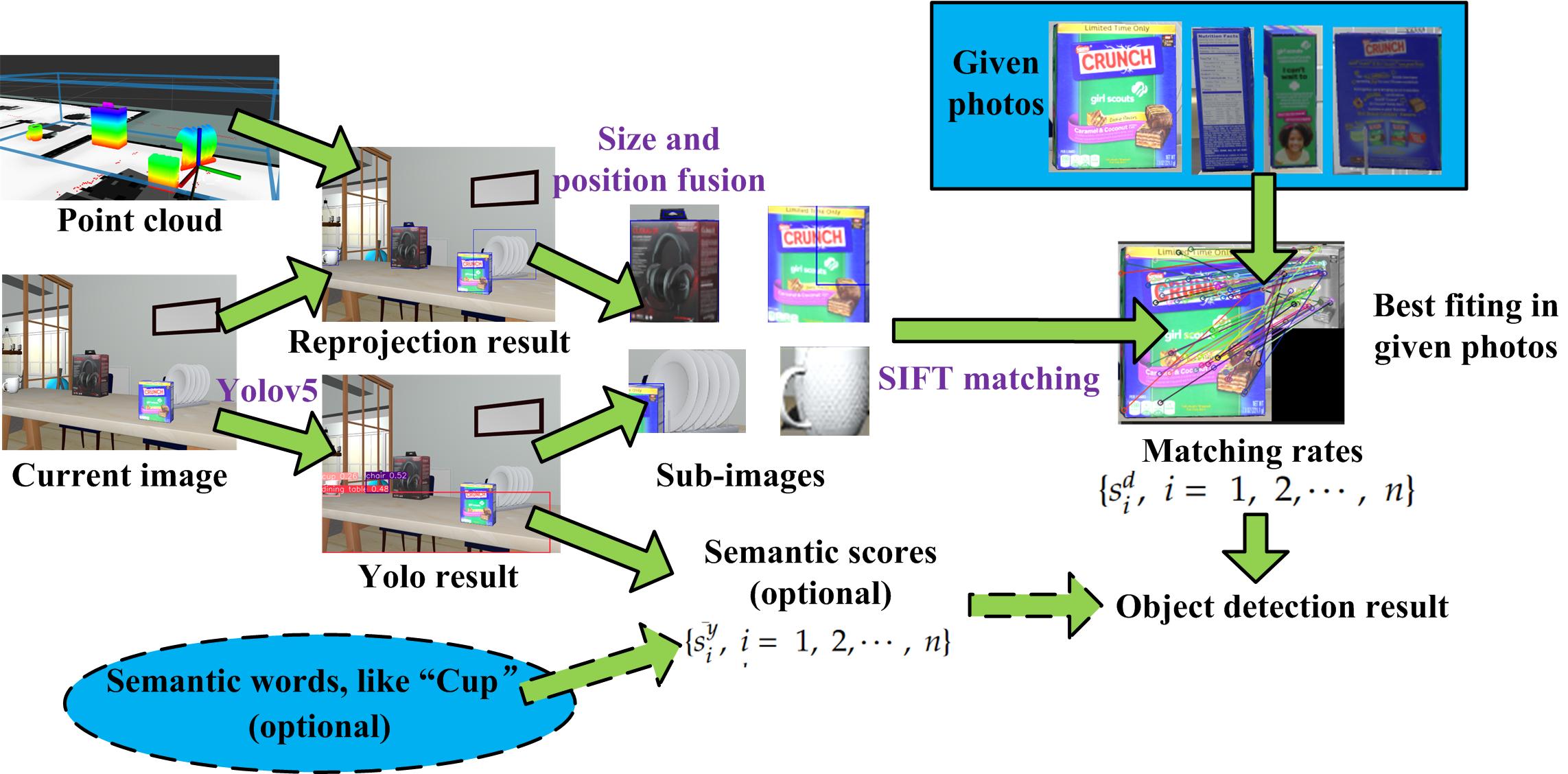}
	\caption{Object detector} 
	\label{7}
\end{figure}

\subsection{Move-ability estimation}
In real-world settings, some objects may be immovable due to size, manipulator limits, or base motion constraints. Our framework focuses on clearing the FOV by manipulating objects, requiring the estimation of movability probabilities and updating these beliefs within the POMDP planning.

Using point cloud segmentation on the fused global cloud, we isolate objects and use the Grasp Pose Detection (GPD) toolbox~\cite{Pas2017Grasp} to predict multiple grasp poses. To reduce complexity, k-means clustering selects $k$ diverse, high-success poses $\bm{p}^g_i,~i=1,2,\cdots,k$. These poses and nearby obstacles are transformed into local frames $T^g_r(\bm{p}^g_i)$ based on pre-visited robot poses $\bm{p}^r_i$, ensuring feasibility. The transformed poses are used as planning targets for the robot manipulator via the moveit toolbox, within a time limit $t_m$, but without execution. Feedback from moveit determines the movability probability $0.5<m_{o^i}\leq 1$. If no solution is found, $m_{o^i}=0.5$~\footnote{Our move-ability estimation is not very efficient to various environments due to the robot pose selection, so we still keep large guess for the object move-ability.}, and objects above a size threshold are non-movable ($m_{o^i}=0$). During planning, movability is assessed by sampling and comparing with $m_{o^i}$. Fig.~\ref{8ss} illustrates candidate grasp poses (blue forks).

\begin{figure}[!th]
	\centering
	\includegraphics[width=5cm]{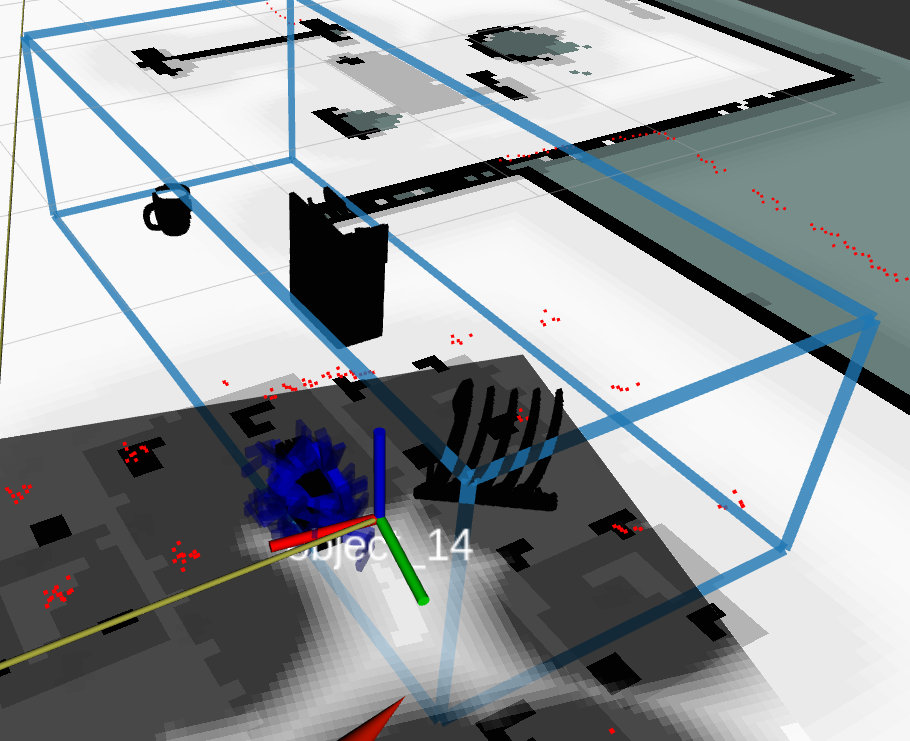}
	\caption{Candidate grasp poses} 
	\label{8ss}
\end{figure}

\section{Simulations and real-world experiments}
\label{s7}
\subsection{Simulations}

We validate our approach using Fetch and Stretch robot simulators in the Gazebo environment, implemented by C++ and Python. The neural process network is trained for 3000 iterations (approximately 4 hours) on a single NVIDIA 3090 GPU. Post-training, we execute our project on a desktop machine, utilizing only the CPU, operating on Ubuntu 18.04, and powered by an Intel Core i7-13700k processor. Our evaluation spans diverse object configurations and scenarios, comparing our method against benchmarks, including continuous action domain benchmark methods (POMCPOW and VOMCPOW) and classical POMDP methods with manual-setting discrete action domain (POMCP, GPOMCP~\cite{chen2024pomdp}, and NPF-$k$CT, which is sub-version of our proposed GNPF-$k$CT method without belief tree reuse~\cite{anonymous2024npfkct}.).

\subsubsection{Problem settings.}\label{s711}
The simulations occur in a realistic Gazebo living environment with diverse object configurations. Navigation relies on a point cloud map and a 2D occupancy grid, featuring furniture and some predefined objects, while most objects are unknown. Large furniture surfaces serve as workspaces, with a target object (a blue snack box) and various obstacles placed randomly. Methods are evaluated across five scenarios, with robot candidate positions defined by four red rectangles around each workspace, unrestricted in orientation. Except for parameter sensitivity analysis, the cluster number $k$ for all simulations and experiments is set as 3~\footnote{In fact, because the workspace is located in 4 areas, the suitable parameters $k$ should be equal to 4, as demonstrated in Section~\ref{s713}. Without loss of genrality, we use 3 to get our main results, which still shows the dominant performance.}. The robot's head motion is constrained to $l_p\in (-\frac{\pi}{12},~\frac{\pi}{12})$, $l_t\in (-0.5,~0.5)$, and the lift motion $l_h$ is limited to $(0.0,~0.4)$.

 In the context of POMDP models for exploration, based on the Fetch robot's RGBD camera setting, the hexahedron FOV is defined by a 60-degree horizontal view angle and a vertical range with a height-width ratio of 480/600. The nearest and farthest planes to the camera center are set at 0.5 meters and 1.7 meters, respectively. The grid size of the grid world for the guessed target object is set as 2 cm. Each planning step is allotted a maximum of 60 seconds. Key parameters include grid updating thresholds $v_p=0.1$ and $v_n=-0.1$, re-initialized grid values for the guessed target object set at 0.2, and reward values are assigned as follows $R_{max}=10^5$, $R_{ct}=5\times10^4$, $R_{co}=10^4$, $R_{min}=-1$, and $R_{ill}=-10^3$.

\subsubsection{NP for primitive action.}
We evaluate the trained neural network's ability to filter meaningless primitives using the PyTorch framework. We present the training loss (Fig.~\ref{fig:sidelossa}) and validation loss (Fig.~\ref{fig:sidelossb}) over epochs, highlighting the model's performance and convergence behavior across 1000 iterations. Validation loss, calculated on testing data, is measured as the median log-likelihood loss, representing the log of the probability density or mass function (higher values indicate better performance). The observed increase in validation loss, alongside the decrease in training loss, demonstrates effective learning and confirms the absence of overfitting issues.

\begin{figure}[!htb]
	\minipage{\columnwidth}
	\centering
     \subfloat[Training loss]{\includegraphics[width=0.48\columnwidth]{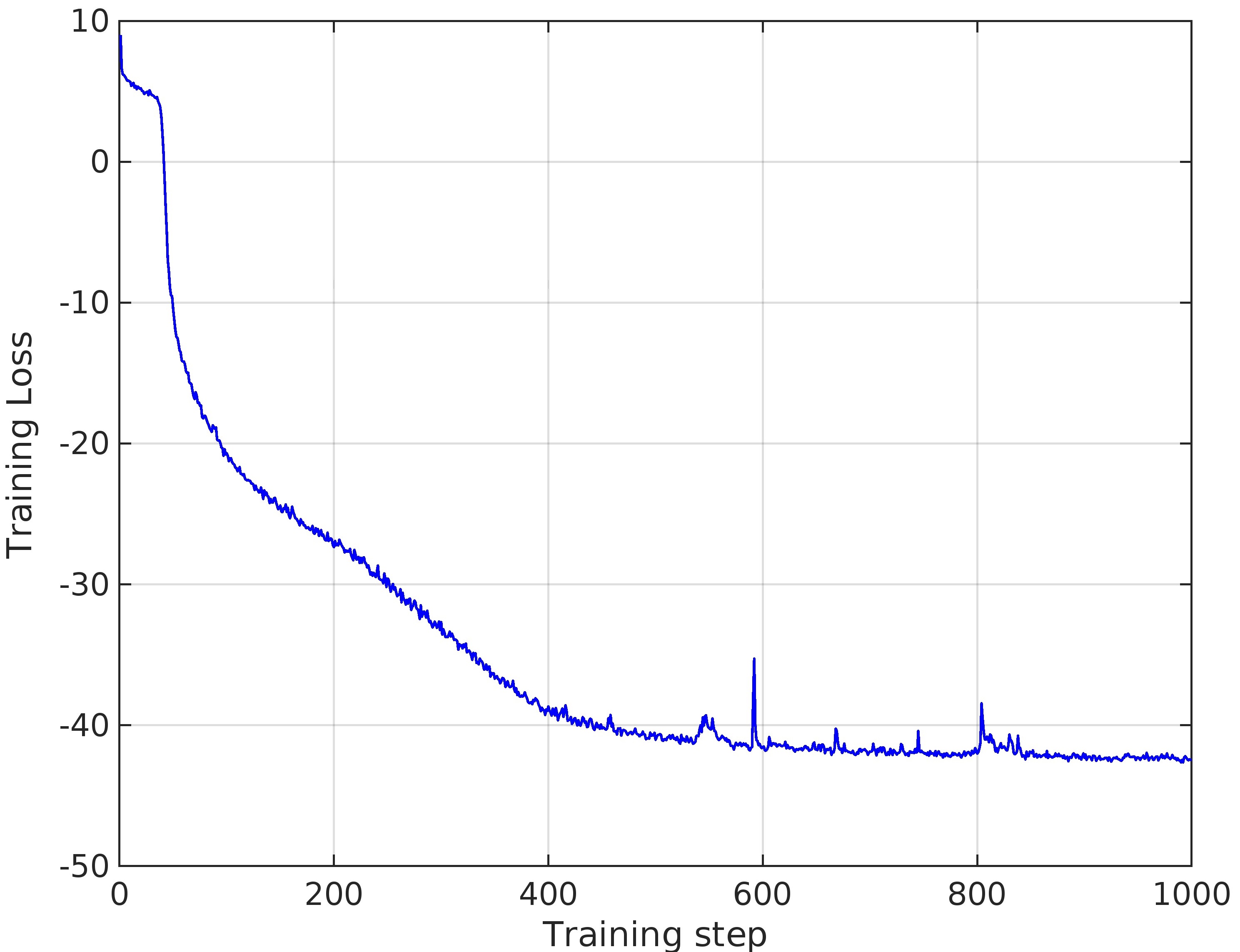}%
    \label{fig:sidelossa}}
 	\hspace{.1cm}
     \subfloat[Validation loss]{\includegraphics[width=0.48\columnwidth]{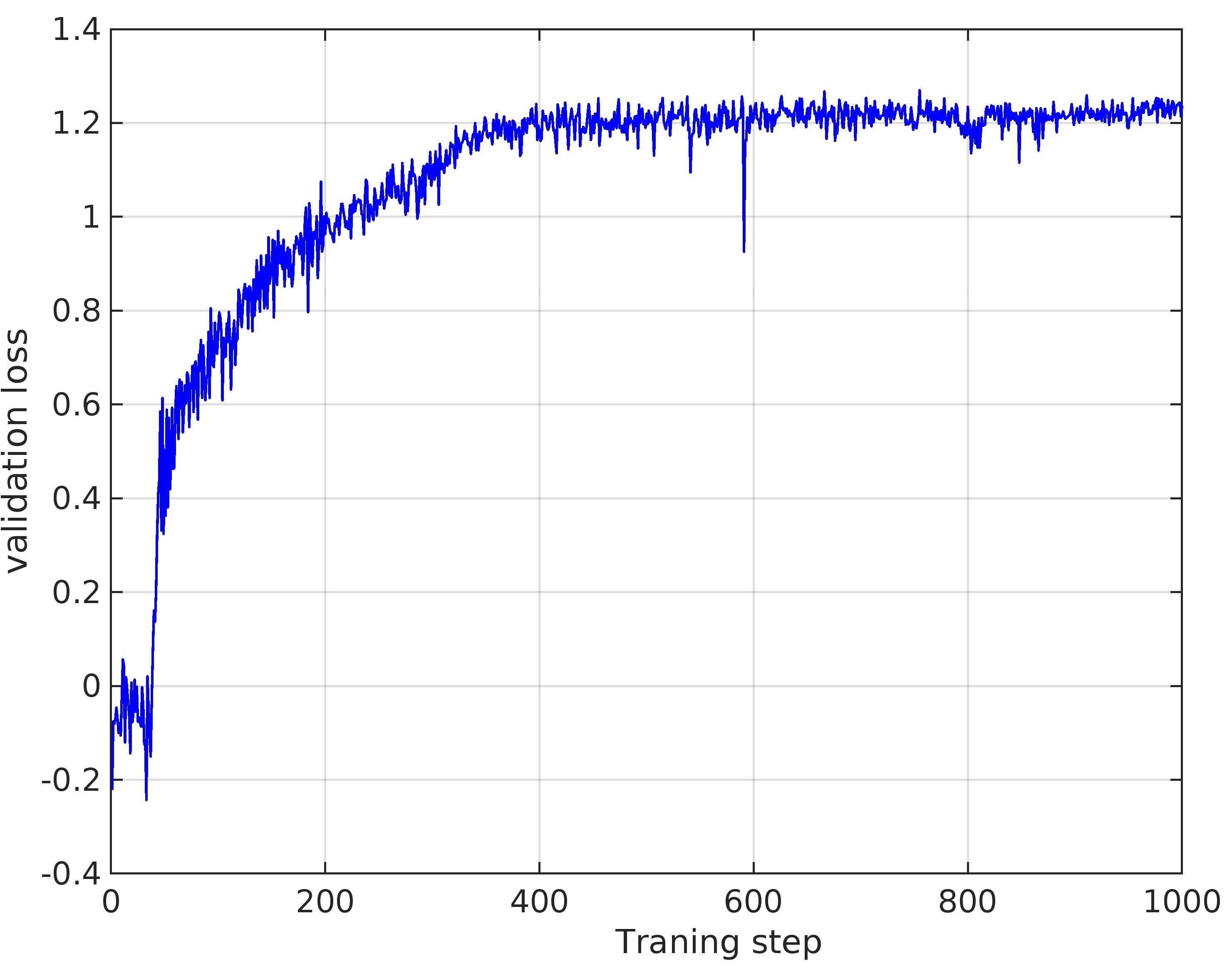}%
    \label{fig:sidelossb}}
	\caption{Loss Trends}
\label{fig:sideloss}
	\endminipage
\end{figure}


Fig.~\ref{3} shows the performance of the trained network after 3000 iterations, comparing observed test data (red line) with the predicted Gaussian distribution's 2-$\sigma$ bound (pink area, truncated to [0,~1]). Data are reordered by the robot's position to highlight trends. Predicted accuracy $Acc$ is defined as:
\begin{equation}\label{sim1}
	Acc = 1- Pro_w
\end{equation}
where $Pro_w$ is the probability of considering \emph{efficient} actions (observed probability is larger than 0.05) as the useless actions (mean value is smaller than 0.05); is $99.02\%$ for the test dataset. We present visual results for two test samples, demonstrating that the NP network effectively filters useless actions, aiding the GNPF-$k$CT solver in reducing the continuous action domain and enhancing efficiency.

\begin{figure}[!htb]
\centering
\includegraphics[width=0.9\linewidth]{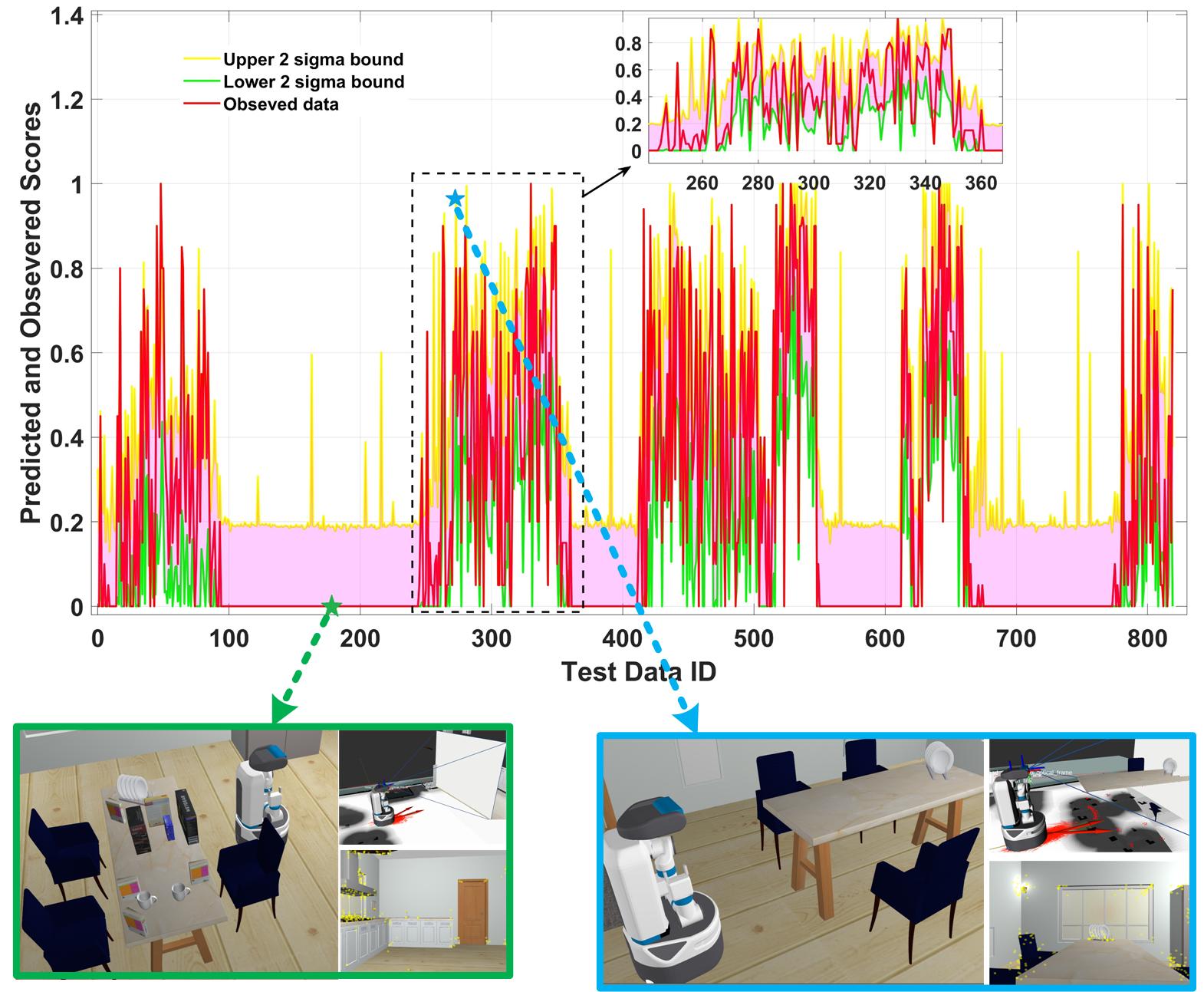}
\caption{The predicted and observed scores.}
\label{3}
 \end{figure}

\subsubsection{GNPF-$k$CT framework.}\label{s713}
\paragraph{Comparison results}
\begin{figure}[!htb]
	\minipage{\columnwidth}
	\centering
	\hspace{.1cm}
     \subfloat[Loose1]{\includegraphics[width=2.1cm,height=1.5cm]{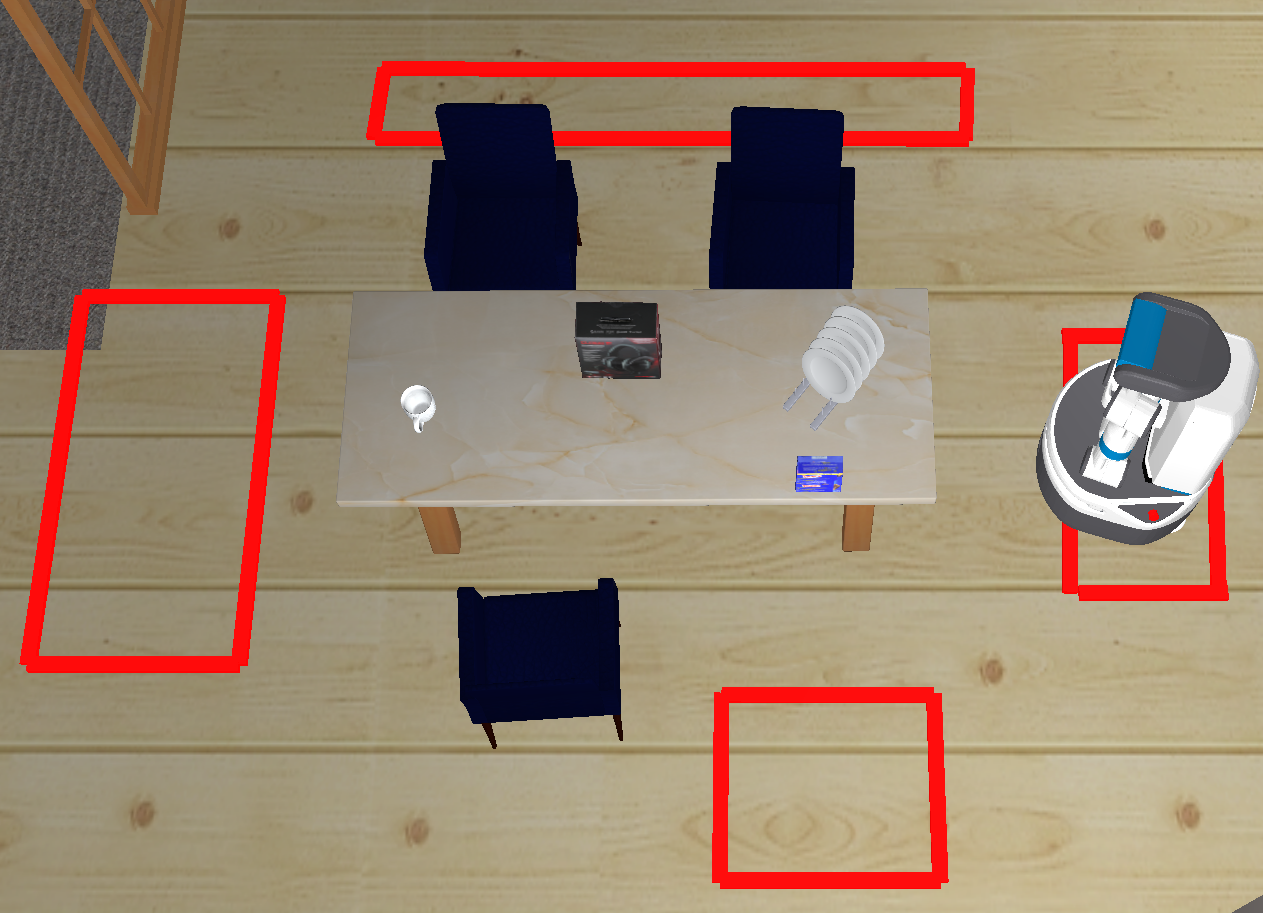}%
    \label{fig1_msg_img_ga}}
 	\hspace{.2cm}
     \subfloat[Loose2]{\includegraphics[width=3.2cm,height=1.5cm]{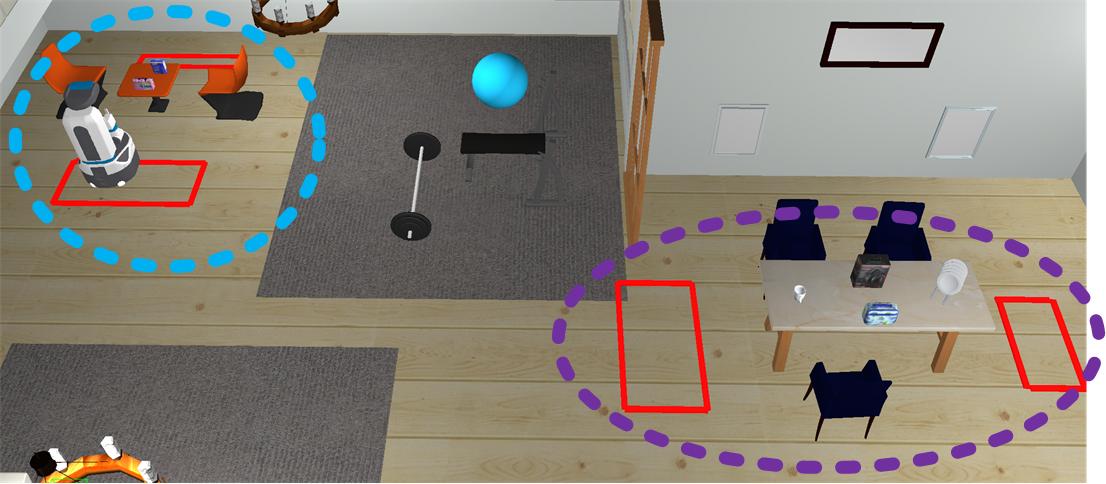}%
    \label{fig1_msg_img_gb}}
 	\hspace{.2cm}
     \subfloat[Hidden1]{\includegraphics[width=2.1cm,height=1.5cm]{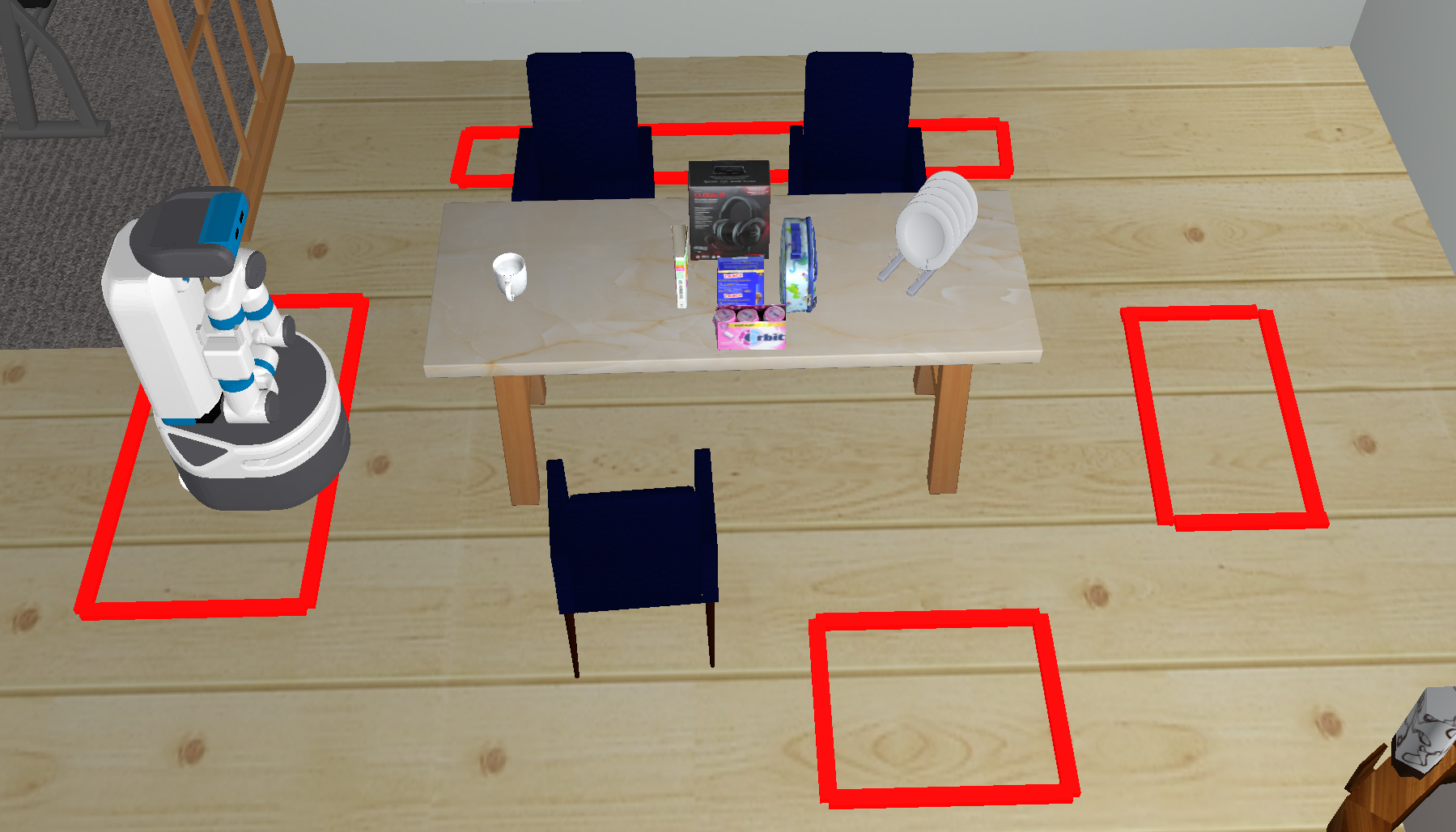}%
    \label{fig1_msg_img_gc}}
 	\hspace{.2cm}
     \subfloat[Covered1]{\includegraphics[width=2.1cm,height=1.5cm]{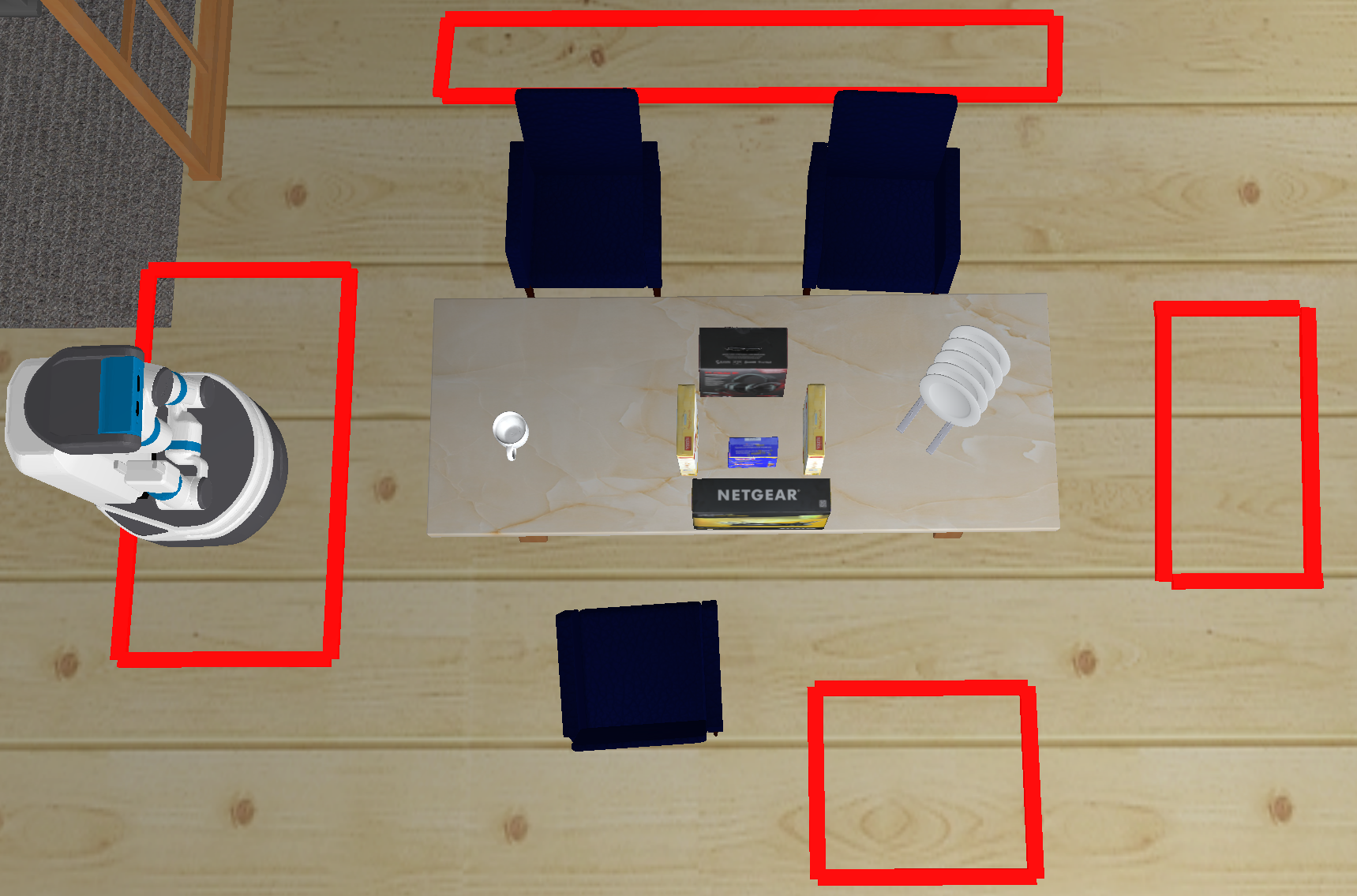}%
    \label{fig1_msg_img_gd}}
 	\hspace{.2cm}
     \subfloat[Complex1]{\includegraphics[width=2.1cm,height=1.5cm]{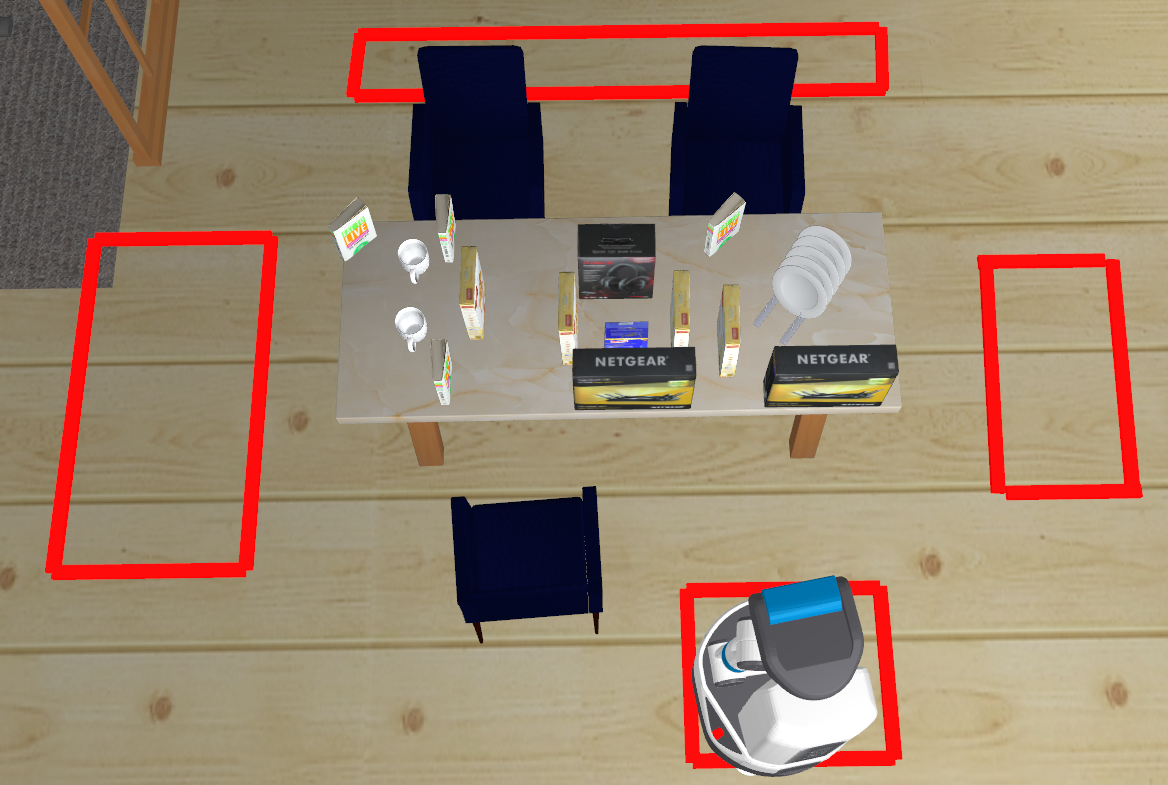}%
    \label{fig1_msg_img_gd}}
	\caption{The planning environment}
 \label{fig1_msg_img_gd_all}
	\endminipage
\end{figure}

Many existing object search methods rely on classical POMCP, which operates in a discrete action domain. We evaluate our method against both discrete-action POMDP solvers (POMCP, GPOMCP) and continuous-action solvers (POMCPOW, VOMCPOW, NPF-$k$CT). For POMCP and GPOMCP, robot configurations $\bm{a}_r$ are manually limited to combinations of 4 selected poses, 3 lift heights, and 9 head orientations, reducing coverage of the continuous action domain. Although this creates an unfair comparison, it highlights the benefits of broader continuous action domains. All methods address object search with a guessed target object and are tested across five scenarios of varying complexity:  Loose1 (4 objects, 1 workspace), Loose2 (6 objects, 2 workspaces), Hidden1 (7 objects, 1 workspace), Covered1 (7 objects, 1 workspace), and Complex1 (15 objects, 1 workspace), as shown in Fig.~\ref{fig1_msg_img_gd_all}. The Covered1 and Complex1 scenarios are particularly challenging, as the robot is unable to observe the target object from most configurations, even when operating at the maximum lift height. Several representative cases illustrating occluded observations, even under favorable robot configurations, are shown in Fig.~\ref{fig1_msg_img_gd_all_xxxz}. Each scenario involves 20 trials, and results, including the $95\%$ confidence interval for discounted cumulative reward, steps, and success rate within 50 steps, are reported in Table~\ref{sample-table} using the Fetch robot simulator. Because our GNPF-$k$CT solver is the combination of the GPOMCP and NPF-$k$CT, this result can also be considered as one of the ablation studies of our framework. Additionally, we know that the neural process is definitely useful for the task, so we do not remove the NP part and the rest ablation is implemented for the $k$-center clustering. In the NPF-$k$CT variant, we remove $k$-center clustering, which refines the continuous action domain (via Refine in Algorithm~\ref{alg:cap4}), while retaining action clustering and initial list updates before MCTS to summarize NP samples. Table~\ref{sample-table211} compares Loose1 and Hidden1 cases, demonstrating that $k$-center clustering significantly enhances the solver performance.

\begin{figure}[!htb]
\centering
\includegraphics[width=0.9\linewidth]{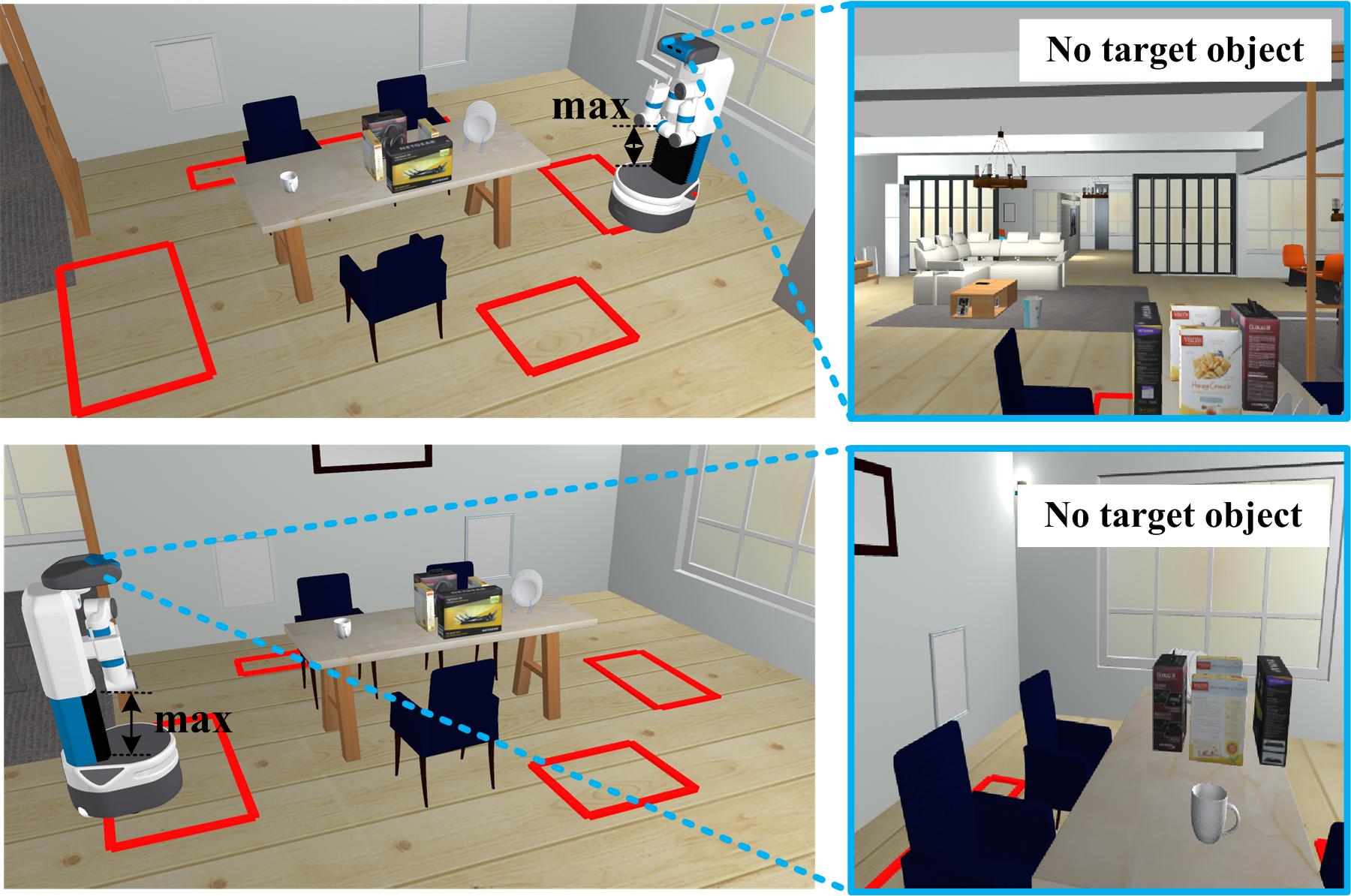}
\caption{Occluded observations even under good robot configurations.}
\label{fig1_msg_img_gd_all_xxxz}
 \end{figure}

\begin{table}[!ht]
	\caption{Ablation study based on NPF method}
	\label{sample-table211}
	\centering
	\begin{tabular}{ll}
		\toprule
		Scenarios&NPF\\
		\midrule
     Loose1 & ${88653.8}\pm{6812.4}~|~{6.7}\pm {0.9}|{100\%}$ \\
     Hidden1 &  ${75139.5}\pm{11521.5}~|~{11.4}\pm {2.7}|{100\%}$\\
		\bottomrule
	\end{tabular}
\end{table}

Our method and NPF-$k$CT exhibit a significant advantage across all scenarios due to the efficient NP filtering and robust clustering refinement. These mechanisms optimize robot configurations—poses, lift heights, and head orientations—for effective object observation. The continuous action refinement further facilitates the discovery of rare and challenging configurations, which are typically difficult to pinpoint in a large continuous action domain. This capability allows the robot to effectively locate and observe target objects, as illustrated in Fig.~\ref{rare}.

\begin{figure}[!htb]
\centering
\includegraphics[width=0.9\linewidth]{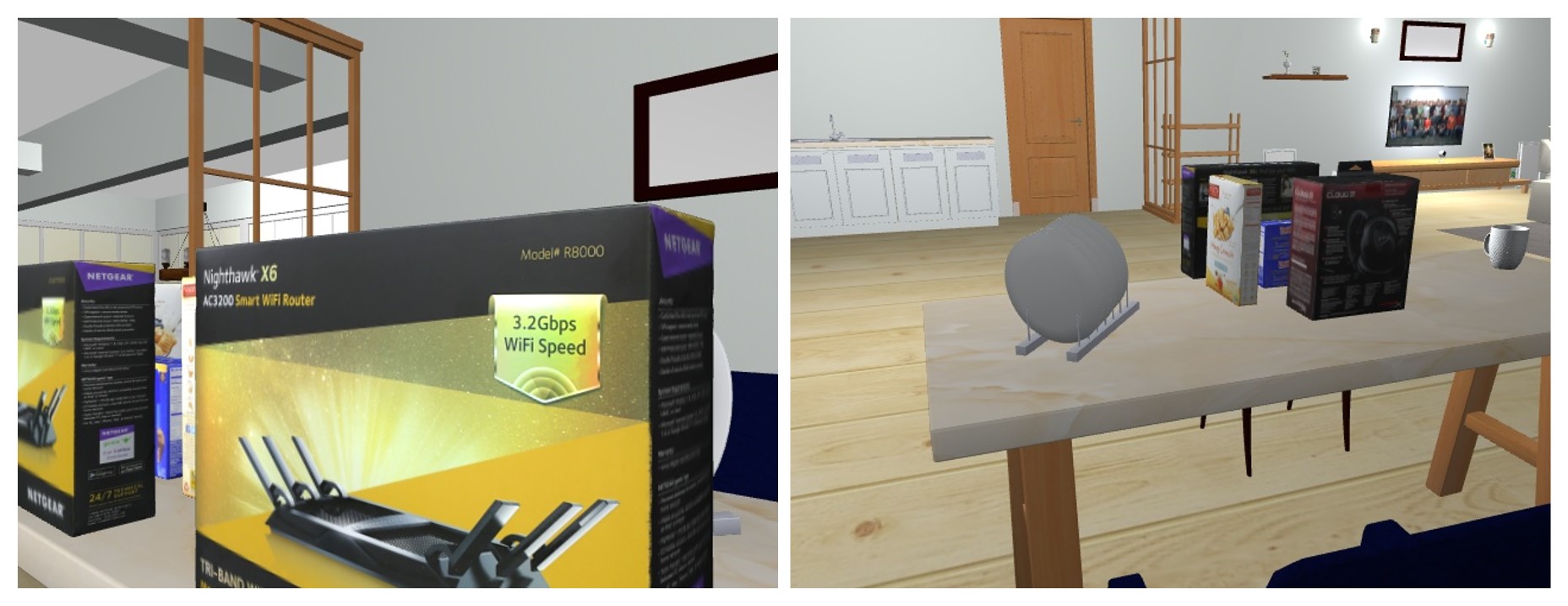}
\caption{Good observations based on tricky robot configurations (very common to obtain based on our method).}
\label{rare}
 \end{figure}
In simpler scenarios like Loose1 and Hidden1, NP filtering and refined clustering enable the robot's FOV to quickly cover the workspace, often observing all objects in the first step. In these cases, belief tree reuse is less effective or counterproductive due to added complexity (Section~\ref{s64}), no new observations, and the more predictable measurements within the belief tree.  However, belief tree reuse becomes advantageous in complex scenarios like Complex1 and Covered1, where the robot encounters numerous unexpected objects during the search process.

We illustrate object search execution in the Covered1 scenario with a 16-step sequence of actions, including adjustments to the robot's configuration, identification of obstacles and target objects, and their removal (highlighted with yellow dashed circles), as shown in Fig.~\ref{4}. For each step, we display the grid world’s odds value, $Odd(\mathcal{G})$, in the lower-left red image and the detected camera image in the lower-right blue image.  Odds values consistently converge on the target object.

\begin{table*}
  \caption{Discounted cumulative reward, steps, and successful rate for POMDP-based methods (within 50 steps)}
  \centering
  \label{sample-table}
  \small{
  \begin{tabular}{lllll}
    \cmidrule(lr){1-4}
    Scenarios&Loose1&\multirow{6}{*}{}& Complex1&\multirow{6}{*}{} \\
    \cmidrule(lr){1-2} \cmidrule(lr){4-4} 
    POMCP& $78629.6\pm 8472.1~|~7.2\pm0.7$~$|$~$\bm{100\%}$& & $27564.2\pm  7591.0~|~47.9\pm3.2$~$|$~${25\%}$ \\
    GPOMCP& $79788.0\pm 4787.4~|~7.4\pm0.6$~$|$~$\bm{100\%}$& &$40278.0\pm  4957.7~|~42.9\pm3.8$~$|$~${50\%}$\\
    POMCPOW& $72612.0\pm 11102.2~|~9.4\pm2.1$~$|$~$\bm{100\%}$ & & ${37023.4}\pm {6951.9}~|~{44.2}\pm{3.4}$~$|$~${75\%}$  &\\
    VOMCPOW&  $77622.9\pm 11406.1~|~8.5\pm1.6$~$|$~$\bm{100\%}$& & ${40555.1}\pm {6830.9}~|~{41.1}\pm{4.4}$~$|$~${90\%}$ &\\
    NPF-$k$CT& $\bm{94795.1}\pm \bm{6350.6}~|~\bm{6.0}\pm \bm{0.7}$~$|$~$\bm{100\%}$ & & ${44737.1}\pm {6669.1}~|~{36.3}\pm {5.1}$~$|$~$ \bm{95\%}$  & \\
    \cmidrule(lr){1-2} \cmidrule(lr){4-4}
    GNPF-$k$CT& $ 89548.1\pm 7391.9 ~|~ 6.7\pm 0.8 $~$|$~$\bm{100\%}$ & & $ \bm{50527.8}\pm \bm{5259.7} ~|~ \bm{33.8}\pm \bm{5.0} $~$|$~$\bm{95\%}$  & \\
    \cmidrule(lr){1-4}
    Scenarios&Hidden1&\multirow{6}{*}{}& Covered1&\multirow{6}{*}{}\\
    \cmidrule(lr){1-2} \cmidrule(lr){4-4}
    POMCP& $ 45815.2\pm7260.9~|~19.4\pm2.8$~$|$~$\bm{100\%}$&  &$31506.9\pm 6249.7~|~ 32.3\pm 6.0$~$|$~80\% & \\
    GPOMCP& $55574.8\pm6225.3~|~15.3\pm2.0$~$|$~$\bm{100\%}$ &  &$34397.8\pm 7381.6~|~26.7\pm 4.8$~$|$~$95\%$ &\\
    POMCPOW& $61728.4\pm8791.3~|~12.9\pm 2.7$~$|$~$\bm{100\%}$ & &$40762.1\pm8401.1~|~23.9\pm 6.3$~$|$~$90\%$&\\
    VOMCPOW& $58286.8\pm10101.1~|~14.9\pm 3.1$~$|$~$\bm{100\%}$ & &$35725.4\pm9880.5~|~25.2\pm 5.9$~$|$~$90\%$& \\
    NPF-$k$CT&$\bm{83377.1}\pm \bm{6427.3}~|~8.5\pm 1.4$~$|$~$\bm{100\%}$ & &${44966.1}\pm{6340.2}~|~{21.7}\pm {3.0}$~$|$~$\bm{100\%}$&  \\
    \cmidrule(lr){1-2} \cmidrule(lr){4-4}
    GNPF-$k$CT&$ 81270.1\pm 5947.1~|~\bm{8.3} \pm \bm{1.3} $~$|$~$\bm{100\%}$ & &$ \bm{55596.9}\pm \bm{8602.0}~|~ \bm{19.4}\pm \bm{3.8} $~$|$~$\bm{100\%}$&  \\
    \cmidrule(lr){1-4}
    Scenarios&Loose2 &\multirow{7}{*}{} &\multirow{7}{*}{\includegraphics[width=0.15\linewidth,height=0.8in]{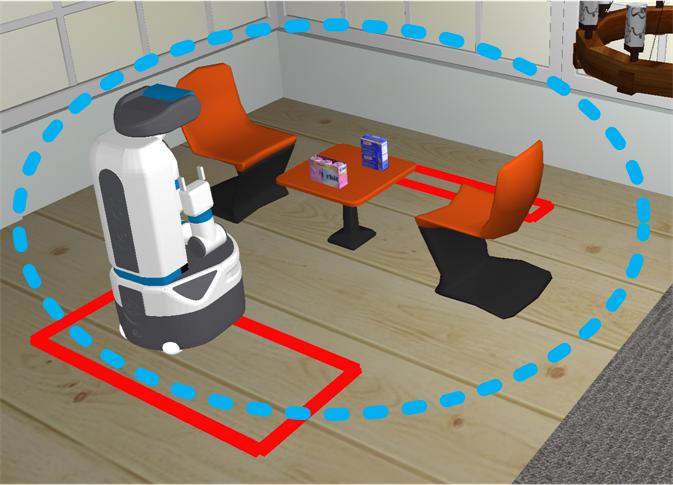}~~\includegraphics[width=0.15\linewidth,height=0.8 in]{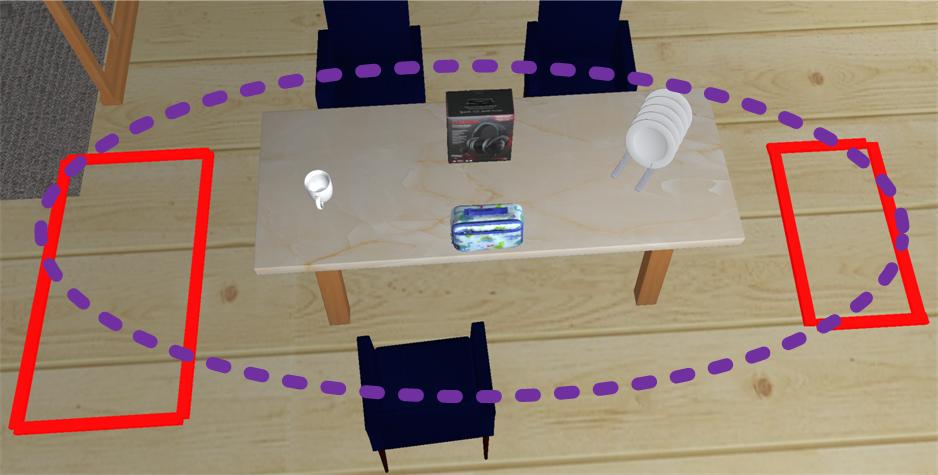}}~&\multirow{7}{*}{}\\
    \cmidrule(lr){1-2} 
    POMCP& $38462.6 \pm 11221.2$~$|$~$23.2\pm6.6$~$|$~95\%  &&\\
    GPOMCP& ${51574.1 \pm 17930.9}$~$|$~${17.4 \pm 2.9}$~$|$~$\bm{100\%}$ & &\\
    POMCPOW& ${{21785.1} \pm {6783.1}}$~$|$~${{32.7} \pm {6.5}}$~$|$~${70\%}$&  &\\
    VOMCPOW&  ${{26860.0} \pm {5779.8}}$~$|$~${{28.2} \pm {6.0}}$~$|$~${85\%}$ &  &\\
    NPF-$k$CT&  ${\bm{69992.7} \pm \bm{8185.8}}$~$|$~${{11.4} \pm {2.0}}$~$|$~$\bm{100\%}$ & &  \\
    \cmidrule(lr){1-2} 
    GNPF-$k$CT&  ${68989.1}\pm {8585.6}$~$|$~$\bm{ 11.1}\pm  \bm{1.7}$~$|$~$\bm{100\%}$ & & ~~~~~~~~~~~~~~(b) Loose2 \\
    \cmidrule(lr){1-4} 
  \end{tabular}}
\end{table*}

\begin{figure*}[!htb]
  \centering
\includegraphics[width=\linewidth]{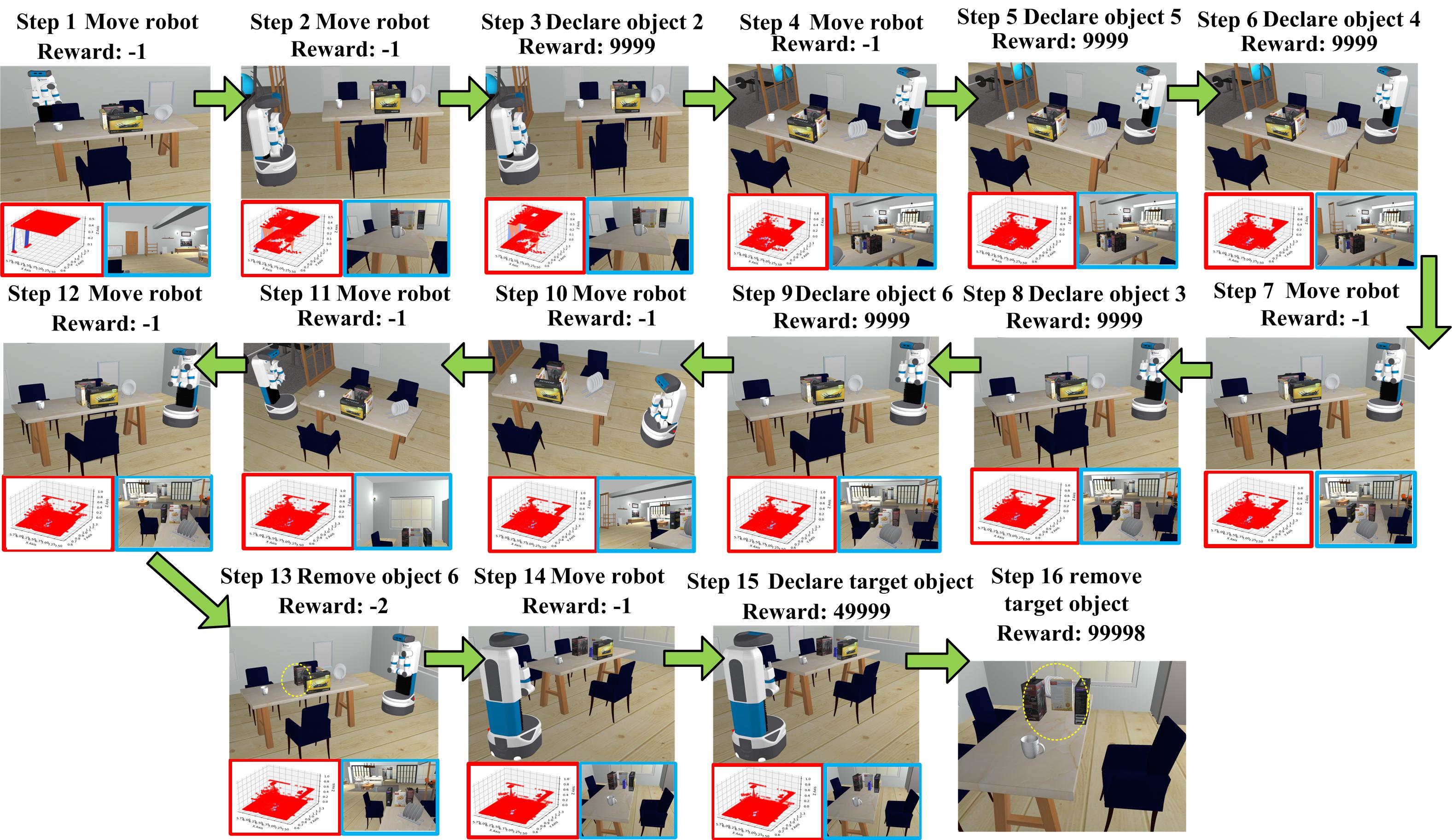}
\caption{The visual progress for Covered1 scenario.}
\label{4}
 \end{figure*}
 
\paragraph{Results with different parameters}

Our method demonstrates robust performance against manual parameter settings. For example, as long as $R_{max}\gg R_{min}$, performance remains consistent. When $R_{min}$ changes from -1 to -20, the final performance shows minimal variation. Comparative results for Hidden1 emphasize this stability, attributed to the compact and well-structured POMDP design  (Table~\ref{sample-table2}).

\begin{table}[!ht]
	\caption{95\% confidence interval of discounted cumulative reward, steps, and successful rate (within 50 steps)}
	\label{sample-table2}
	\centering
    \small{
	\begin{tabular}{ll}
		\toprule
		Scenarios&Hidden\_1\\
		\midrule
        $R_{min}=-1$ &  $89548.1\pm 7391.9 ~|~ 6.7\pm 0.8 $~$|$~$\bm{100\%}$  \\
        $R_{min}=-20$ &  $91579.4\pm 7536.1 ~|~ 6.5\pm 0.9 $~$|$~$\bm{100\%}$     \\
		\bottomrule
	\end{tabular}}
\end{table}

We also highlight several cases where certain parameters can significantly impact performance, such as the threshold for declaring actions $n_{odds}$. This threshold determines the grid values used for comparison, crucial for assessing declaring action success. Smaller $n_{odds}$ increases task difficulty, requiring object detection from multiple orientations. We conducted statistical analyses in Fig.~\ref{fig:side} with $n_{odds}=2,~4,~6$ for several representative methods in a scenario featuring 6 objects (Fig.~\ref{5}). While all methods achieved 100\% success, the advantage of GNPF-$k$CT and NPF-$k$CT diminished with easier $n_{odds}$ setting. In our real-world experiments with the Stretch robot (Fig.~\ref{6a}), $n_{odds}$ was set to 6.

\begin{figure}[!th]
	\centering
	\begin{subfigure}[b]{0.48\linewidth}
		\centering
  \includegraphics[width=1.6in,height=1.5in]{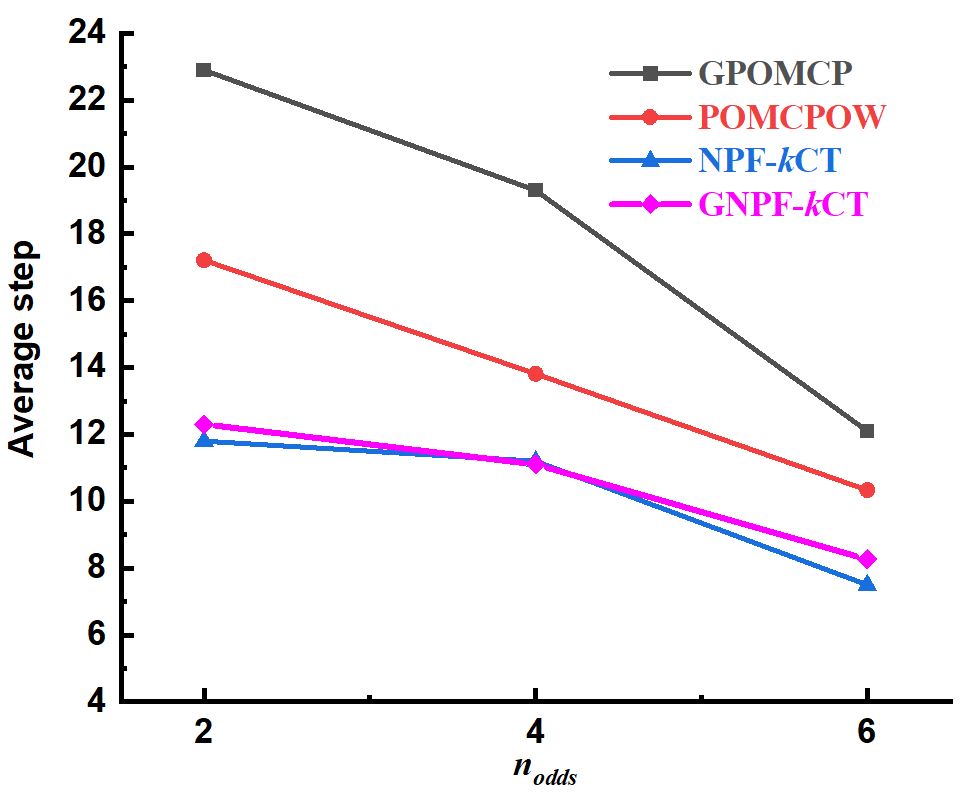}
		\caption{Average step}
	\end{subfigure}%
	\begin{subfigure}[b]{0.48\linewidth}
		\centering
  \includegraphics[width=1.6in,height=1.5in]{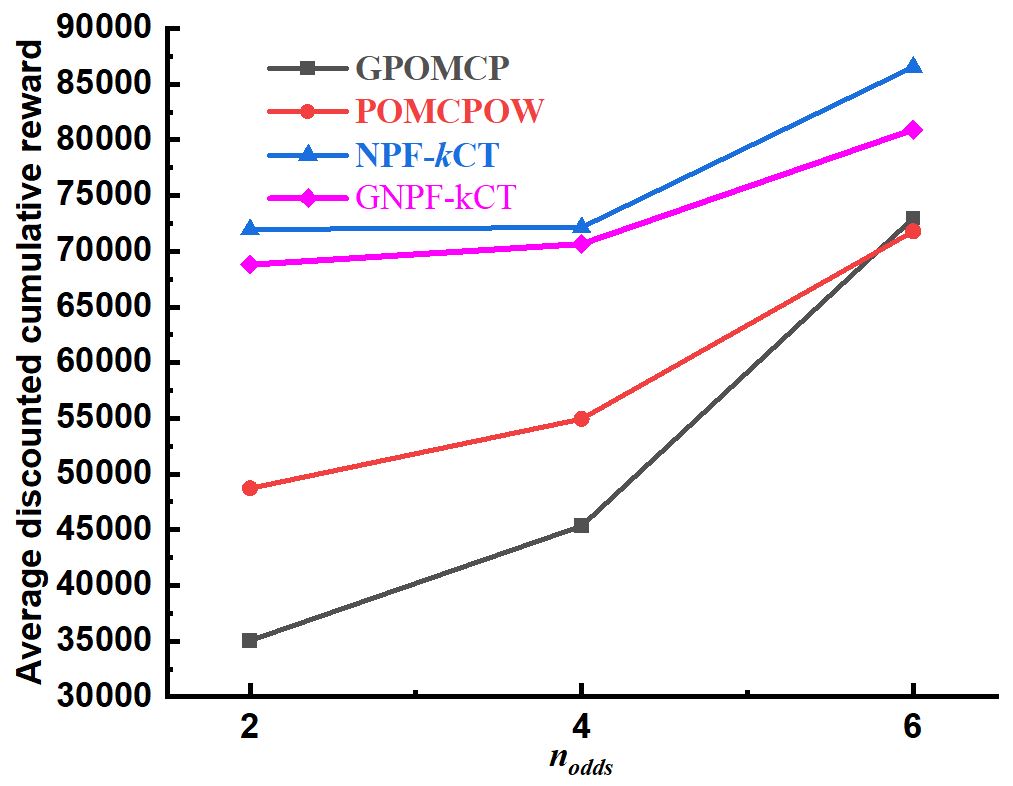}
		\caption{Average reward}
	\end{subfigure}
	\caption{Results for different $n_{odd}$ values.}\label{fig:side}
\end{figure}

\begin{figure}[!htb]
	\centering
\includegraphics[width=0.7\linewidth]{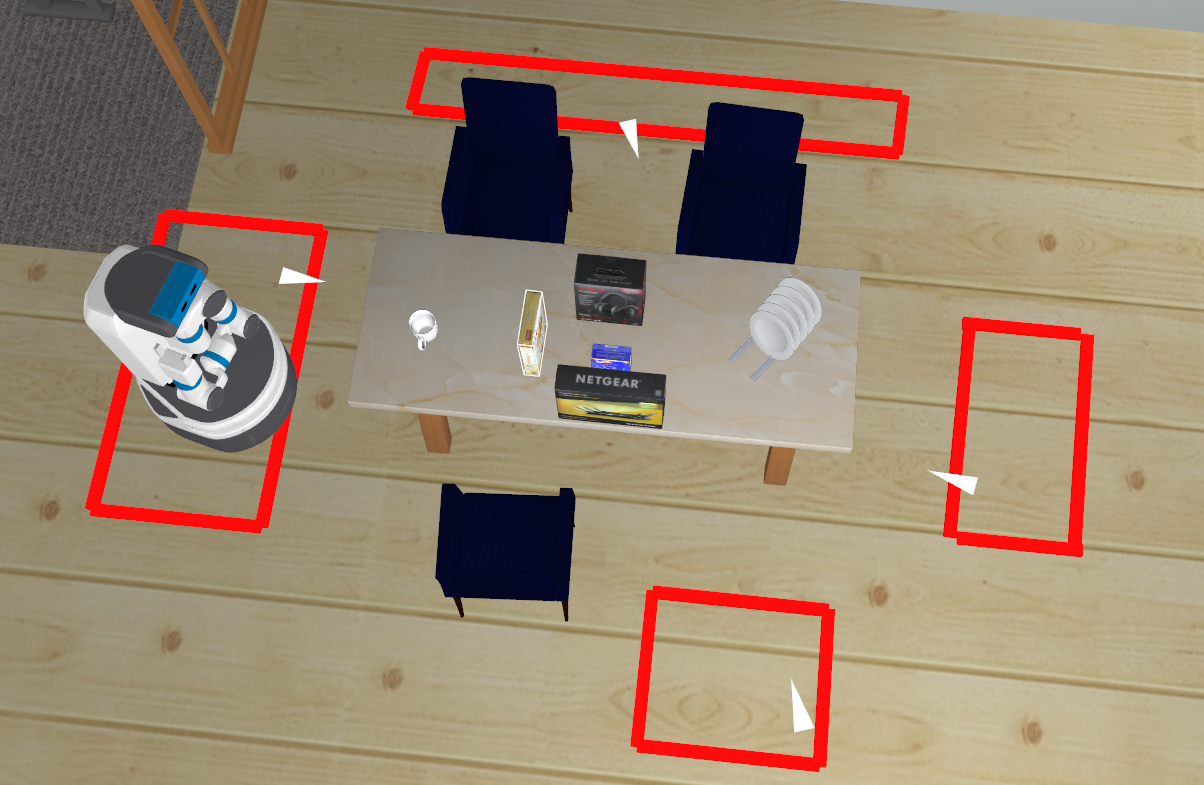}
\caption{The scenario with 6 objects with candidate continuous position domain (red rectangles) and discrete pose domain (white triangles).}
\label{5}
 \end{figure}

The parameter $k$ significantly influences the refining speed of the continuous action domain, affecting performance across tasks. To assess its impact, we varied $k$ from 2 to 8 for the 6-object scenario (Fig.~\ref{5}). Fig.~\ref{fig:side11} shows the 95\% confidence intervals for discounted cumulative rewards (black line) and steps (red line), with results grouped by colored dashed boxes for corresponding $k$ values. The analysis indicates optimal clustering at $k=4$ and $k=7$, while other values reduce solver performance.

\begin{figure}[!htb]
	\centering
\includegraphics[width=0.8\linewidth]{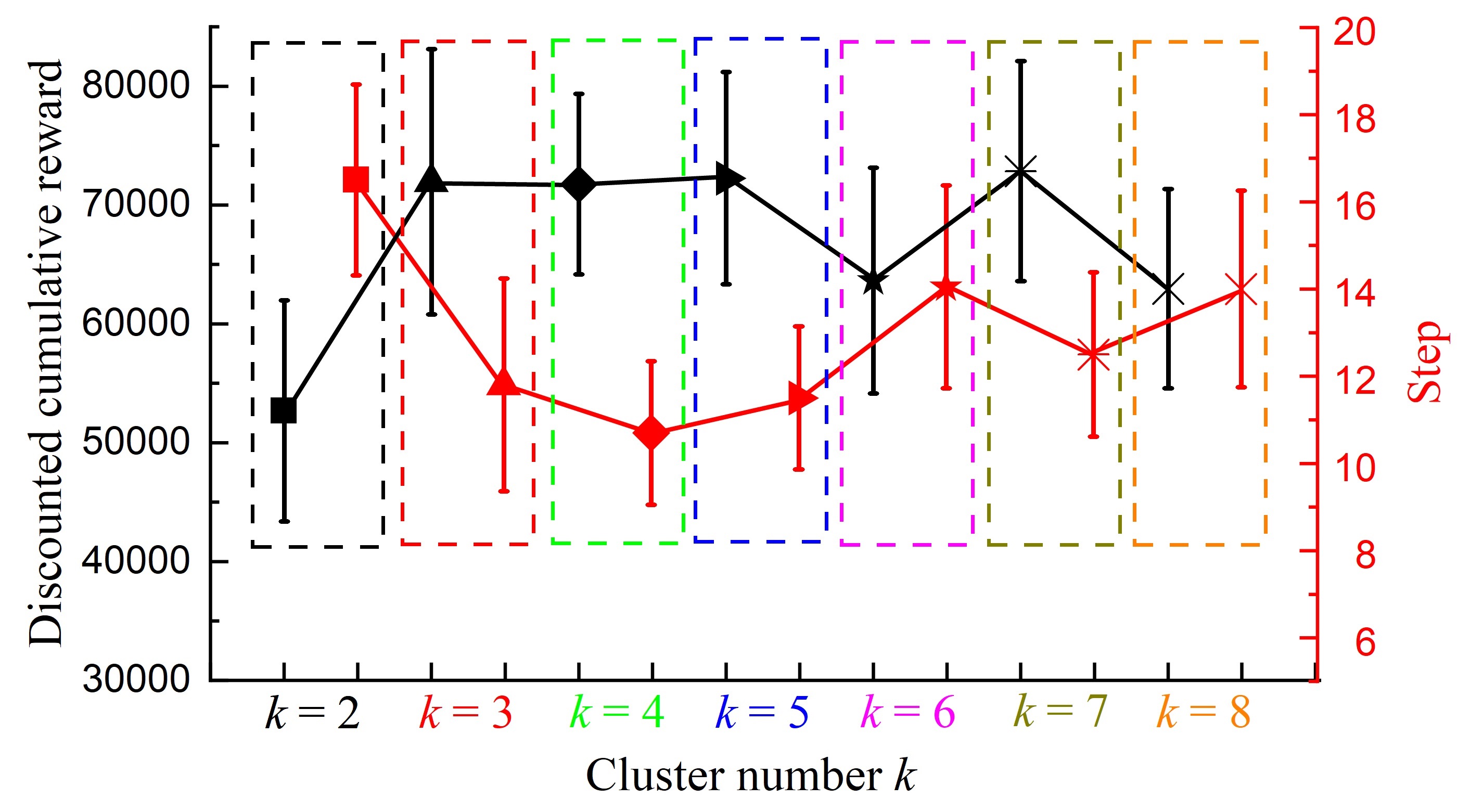}
\caption{Results for different cluster numbers $k$.}
\label{fig:side11}
 \end{figure}

\paragraph{Results with another simulator:}

Our method is adaptable to any mobile robot with similar sensors and configurations, not limited to the Fetch robot. As an example, we integrated it with the Stretch robot simulator to further validate the framework's applicability. Despite structural differences between the Stretch and Fetch robots, we tested the pre-trained Fetch network on Stretch without new data to evaluate its generalization in updating grid belief via the head camera. Problem settings from Section~\ref{s711} were retained, except for removing lift motion $l_h$. In Loose1 and Fig.~\ref{5} scenarios, the Stretch robot achieved discounted cumulative rewards, steps, and success rates of ${85181.2}\pm{9092.8}~|~{7.4}\pm {1.3}$~$|$~${100\%}$ (Loose1) and {${52558.4}\pm{5255.8}~|~{17.5}\pm {3.0}$~$|$~${100\%}$} (Fig.~\ref{5}), slightly lower than Fetch results (Table~\ref{sample-table}, Fig.~\ref{fig:side}) but showing robustness. We recollect the data using the Stretch robot simulator, re-train the network with the same settings, and finally re-run the whole planning. The updated results are ${88813.3}\pm{6480.2}~|~{6.8}\pm {1.1}$~$|$~${100\%}$ (Loose1) and ${65189.3}\pm{9459.3}~|~{13.8}\pm {3.0}$~$|$~${100\%}$ (Fig.~\ref{5}), aligning closely with the Fetch robot outcomes. A visual example of the Stretch robot completing Loose1 in 6 steps is shown in Fig.~\ref{8}.

\begin{figure}[!htb]
\centering
\includegraphics[width=0.9\linewidth]{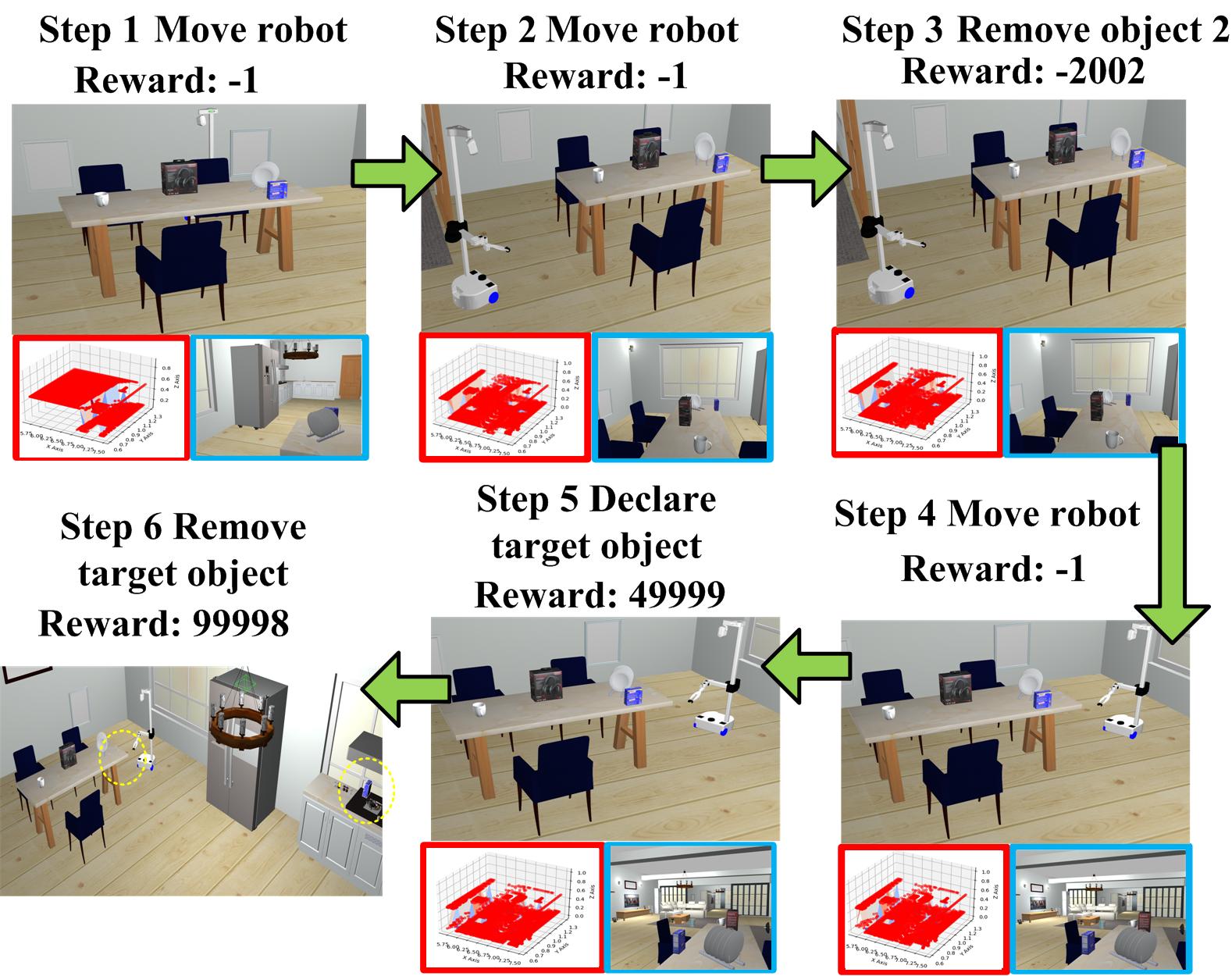}
\caption{Results for Loose1 scenario using Stretch robot.}
\label{8}
 \end{figure}

\subsection{Real-world experiments}

\begin{figure}[!htb]
\centering
\includegraphics[width=0.6\linewidth]{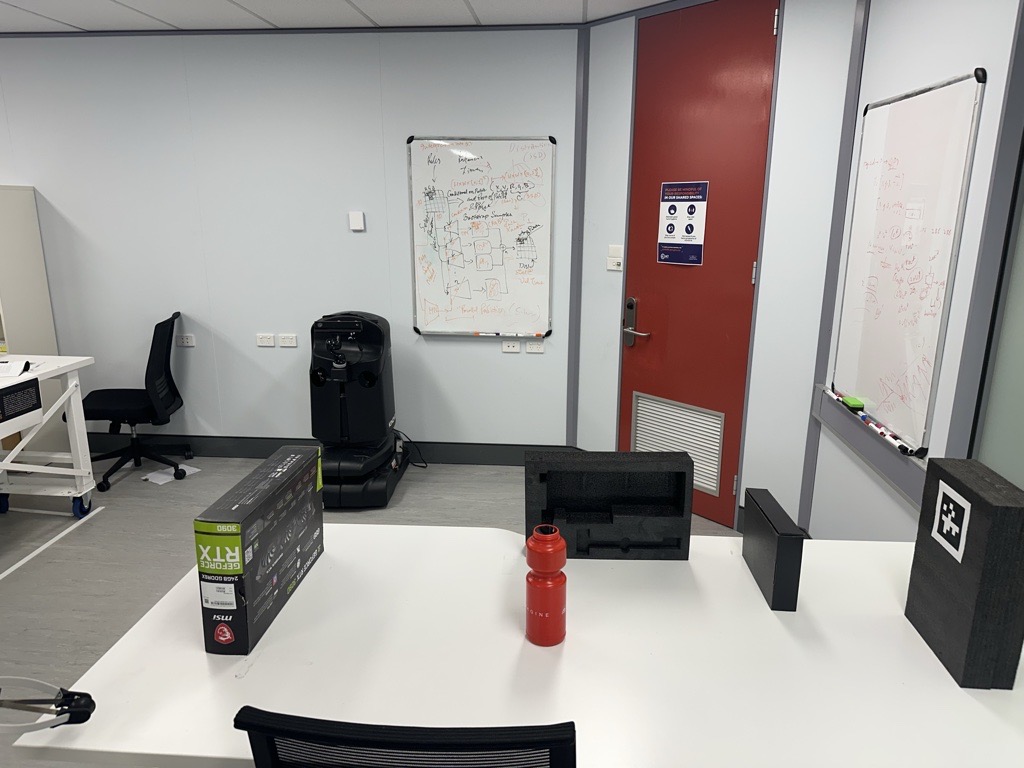}
\caption{The real-world planning environment.}\label{6a}
\end{figure}

\begin{figure}[!htb]
\centering
	\includegraphics[width=0.8\linewidth]{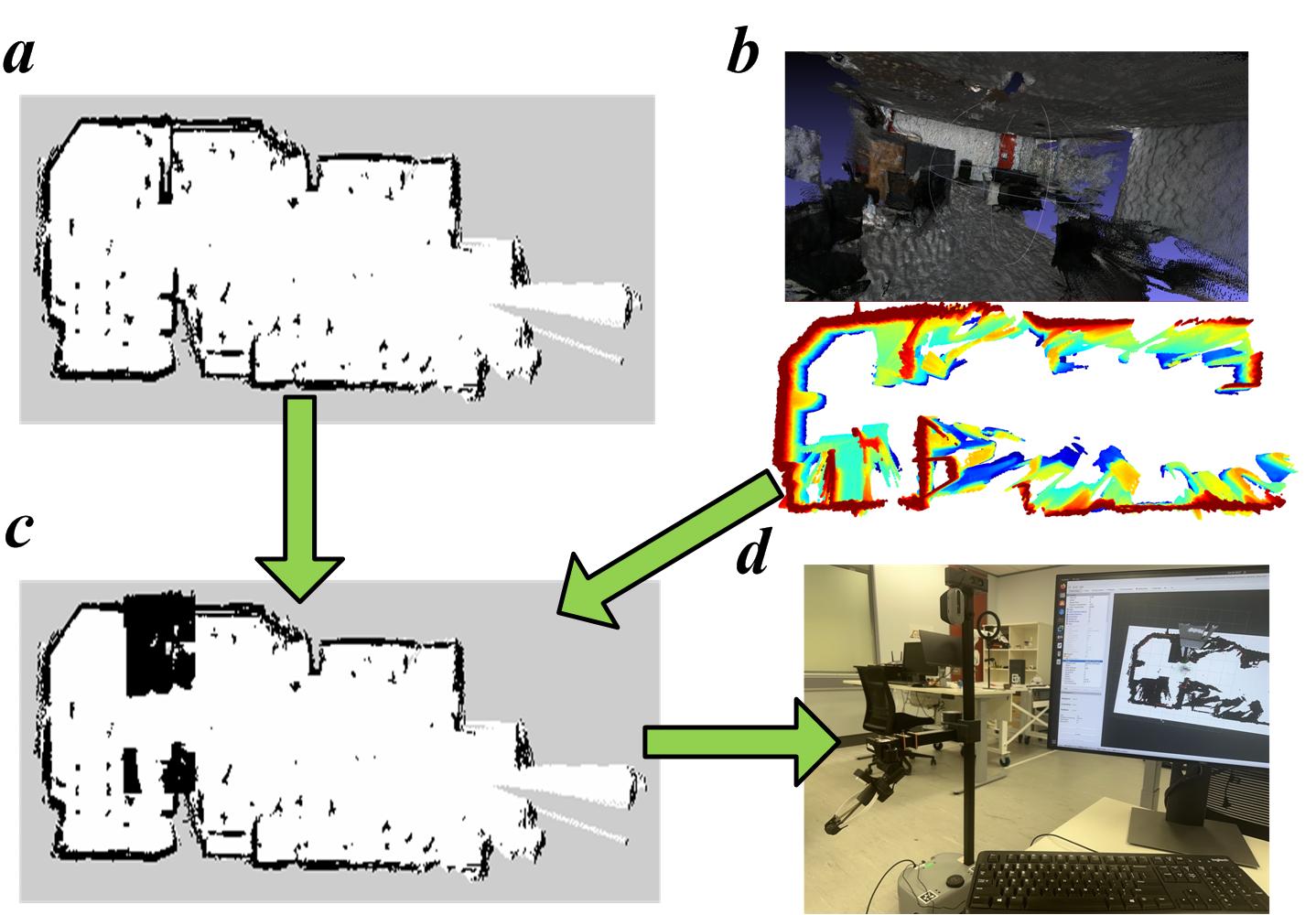}
	\caption{The used maps in real world experiments.}\label{6}
\end{figure}

We validate our method on the real-world Stretch platform (Fig.~\ref{6a}). The environment's 2D occupancy grid and point cloud maps, generated with RTAB-SLAM~\cite{labbe2019rtab}, are shown in Fig.~\ref{6} a and b for localization. These maps are fused to create a safer navigation map with expanded obstacle areas (Fig.~\ref{6} c and d). While object-level primitives are assumed to be fully implemented (which is challenging in reality), performance may decrease due to failed actions, resulting in a success rate below 100\% (6 successes in 20 trials). The robot's goal is to locate a red bottle. In 20 real-world trials, all attempts with successful primitive actions were successful, while others failed due to the failure of the primitives because of hardware, communication, or perception issues. To illustrate the object search process, Fig.~\ref{4xxx} shows the real robot's action sequence in 5 steps, including pose adjustments, camera orientation changes, and declarations of obstacles and the target object (yellow dashed circle). Fig.~\ref{6b} highlights four cases involving primitives, with additional experimental results available in the attached video.  

\begin{figure}[!htb]
\centering
	\includegraphics[width=0.9\linewidth]{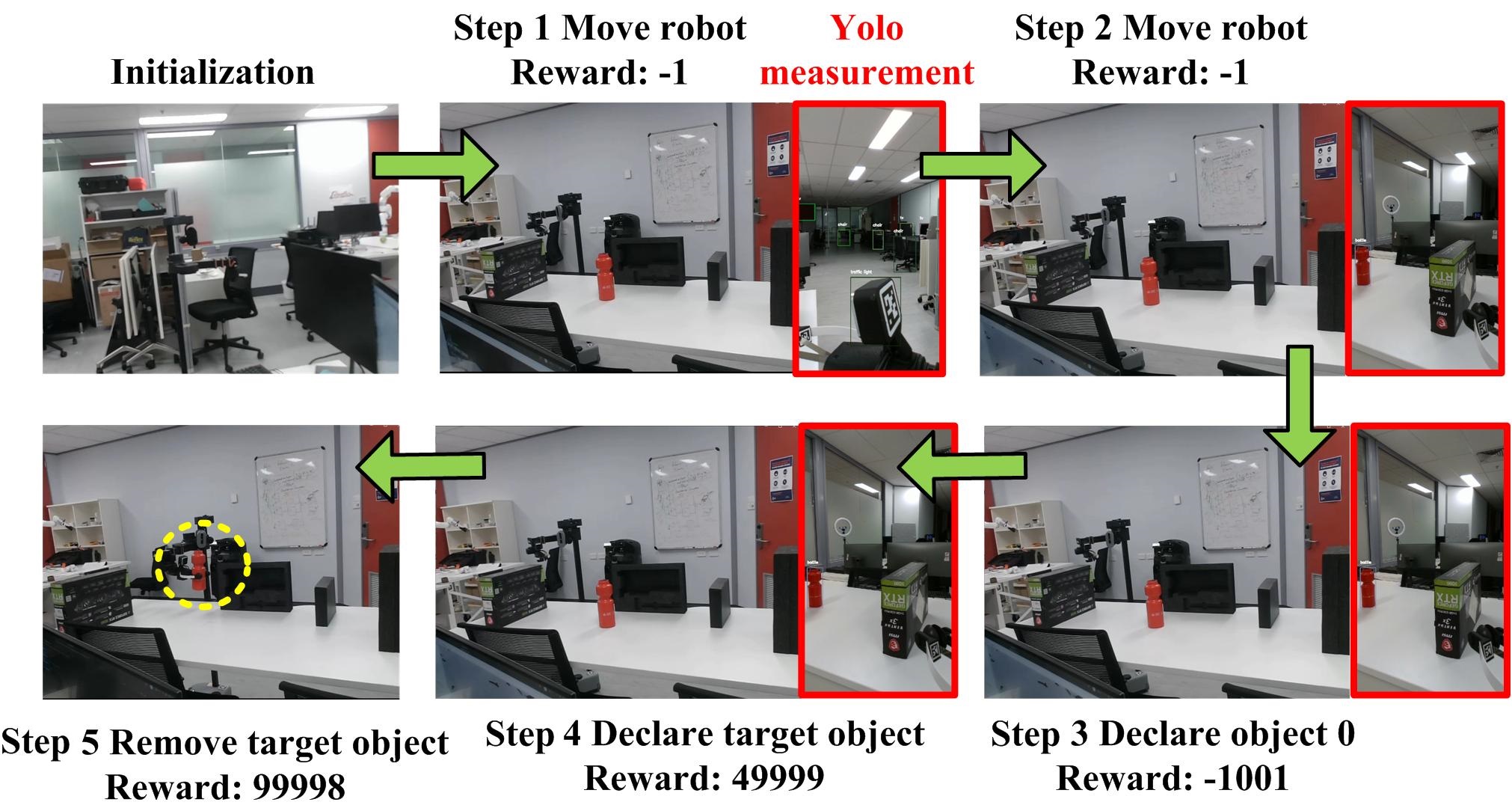}
	\caption{Real-world experimental results.}\label{4xxx}
\end{figure}

\begin{figure}[!htb]
\centering
	\includegraphics[width=0.9\linewidth]{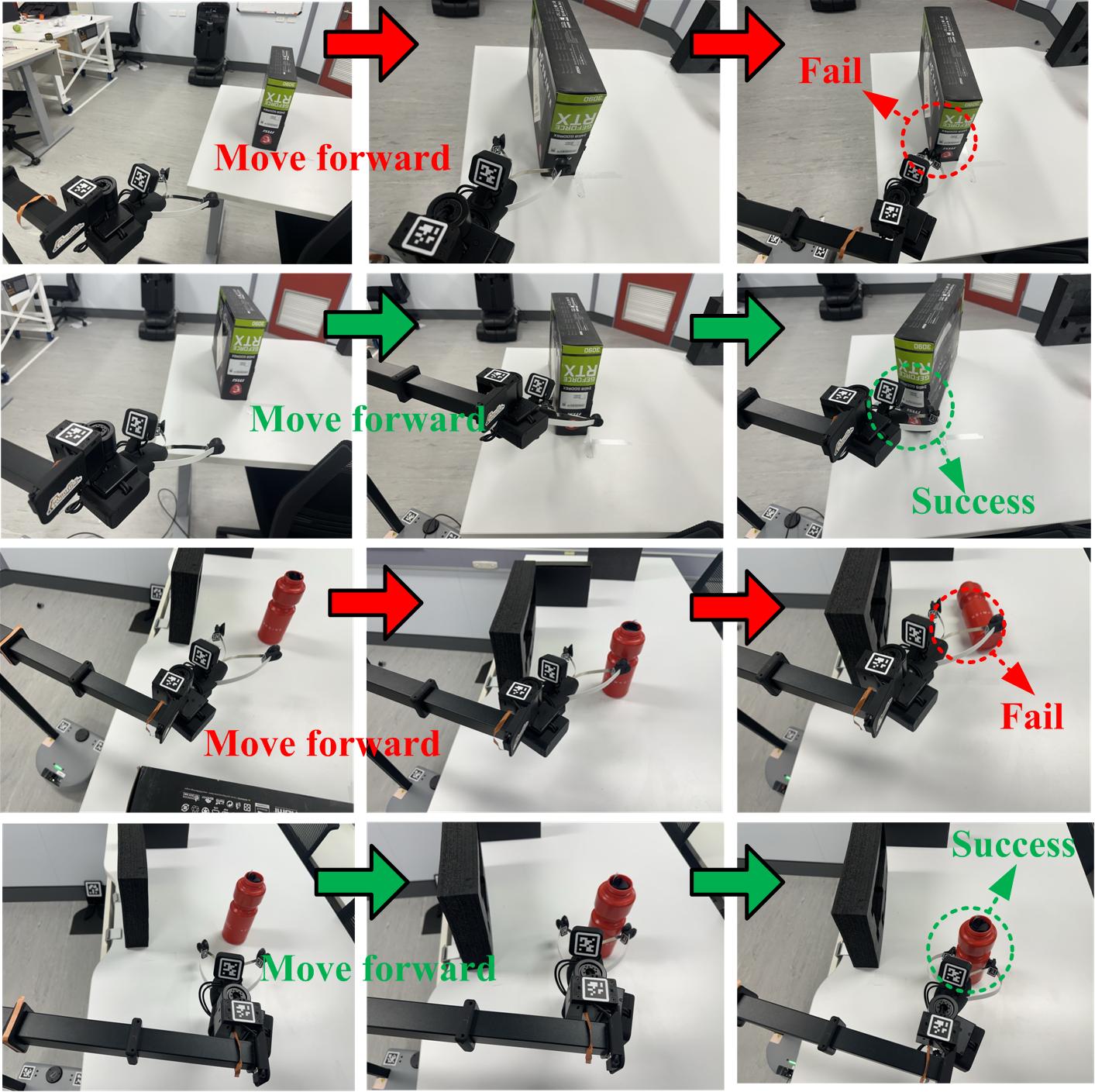}
	\caption{Failed and successful real-world primitive actions.}\label{6b}
\end{figure}

\subsection{Comparison with non-POMDP framework}
All previous results were obtained using our proposed POMDP framework. In this section, we compare our method with several representative non-POMDP planning frameworks, including Random, NPF-G, SGoLAM~\cite{KimJ}, SayPlan~\cite{Rana1953Sayplan}, and MoMa-LLM~\cite{HonerkampHonerkamp}. Since these planning methods differ substantially in their original problem formulations and configuration settings, we adapt them to our task by minimally revising their action spaces and unifying the perception module across all methods, ensuring a fair comparison. We first provide a qualitative summary of the original capabilities of these methods, as reported in the corresponding references, in Table~\ref{sample-table211xxxx}.


\begin{table*}[!ht]
	\caption{Ablation study based on NPF method}
	\label{sample-table211xxxx}
	\centering
	\begin{tabular}{lllll}
		\toprule
		Scenarios&Simulator& Semantic Perception& Interactive actions& Uncertainty\\
		\midrule
     Ours & Gazebo & Real (match, pose, size) & True(remove objects)& True\\
     SGoLAM &  Habitat & Real (match, position)& False& True\\
     Sayplan & No mention & Pre-given& False& False\\
     MoMa-LLM &   iGibson&Pre-given(QR code)& True(open cover)& False\\
		\bottomrule
	\end{tabular}
\end{table*}
As shown in Table~\ref{sample-table211xxxx}, our framework is implemented in the high-fidelity Gazebo environment and operates with minimal prior information about the scene. The entire pipeline, including perception and manipulation actions, is designed to be directly transferable to real robotic systems with very little modification. In contrast, LLM-based methods such as SayPlan and MoMa-LLM typically rely on pre-given semantic annotations or simulator-provided ground-truth information. For example, MoMa-LLM directly accesses pixel-level semantic perception from the iGibson simulator and employs QR-code markers in real-world experiments, both of which are treated as uncertainty-free prior knowledge (marked as Pre-given in Table~\ref{sample-table211xxxx}). These simplified assumptions make certain actions easier, for instance, exploration can be performed by a single rotation that immediately reveals all object labels. Moreover, SGoLAM and SayPlan do not support interactive actions that modify the environment configuration. While MoMa-LLM includes limited interactions such as opening doors or refrigerators, it does not allow interactive manipulation of detected objects. In contrast, our framework explicitly models environment interaction and uncertainty, enabling more realistic and transferable decision-making. We next detail the compared methods:

\paragraph{Random} This baseline adopts a random high-level policy that selects actions uniformly at random from both discrete and continuous action spaces. Specifically, we first sample a scalar $\gamma_s \sim \mathcal{U}(0,1)$. If $\gamma_s > 0.3$, a continuous action is selected, and its parameters are sampled uniformly from the predefined continuous ranges. If $\gamma_s \leq 0.3$, we check whether any objects are available for declaration or removal. If such objects exist, one is randomly selected for declaration or removal; otherwise, a continuous action is selected as in the case of $\gamma_s > 0.3$.

\paragraph{NPF-G} This baseline employs a greedy high-level decision-making strategy over both discrete and continuous actions. Since object declaration and target-object removal consistently yield positive rewards, these actions are always selected whenever available. Otherwise, a scalar $\gamma_s \sim \mathcal{U}(0,1)$ is sampled and compared with 0.1. If $\gamma_s > 0.1$, a continuous action is selected by choosing the one with the highest signal-to-noise ratio among actions predicted as efficient by the trained neural network, with a likelihood of at least 20\%. If $\gamma_s \leq 0.1$, the policy checks for available removal actions and randomly selects one if present; if none are available, it defaults to the same continuous action selection strategy used when $\gamma_s > 0.1$.

\paragraph{SGoLAM} SGoLAM~\cite{KimJ} achieved top performance in the CVPR 2021 MultiON Challenge using a purely non-learning-based strategy. The method first conducts frontier-based exploration to search the environment and locate the target object, and then switches to a dedicated planner for efficient navigation. To ensure a fair comparison with our framework, which explicitly provides the workspace area, we reduce the frontier exploration step size and perform exploration directly in the continuous position domain (highlighted by red rectangles in Fig.~\ref{5}). Without this modification, SGoLAM exhibits significantly degraded performance in our setting. Moreover, since the continuous exploration positions are already close to the objects, no additional navigation step is required to further increase the target object’s pixel size after frontier exploration. For perception, we use the same configuration as our method. At each frontier point, exploration is performed by selecting the three most distinctive orientations from 100 uniformly sampled candidates.

\paragraph{Sayplan}: This baseline follows a prompt structure adapted from SayPlan~\cite{Rana1953Sayplan}, reformulated to use instruction-based reasoning and a scene graph represented in JSON format. To better fit our task setting, we extend the original action set by introducing additional primitive actions, including object declaration and removal of declared objects. Example prompts for high-level LLM-based reasoning and low-level navigation toward known objects are provided in Appendix~\ref{appendix_g}. The original SayPlan framework was evaluated using GPT-3 and GPT-4~\cite{OpenAI_GPT}. To assess the performance of this framework across different LLMs under identical prompts, we evaluate three models: gpt-3.5-turbo, gpt-4, and gpt-5, denoted as SayPlan3, SayPlan4, and SayPlan5, respectively, in Table~\ref{sample-table21222}.Each method, except ours (20 times), is repeated 10 times because of the high cost of API usage. For perception, we adopt the same configuration as SGoLAM, selecting the three most distinctive orientations at each exploration point.

\paragraph{MoMa-LLM} This method~\cite{HonerkampHonerkamp} employs high-level LLM-based reasoning over scalable and dynamic scene graphs, combined with efficient low-level policies for interactive manipulation tasks. An example of the high-level reasoning prompt is provided in Appendix~\ref{appendix_g}. For navigation actions, MoMa-LLM uses the same low-level prompt structure as SayPlan, as shown in Appendix~\ref{appendix_g}. The original MoMa-LLM framework was evaluated using GPT-3 and GPT-4~\cite{OpenAI_GPT}. To assess its performance across different LLM backends under identical prompts, we evaluate three models, including gpt-3.5-turbo, gpt-4, and gpt-5, and denote them as MoMa3, MoMa4, and MoMa5, respectively, in Table~\ref{sample-table21222}. The original MoMa-LLM method does not add the geometry properties of the objects, like the poses and sizes of the objects, in the high-level reasoning prompt. So as to further improve the capacity of this framework in this task, we add these information and the revised prompts are shown in Appendix~\ref{appendix_g}. The additional results using gpt-5 are shown in Table~\ref{sample-table21222} and denote it as MoMa5\_g.
For perception, we adopt the same configuration as SGoLAM and SayPlan, selecting the three most distinctive orientations at each exploration point.




The comparison results are shown in Table~\ref{sample-table21222}. Because the discounted cumulative rewards do not make sense here, we ignore them in this table. We can find that our method shows the better performance than the others.



\begin{table*}
  \caption{Steps and successful rate (within 50 steps) for non-POMDP framework}
  \centering
  \label{sample-table21222}
  \small{
  \begin{tabular}{l|l|l|l|l}
    \cmidrule(lr){1-5}
    Scenarios&Loose1& Hidden1& Covered1&Complex1 \\
    \cmidrule(lr){1-5}
    Random& $ 45.0\pm  {7.4}~|~{30\%}$ &  $ 50.0\pm  {0.0}~|~{0\%}$  & $50.0\pm  {0.0}~|~{0\%}$ &  $ 50.0\pm  {0.0}~|~{0\%}$\\
    SGoLAM& $  21.8  \pm  { 8.9 }~|~{ 90\%}$    & $27.7 \pm  {12.4}~|~{80 \%}$ &  $  48.1  \pm  {  3.4 }~|~ { 20\%}$&  $50.0\pm  {0.0}~|~{0\%}$ \\
    NPF-G&  $  11.2 \pm  { 3.4 }~|~\bm{100\%}$   & ${14.4}\pm {8.1}~|~\bm{100\%}$ & $\bm{18.0}\pm \bm{5.7}~|~\bm{100\%}$ &  ${35.7}\pm {10.2}~|~{60\%}$\\
    Sayplan3&  $ 42.5\pm  {9.0}~|~{30\%}$ & $46.6\pm  {7.7}~|~{10\%}$  & $ 50.0\pm  {0.0}~|~{0\%}$ & $ 50.0\pm  {0.0}~|~{0\%}$ \\
    Sayplan4&  $ 21.0\pm  {7.3}~|~{90\%}$ &  $ 35.9\pm  {13.1}~|~{40\%}$ & $ 46.7\pm  {5.9}~|~{20\%}$ & $ 50.0\pm  {0.0}~|~{0\%}$ \\
    Sayplan5&  $  12.2\pm  { 3.5}~|~\bm{100\%}$   & $ 20.5  \pm  { 10.0}~|~{90\%}$ & $ 47.4\pm  {5.8}~|~{10\%}$ &  $ 50.0\pm  {0.0}~|~{0\%}$\\
    MoMa3&  $ 32.8\pm  {11.1}~|~{60\%}$   & $ 42.3\pm  {9.4}~|~{30\%}$ &  $ 47.8\pm  {5.0}~|~{10\%}$& $ 50.0\pm  {0.0}~|~{0\%}$ \\
    MoMa4&   $ 19.7\pm  {7.0}~|~\bm{100\%}$  & $46.8\pm  {5.4}~|~{20\%}$ & $ 45.5\pm  {5.6}~|~{30\%}$ & $ 50.0\pm  {0.0}~|~{0\%}$ \\
    MoMa5&    $10.7\pm  {2.7}~|~\bm{100\%}$   & $30.1\pm  {11.2}~|~{70\%}$ & $38.4\pm  {9.4}~|~{50\%}$ &  $48.2\pm  {4.1}~|~{10\%}$\\
    MoMa5\_g&   $ 10.6\pm  {1.8 }~|~\bm{100\%}$   & $27.7\pm  {10.4}~|~{80\%}$ & $32.0\pm  {9.9}~|~{70\%}$ & $49.0\pm  {2.3}~|~{10\%}$ \\
    \cmidrule(lr){1-5}
    Ours& $ \bm{6.7}\pm  \bm{0.8}~|~\bm{100\%}$  & $  \bm{8.3} \pm \bm{1.3}~|~\bm{100\%}$&${19.4}\pm {3.8}~|~\bm{100\%}$ & $\bm{33.8}\pm \bm{5.0}~|~\bm{95\%}$ \\
    \cmidrule(lr){1-5}
  \end{tabular}}
\end{table*}
Most of these methods are developed and evaluated in large-scale environments using simulators such as iGibson and Habitat. When applied to our single-room setting, they exhibit substantially degraded performance. This motivates a discussion of several common misconceptions in the literature, which we summarize as follows:

\begin{question}\label{Q5}
Is a single-room environment with clustered objects inherently easier for mobile robots than multi-room, multi-level environments with sparse objects? Are existing LLM-based methods currently superior to model-based approaches?
\end{question}

\begin{discussion}
Our answers to both questions are \textbf{no}, as demonstrated by the experimental results presented earlier. 

A key reason is that the field of view (FOV) of common mobile robots becomes severely limited when operating close to objects, making visual blind spots frequent and unavoidable. At the same time, operating at longer distances is unfavorable for reliable object matching and precise manipulation, necessitating close-range interactions. Many existing works adopt simplified assumptions, for example, treating the object search task as complete once the target object enters a predefined sensing range, which are often unrealistic in practice. Such approaches also fail to fully exploit changes in the robot’s configuration, such as head orientation or lift height, thereby limiting their effectiveness in realistic object search scenarios, particularly when objects are densely clustered.

Another contributing factor to the inferior performance of LLM-based methods in our setting is their limited capacity to faithfully represent physical constraints and spatial uncertainty. Natural language abstractions are inherently ambiguous and lack the precision required for fine-grained robotic execution. Importantly, object search in a cluttered single-room environment constitutes a physically grounded task involving perception, manipulation, and configuration changes, rather than a purely topological navigation problem commonly studied in large-scale, sparsely populated environments.

Furthermore, we observe that for LLM-based methods, it remains challenging to design prompts that guarantee not only the feasibility but also the optimality of both the problem formulation and the resulting solution. Most existing studies primarily focus on whether a task can be completed, with limited analysis of optimal decision-making under uncertainty. This limitation is closely related to the absence of an explicit internal world model in current LLMs, which hinders long-horizon action–outcome simulation.

Nevertheless, our results indicate a clear improvement in long-horizon planning performance from GPT-3.5 to GPT-5 (Table~\ref{sample-table21222}), suggesting that progress in this direction is ongoing. Overall, we believe that integrating model-based reasoning tailored to single-room physical interactions with learning-based global reasoning mechanisms represents a more promising research direction, and is likely to lead to more realistic and robust solutions for real-world object search tasks.
\end{discussion}

\subsection{Limitation} 

Our approach faces several limitations. First, reliance on a pre-existing map is impractical for real-world robots, but this can be mitigated with advanced SLAM techniques. Second, achieving a 100\% success rate for primitive actions in real-world scenarios is unrealistic, affecting overall performance. Additionally, our use of point cloud segmentation may generate incorrect bounding boxes for objects with large contact areas, leading to errors in data association and belief updates. Object detection methods like YOLO, SIFT, and color matching also face challenges in environments with limited features, low coverage, or poor lighting. But we believe these issues can be addressed through advances in perception, navigation, and execution, which are outside the scope of this paper.

\section{Conclusion and future work}
\label{s8}
This paper introduces GNPF-$k$CT, a novel POMDP-based framework and solver for object search tasks in a 3D environment using a mobile robot manipulator. The problem is modeled as an 8-tuple POMDP, focusing on FOV adjustments and belief updates for observed objects. By incorporating a guessed target object, our approach efficiently locates hidden objects by removing obstacles, a limitation in existing methods. Our framework combines a pre-trained NP network to select potential primitive actions from Gazebo simulation data with adaptive $k$-clustering hypersphere partitioning to refine the continuous action space for efficient MCTS exploration. This integrates short-term learning with a compact long-term POMDP framework for improved performance. To address challenges from growing state spaces, we reuse the previous belief tree and apply a revised UCB1 bound, incorporating hypersphere range and belief distance for efficient and robust decision-making. Experiments show that, our solver surpasses classical methods in cumulative reward, task steps, and success rate under identical computational constraints. Notably, our framework requires only maps, photos, and onboard sensors, eliminating the need for geometric models or additional object-specific data, and making it highly practical for real-world applications.

This paper lays the groundwork for autonomous object search. Future work will integrate advanced methods, such as LLMs~\cite{Zirui2024Large} to embed expert knowledge via prompts~\cite{Tang2025Task}, or world model~\cite{Zheng_Yupeng} to fast predict system in latent space, and enable autonomous map creation and workspace generation. More importantly, in experimental results, we find that one of the main limitations is the success rate of the complicated primitive actions. Even though we can consider these uncertainties in the planning framework, they will still cause the low success rate of the whole system with poor preventive implementation. The well-signed Vision-Language-Action (VLA) method, like RT-2~\cite{Zitkovich_B} and Octo~\cite{Team_Octo}, may be helpful in completing better primitive actions. These advancements aim to develop fully autonomous, end-to-end systems capable of identifying any target in unknown environments with a single robotic platform and onboard sensors.

\section*{Appendix}
\appendix
\section{Proof for Theorem~~\ref{t_add_new}}\label{a1}
We rewrite the formulation by $g'(\theta)\triangleq g(\bm{a},~\bm{s},~\alpha(\mathcal{P}))$ given fixed $\bm{s}$ and $\alpha(\mathcal{P})$. Based on $g'(\theta) \sim nn(\mu(\theta),~\sigma(\theta))$, with the fixed input $\theta_i$ and $i\leq 1$, we have the Gaussian distribution $g'(\theta_i)\sim \mathcal{N}(\mu(\theta_i),~\sigma(\theta_i))$. Let $z_i = \frac{g'(\theta_i) - \mu(\theta_i)}{\sigma(\theta_i)}\sim \mathcal{N}(0,1)$. For a Gaussian distribution with mean 0 and variance 1, we have:
\begin{equation}\label{smooth-constraint}
	\begin{aligned}
	&Pr(z_i > \eta_i)\\ & = \int_{\eta_i}^{+\infty} \frac{1}{\sqrt{2\pi}} \exp^{-z^2/2}dz\\
 & = \int_{\eta_i}^{+\infty} \frac{1}{\sqrt{2\pi}} \exp^{-(z-\eta_i)^2/2- z\eta_i +\eta_i^2/2}dz\\
 & =\exp^{-\eta_i^2/2} \int_{\eta_i}^{+\infty} \frac{1}{\sqrt{2\pi}} \exp^{-(z-\eta_i)^2/2- z\eta_i +\eta_i^2}dz\\
 & =\exp^{-\eta_i^2/2} \int_{\eta_i}^{+\infty} \frac{1}{\sqrt{2\pi}} \exp^{-(z-\eta_i)^2/2}\exp^{- z\eta_i +\eta_i^2}dz.
	\end{aligned}
\end{equation} 
were  $Pr(\star)$ is the probability of event $\star$. Because we can set $\eta_i> 0$ and $z_i > \eta_i$, we have: $- z\eta_i +\eta_i^2 <0$. So, we have:
\begin{equation}\label{smooth-constraint}
	\begin{aligned}
	&Pr(z_i > \eta_i) \leq \exp^{-\eta_i^2/2} \int_{\eta_i}^{+\infty} \frac{1}{\sqrt{2\pi}} \exp^{-(z-\eta_i)^2/2}dz\\& = \exp^{-\eta_i^2/2}/2.
	\end{aligned}
\end{equation} 

Introduce $z_i = \frac{g'(\theta_i) - \mu(\theta_i)}{\sigma(\theta_i)}$, we have:
\begin{equation}\label{smooth-constraint}
	\begin{aligned}
	&Pr(g'(\theta_i)-\mu(\theta_i)>\eta_i\sigma(\theta_i)) \leq \exp^{-\eta_i^2/2}/2,\\
& Pr(g'(\theta_i)-\mu(\theta_i)<-\eta_i\sigma(\theta_i)) \leq \exp^{-\eta_i^2/2}/2.
	\end{aligned}
\end{equation} 

We have:
\begin{equation}\label{smooth-constraint}
	\begin{aligned}
	&Pr(|g'(\theta_i)-\mu(\theta_i)|>\eta_i\sigma(\theta_i)) \leq  \exp^{-\eta_i^2/2}.
	\end{aligned}
\end{equation}

The complementary set of $|g'(\theta_i)-\mu(\theta_i)|>\eta_i\sigma(\theta_i)$ is $|g'(\theta_i)-\mu(\theta_i)|\leq\eta_i\sigma(\theta_i)$ and its corresponding probability is bigger than $1-\exp^{-\eta_i^2/2}$. Let $\delta =\exp^{-\eta_i^2/2} \Rightarrow \eta_i = \sqrt{-2\ln(\delta)}$, we have:

\begin{equation}\label{smooth-constraint}
	\begin{aligned}
	&Pr(|g'(\theta_i)-\mu(\theta_i)|\leq \eta_i\sigma(\theta_i)) \geq 1-\delta.
	\end{aligned}
\end{equation} 

We have the scoring function $g'(\theta_i)$ is bounded by $(\mu(\theta_i) - \sqrt{-2\ln(\delta)}\sigma(\theta_i),~\mu(\theta_i) + \sqrt{-2\ln(\delta)}\sigma(\theta_i))$ with a probability bigger than $1 - \delta$. So, if its lower bound $\mu(\theta_i) - \sqrt{-2\ln(\delta)}\sigma(\theta_i)> 0 \Leftrightarrow \mu(\theta_i) >\sqrt{-2\ln(\delta)}\sigma(\theta_i)$, the scoring function satisfying:
\begin{equation}\label{smooth-constraint}
	\begin{aligned}
	&Pr(g'(\theta_i)>0) > 1 - \delta.
	\end{aligned}
\end{equation}

\section{Network structure for scoring function}\label{ab2}
The structure of the used NPs model for the scoring function is shown in Fig.~\ref{fig:landscape}.

    \begin{figure*}[htb]
\rotatebox{90}{\includegraphics[width=7in]{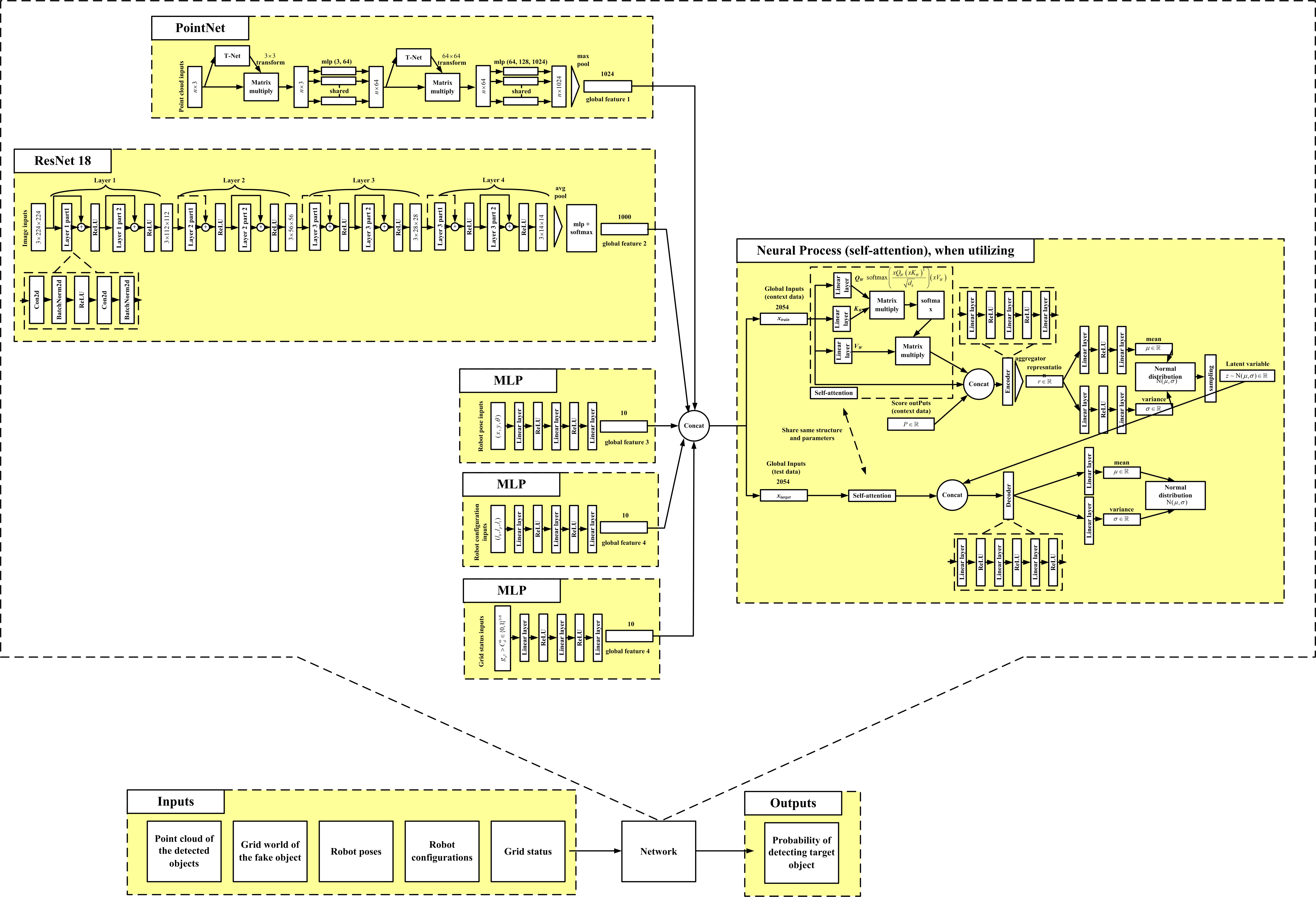}}
\caption{Architecture of the used network. }
\label{fig:landscape}
    \end{figure*}

\section{Proof for Theorem~~\ref{t_add_new1a}}\label{a3}
Let's consider one of the new extended beliefs $\bm{b}(\bm{s}_{new}': \bm{s}'\rightleftarrows \bm{s}_{add}')$. Based on the well-known $\alpha$-vector and $depth=1$, we have the optimal value of belief $\bm{b}(\bm{s}_{new}': \bm{s}'\rightleftarrows \bm{s}_{add}')$ can be written as:
\begin{equation}\label{bound}
	\begin{aligned}
&V^*(\bm{b}(\bm{s}_{new}': \bm{s}'\rightleftarrows \bm{s}_{add}'))\\
= &\max_{\bm{a}} Q^*(\bm{b}(\bm{s}_{new}': \bm{s}'\rightleftarrows \bm{s}_{add}'),~\bm{a})\\
=&\sum_{\bm{s}_{new}'\in S}
\alpha(\bm{s}_{new}')\bm{b}(\bm{s}_{new}': \bm{s}'\rightleftarrows \bm{s}_{add}'),
	\end{aligned}
\end{equation}
where $\alpha(\bm{s}_{new}')=R(\bm{s}'_{new},~\bm{a}^*) + \gamma V^*(\bm{s}'_{new},~\bm{b}(\bm{s}_{new}': \bm{s}'\rightleftarrows \bm{s}_{add}'),~\bm{a}^*)$ and $\bm{a}^*$ is the optimal action. It is noted that the $\alpha$-vector $\alpha(\bm{s}_{new}')$ is bounded by $\frac{R_{max}}{1-\gamma}$, introducing the particle representation for the belief $\bm{b}(\bm{s}_{new}': \bm{s}'\rightleftarrows \bm{s}_{add}')=\frac{1}{N}\sum^N_{i=1}\mathbb{I}(\bm{s}_{new}'={\bm{s}_{new}^i}': \bm{s}'\rightleftarrows \bm{s}_{add}')$, so we have:
\begin{equation}\label{bound}
	\begin{aligned}
&\|V^*(\bm{b}(\bm{s}_{new}': \bm{s}'\rightleftarrows \bm{s}_{add}'))-V^*(\bm{b}(\bm{s}_{new}': \bm{s}'\rightarrow \\&\bm{s}_{add}'))\|_1\\
&= \|\frac{1}{N}\sum^N_{i=1}\alpha(\bm{s}_{new}') \mathbb{I}(\bm{s}_{new}'={\bm{s}_{new}^i}': \bm{s}'\rightleftarrows \bm{s}_{add}')\\
&-\frac{1}{N}\sum^N_{i=1}\alpha(\bm{s}_{new}') \mathbb{I}(\bm{s}_{new}'={\bm{s}_{new}^i}': \bm{s}'\rightarrow \bm{s}_{add}')\|_1\\
&\leq \|\frac{1}{N}\sum^N_{i=1}\frac{R_{max}}{1-\gamma} \mathbb{I}(\bm{s}_{new}'={\bm{s}_{new}^i}': \bm{s}'\rightleftarrows \bm{s}_{add}')\\ 
&-\frac{1}{N}\sum^N_{i=1}\frac{R_{max}}{1-\gamma} \mathbb{I}(\bm{s}_{new}'={\bm{s}_{new}^i}': \bm{s}'\rightarrow \bm{s}_{add}')\|_1
	\end{aligned}
\end{equation}

\begin{equation}\nonumber
	\begin{aligned}
&=\frac{R_{max}}{1-\gamma}\|\bm{b}(\bm{s}_{new}': \bm{s}'\rightleftarrows \bm{s}_{add}')-\bm{b}(\bm{s}_{new}': \bm{s}'\rightarrow \\&\bm{s}_{add}')\|_1\\
& \leq \frac{R_{max}}{1-\gamma} \delta.
	\end{aligned}
\end{equation}

Based on Eq.~\eqref{smooth-1} and the optimal action $\bm{a}^*$, we have:

\begin{equation}\nonumber
\begin{aligned}
&\|V^*(\bm{b}(\bm{s}_{new})) - \hat{V}^*(\bm{b}(\bm{s}_{new}))\|_1\\
&=\| \sum_{\bm{s}_{new}\in S} R(\bm{s}_{new}, \bm{a}^*, \bm{s}_{new}') \bm{b}(\bm{s}_{new})+  \gamma\sum_{o‘\in O}
\end{aligned}
\end{equation}

\begin{equation}\nonumber
\begin{aligned}
&\sum_{\bm{s}_{new}'\in S}Z(\bm{s}_{new}', \bm{a}^*, \bm{o}')T(\bm{s}_{new}, \bm{a}^*, \bm{s}_{new}')V^*(\bm{b}(\bm{s}_{new}'\\
& :\bm{s}'\rightleftarrows \bm{s}_{add}'))-\sum_{\bm{s}_{new}\in S} R(\bm{s}_{new}, \bm{a}^*, \bm{s}_{new}') \bm{b}(\bm{s}_{new})\\& - \gamma\sum_{o’\in O}\sum_{\bm{s}_{new}'\in S}Z(\bm{s}_{new}', \bm{a}^*, \bm{o}')T(\bm{s}_{new}, \bm{a}^*, \bm{s}_{new}')\\
&V^*(\bm{b}(\bm{s}_{new}': \bm{s}'\rightarrow \bm{s}_{add}'))\|_1\\
&=\| \gamma\sum_{o‘\in O}\sum_{\bm{s}_{new}'\in S}Z(\bm{s}_{new}', \bm{a}^*, \bm{o}')T(\bm{s}_{new}, \bm{a}^*, \bm{s}_{new}')\\&V^*(\bm{b}(\bm{s}_{new}': \bm{s}'\rightleftarrows \bm{s}_{add}'))- \gamma\sum_{o’\in O}\sum_{\bm{s}_{new}'\in S}Z(\bm{s}_{new}', \bm{a}^*\\
&, \bm{o}')
\end{aligned}
\end{equation}

\begin{equation}\label{EQ_28}
\begin{aligned}
&\left.T(\bm{s}_{new}, \bm{a}^*, \bm{s}_{new}')V^*(\bm{b}(\bm{s}_{new}': \bm{s}'\rightarrow \bm{s}_{add}'))\right\|_1\\
&\leq  \gamma\sum_{o‘\in O}\sum_{\bm{s}_{new}'\in S}Z(\bm{s}_{new}', \bm{a}^*, \bm{o}')T(\bm{s}_{new}, \bm{a}^*, \bm{s}_{new}')\\
&\frac{R_{max}}{1-\gamma} \delta\\
&\leq\frac{\gamma R_{max}}{1-\gamma} \delta.
\end{aligned}
\end{equation}

The proof is completed.

\section{Proof for Corollary~~\ref{t_add_new3}}\label{a4}
Let's consider one of the approximated belief nodes in the $depth$-1-th and $depth$-th layers with belief $\bm{b}(\bm{s}_{new}^{(depth-1)}: \bm{s}^{(depth-1)}\rightarrow \bm{s}_{add}^{(depth-1)})$ and $\bm{b}(\bm{s}_{new}^{(depth)}: \bm{s}^{(depth)}\rightarrow \bm{s}_{add}^{(depth)})$. Based on Eq.~\ref{EQ_28}, compared between the approximate $\bm{b}(\bm{s}_{new}^{(depth)}: \bm{s}^{(depth)}\rightarrow \bm{s}_{add}^{(depth)})$ and the correct belief $\bm{b}(\bm{s}_{new}^{(depth)}: \bm{s}^{(depth)}\rightleftarrows \bm{s}_{add}^{(depth)})$, its optimal value will satisfy:

\begin{equation}\label{boundxx}
	\begin{aligned}
&\|V^*(\bm{b}(\bm{s}_{new}^{(depth)}: \bm{s}^{(depth)}\rightleftarrows \bm{s}_{add}^{(depth)}))-\\&V^*(\bm{b}(\bm{s}_{new}^{(depth)}: \bm{s}^{(depth)}\rightarrow \bm{s}_{add}^{(depth)}))\|_1
\leq  \frac{ R_{max}}{1-\gamma} \delta.
	\end{aligned}
\end{equation}

Based on Eq.~\eqref{smooth-1} and the optimal action $\bm{a}^*$ and the received observation $\bm{o}'$, we have:

\begin{equation}\label{EQ_28}
\begin{aligned}
&\|V^*(\bm{b}(\bm{s}^{(depth-1)}_{new})) - \hat{V}^*(\bm{b}(\bm{s}^{(depth-1)}_{new}))\|_1\\
&=\| \sum_{\bm{s}^{(depth-1)}_{new}\in S} R(\bm{s}^{(depth)}_{new}, \bm{a}^*, \bm{s}^{(depth-1)}_{new})\\& \bm{b}(\bm{s}^{(depth-1)}_{new}: \bm{s}^{(depth-1)}\rightleftarrows \bm{s}_{add}^{(depth-1)})\\&+  \gamma\sum_{o’\in O}\sum_{\bm{s}_{new}^{(depth)}\in S}Z(\bm{s}_{new}^{(depth)}, \bm{a}^*, \bm{o}')T(\bm{s}_{new}^{(depth-1)},\\&\bm{a}^*, \bm{s}_{new}^{(depth)})V^*(\bm{b}(\bm{s}_{new}^{(depth)}: \bm{s}^{(depth)}\rightleftarrows \bm{s}_{add}^{(depth)}))\\
&-\sum_{\bm{s}^{(depth-1)}_{new}\in S} R(\bm{s}^{(depth)}_{new}, \bm{a}^*, \bm{s}^{(depth-1)}_{new})\\& \bm{b}(\bm{s}^{(depth-1)}_{new}: \bm{s}^{(depth-1)}\rightarrow \bm{s}_{add}^{(depth-1)})\\&-  \gamma\sum_{o’\in O}\sum_{\bm{s}_{new}^{(depth)}\in S}Z(\bm{s}_{new}^{(depth)}, \bm{a}^*, \bm{o}')T(\bm{s}_{new}^{(depth-1)},\\&\bm{a}^*, \bm{s}_{new}^{(depth)})V^*(\bm{b}(\bm{s}_{new}^{(depth)}: \bm{s}^{(depth)}\rightarrow \bm{s}_{add}^{(depth)}))\|_1\\
&\leq \| \sum_{\bm{s}^{(depth-1)}_{new}\in S} R(\bm{s}^{(depth)}_{new}, \bm{a}^*, \bm{s}^{(depth-1)}_{new})\\& \bm{b}(\bm{s}^{(depth-1)}_{new}: \bm{s}^{(depth-1)}\rightleftarrows \bm{s}_{add}^{(depth-1)})\\
&-\sum_{\bm{s}^{(depth-1)}_{new}\in S} R(\bm{s}^{(depth)}_{new}, \bm{a}^*, \bm{s}^{(depth-1)}_{new})\\& \bm{b}(\bm{s}^{(depth-1)}_{new}: \bm{s}^{(depth-1)}\rightarrow \bm{s}_{add}^{(depth-1)})\|_1+
\end{aligned}
\end{equation}

\begin{equation}\nonumber
\begin{aligned}
&\| \gamma\sum_{o’\in O}\sum_{\bm{s}_{new}^{(depth)}\in S}Z(\bm{s}_{new}^{(depth)}, \bm{a}^*, \bm{o}')T(\bm{s}_{new}^{(depth-1)},\\&\bm{a}^*, \bm{s}_{new}^{(depth)})V^*(\bm{b}(\bm{s}_{new}^{(depth)}: \bm{s}^{(depth)}\rightleftarrows \bm{s}_{add}^{(depth)}))\\
&-  \gamma\sum_{o’\in O}\sum_{\bm{s}_{new}^{(depth)}\in S}Z(\bm{s}_{new}^{(depth)}, \bm{a}^*, \bm{o}')T(\bm{s}_{new}^{(depth-1)},\\&\bm{a}^*, \bm{s}_{new}^{(depth)})V^*(\bm{b}(\bm{s}_{new}^{(depth)}: \bm{s}^{(depth)}\rightarrow \bm{s}_{add}^{(depth)}))\|_1\\
&\leq   R_{max} \| \sum_{\bm{s}^{(depth-1)}_{new}\in S}\bm{b}(\bm{s}^{(depth-1)}_{new}: \bm{s}^{(depth-1)}\rightleftarrows \\
&\bm{s}_{add}^{(depth-1)})-  \sum_{\bm{s}^{(depth-1)}_{new}\in S}\bm{b}(\bm{s}^{(depth-1)}_{new}: \bm{s}^{(depth-1)}\\&\rightarrow \bm{s}_{add}^{(depth-1)})\|_1+\gamma\sum_{o’\in O}\sum_{\bm{s}_{new}^{(depth)}\in S}Z(\bm{s}_{new}^{(depth)}, \bm{a}^*,\\
& \bm{o}') T(\bm{s}_{new}^{(depth-1)},\bm{a}^*, \bm{s}_{new}^{(depth)})\|V^*(\bm{b}(\bm{s}_{new}^{(depth)}: \bm{s}^{(depth)}\\& \rightleftarrows\bm{s}_{add}^{(depth)}))-  V^*(\bm{b}(\bm{s}_{new}^{(depth)}: \bm{s}^{(depth)}\rightarrow \\&\bm{s}_{add}^{(depth)}))\|_1\\
&\leq   R_{max}\delta+ \gamma\sum_{o’\in O}\sum_{\bm{s}_{new}^{(depth)}\in S}Z(\bm{s}_{new}^{(depth)}, \bm{a}^*, \bm{o}')\\&T(\bm{s}_{new}^{(depth-1)},~\bm{a}^*, \bm{s}_{new}^{(depth)})\frac{ R_{max}}{1-\gamma} \delta\\
&\leq R_{max}\delta + \gamma \frac{ R_{max}}{1-\gamma} \delta\\
&=\frac{ R_{max}}{1-\gamma} \delta.
\end{aligned}
\end{equation}
This formulation is the same as the one in~\eqref{boundxx} in format. The $\delta$ is the upper bound of the belief distances of different belief nodes. Following a similar recurrence process until the first layer, we can get:

\begin{equation}\label{EQ_28}
\begin{aligned}
&\|V^*(\bm{b}(\bm{s}_{new})) - \hat{V}^*(\bm{b}(\bm{s}_{new}))\|_1\\&=\| \sum_{\bm{s}_{new}\in S} R(\bm{s}_{new}, \bm{a}^*, \bm{s}_{new}') \bm{b}(\bm{s}_{new})+  \gamma\sum_{o‘\in O}
\\&\sum_{\bm{s}_{new}'\in S}Z(\bm{s}_{new}', \bm{a}^*, \bm{o}')T(\bm{s}_{new}, \bm{a}^*, \bm{s}_{new}')V^*(\bm{b}(\bm{s}_{new}'\\
&: \bm{s}'\rightleftarrows \bm{s}_{add}'))-\sum_{\bm{s}_{new}\in S} R(\bm{s}_{new}, \bm{a}^*, \bm{s}_{new}') \bm{b}(\bm{s}_{new}) \\
&- \gamma\sum_{o’\in O}\sum_{\bm{s}_{new}'\in S}Z(\bm{s}_{new}', \bm{a}^*, \bm{o}')T(\bm{s}_{new}, \bm{a}^*, \bm{s}_{new}')V^*\\
&(\bm{b}(\bm{s}_{new}': \bm{s}'\rightarrow \bm{s}_{add}'))\|_1\\
&=\|  \gamma\sum_{o‘\in O}\sum_{\bm{s}_{new}'\in S}Z(\bm{s}_{new}', \bm{a}^*, \bm{o}')T(\bm{s}_{new}, \bm{a}^*, \bm{s}_{new}')\\&V^*(\bm{b}(\bm{s}_{new}': \bm{s}'\rightleftarrows \bm{s}_{add}'))- \gamma\sum_{o’\in O}\sum_{\bm{s}_{new}'\in S}Z(\bm{s}_{new}', \bm{a}^*\\
&, \bm{o}')T(\bm{s}_{new}, \bm{a}^*, \bm{s}_{new}')V^*(\bm{b}(\bm{s}_{new}': \bm{s}'\rightarrow \bm{s}_{add}'))\|_1
\end{aligned}
\end{equation}

\begin{equation}\nonumber
\begin{aligned}
&\leq  \gamma\sum_{o‘\in O}\sum_{\bm{s}_{new}'\in S}Z(\bm{s}_{new}', \bm{a}^*, \bm{o}')T(\bm{s}_{new}, \bm{a}^*, \bm{s}_{new}')\\
&\frac{R_{max}}{1-\gamma} \delta\\
&\leq\frac{\gamma R_{max}}{1-\gamma} \delta.
\end{aligned}
\end{equation}
The proof is completed.

\section{Proof for Corollary~~\ref{t_add_new3z}}\label{a5}

\begin{proof}
Based on Corollary~\ref{t_add_new2s}, we have:
\begin{equation}\label{EQ_281}
\begin{aligned}
&\hat{V}(\bm{b}, \bm{a}_{sub}^*)\leq {V}(\bm{b}, \bm{a}_{sub}^*) + \frac{\gamma R_{max}}{1-\gamma} \delta
\end{aligned}
\end{equation}

and 

\begin{equation}\label{EQ_282}
\begin{aligned}
&{V}(\bm{b}, \bm{a}^*)-\frac{\gamma R_{max}}{1-\gamma} \delta \leq \hat{V}(\bm{b}, \bm{a}^*).
\end{aligned}
\end{equation}

The solution does not change, which means that the approximated value function using sub-optimal action $\bm{a}_{sub}$ is smaller than the one with optimal action, satisfying  $\hat{V}(\bm{b}, \bm{a}^*)> \hat{V}(\bm{b}, \bm{a}_{sub}^*)$. Here, considering Eq.~\eqref{EQ_281} and Eq.~\eqref{EQ_282}, if the upper bound of $\hat{V}(\bm{b}, \bm{a}_{sub}^*)$ is smaller than the lower bound of $\hat{V}(\bm{b}, \bm{a}^*)$, satisfying:

\begin{equation}\label{EQ_2821}
\begin{aligned}
&{V}(\bm{b}, \bm{a}^*) - {V}(\bm{b}, \bm{a}_{sub}^*) > \frac{2\gamma R_{max}}{1-\gamma} \delta \\
&\Rightarrow {V}(\bm{b}, \bm{a}_{sub}^*) + \frac{\gamma R_{max}}{1-\gamma} \delta < {V}(\bm{b}, \bm{a}^*)-\frac{\gamma R_{max}}{1-\gamma} \delta.
\end{aligned}
\end{equation}
We have:
\begin{equation}\label{EQ_2811}
\begin{aligned}
&\hat{V}(\bm{b}, \bm{a}_{sub}^*)< \hat{V}(\bm{b}, \bm{a}^*).
\end{aligned}
\end{equation}

The proof is completed.
\end{proof}

\section{Parameter selection}
Our solver relies on 8 parameters, including the refining clustering number $k$, coefficients $\omega_1$, $\omega_2$, and $\omega_3$ for the MCTS action selection strategy, a self-defined exploration constant $C_r$, the minimum radius for partitioning $D_{lim}$, and two coefficients controlling the refining velocity, with $0 \leq \bar{\omega}_2 < \bar{\omega}_1 \leq 1$. Most of these parameters were not fine-tuned for optimal performance; instead, they were quickly identified or chosen intuitively. Here, we would like to add some explanations to help users quickly determine the parameters within a short time, like 20 minutes. Let’s go through each parameter: 

\textbf{Refining Clustering Number} $k$: 

Fig.~\ref{fig:side11} shows that $k=3, 4$, and 5 result in similar performance, indicating robustness and flexibility in selecting this parameter. For general POMDP problems, I recommend using the default value $k=3$ (chosen arbitrarily before parameter experiments) or the optimal clustering value, obtained during the list initialization update before MCTS, which $k_i=4$ in our paper. If the value $k_i=4$ obtained in the list initial update before MCTS is used, this parameter $k$ becomes non-heuristic and consistently achieves good performance.

\textbf{Coefficients} $\omega_1$, $\omega_2$, and $\omega_3$: 

These coefficients are straightforward to select. Without specific domain knowledge, I use the common value $\omega_1=\sqrt{2}$, derived theoretically from the multi-armed bandit problem based on Hoeffding's inequality. For $\omega_2$ and $\omega_3$, ensure that $\omega_2range_i$ and $\omega_3 BD$ are comparable to another two terms: $\hat Q (\bm{node}_o, \bm{a})$ and $\omega_1 \sqrt{\frac{\ln N (\bm{node}_o)}{N (\bm{node}_o,~\bm{a})}}$. Because adding or removing the $BD$ term does not make too large differences and the final result is similar, $\omega_3$ can be set smaller than $\omega_2$.

\textbf{Self-Defined Exploration Constant} $C_r$:  

To determine $C_r$, users can follow this simple process: 

\begin{itemize}
     \item Run the POMDP problem for one step within the time limit and identify the number of their commonly used particles, like $N_p=150$. 

     \item Estimate the mean partitioning radius based on the problem setting, like $mean(range^*)$ is about 0.5. 

     \item The selection of $C_r$ is to make sure that $1/(C_rmean(range^*)^2)$ is about 30\%-50\% of the particle number $N_p$.  This ensures partitioning refines the continuous action domain at least 3-4 times. 

     \item  We have the selection $C_r$ is set as $1/((0.3~\text{to}~0.5)N_pmean(range^*)^2)$. 
\end{itemize}

\textbf{Coefficients for Refining Velocity} $\bar{\omega}_2$ and $\bar{\omega}_1$:

The coefficients $\bar{\omega}_2 =0.3$ and $\bar{\omega}_1=0.6$ were chosen arbitrarily, without extensive consideration. Other similar settings should also work well. 

\textbf{Minimum Radius for Partitioning} $D_{lim}$:

In this paper, $D_{lim}$ is set to 0.2 without significant adjustments. This value covers a small region in $\mathbb{R}^9$, including 6D base pose, $p$, lift height $l_h$, and the pan and tilt angles $l_p$ and $l_t$ of the robot’s head, as demonstrated with the Fetch robot. For example, two configurations with a distance of 0.2 might differ by 0.05 meters in $x$, $y$ axis, 0.1 radians (about 5.7 degrees) in orientation, 0.05 meters in $l_h$, 0. 1 radians in $l_p$, 0. 1 radians in $l_t$. The resulting Euclidean distance is $\sqrt{(0.05)^2\cdot2+(0.1)^2\cdot3+ (0.05)^2}= 0.194$, which are within this small range. For other POMDP problems, users can adjust $D_{lim}$ to ensure actions within this range have similar physical meanings with acceptable differences.  

\section{LLM prompts example for Sayplan}\label{appendix_g}
The high-level LLM prompts example for Sayplan is shown in this section.
\begin{figure*}[!ht]
\centering
\begin{tcolorbox}[
  enhanced,
  sharp corners,
  colback=white,
  colframe=gray!50,
  drop shadow southeast,
  fontupper=\ttfamily\small,
  boxrule=0.5pt,
  arc=2mm,
  left=5mm,
  right=5mm,
  top=3mm,
  bottom=3mm,
  title={\textbf{LLM prompts example for Sayplan}},
  coltitle=black
]

\textbf{Content:} You are a Fetch robot in an unexplored house. Your task is to find and remove a blue snack box, its corresponding property is labelled as target$\_$object, based on your head camera. 

You have the coordinates of the four vertices of each workspace, and all objects are limited to these known workspaces:

Workspace 1: (5.7, 0.61), (5.66, 1.29), (7.5, 1.29), (7.5, 0.61).

You know the move ranges of your pan (radian), tilt (radian), and lift (meter) joints are: (-0.261666667, 0.261666667), (0.0, 0.5), (0.0, 0.4)

\vspace{1em}
\textbf{You have the following actions available that you can use to achieve this task:}
\begin{enumerate}\itemsep 0pt
  \item \texttt{navigate(room\_name, object\_name)}: navigate to a configuration (pose and joint angles) that can observe this object in this room.
  \item \texttt{claim\_to\_obstacle(room\_name, object\_name)}: claim this object is the obstacle object, when the object has a high possibility to be the obstacle object.
  \item \texttt{claim\_to\_target(room\_name, object\_name)}: claim this object is the target object, when the object has a high possibility to be the obstacle object.
  \item \texttt{remove(room\_name, object\_name)}: remove this target or obstacle object to complete task or free some field of view for better exploration.
  \item \texttt{explore(room\_name)}: randomly select a frontier to explore the unknown space near one of the rooms that is not fully explored yet.
  \item \texttt{done()}: call when the task (target object is removed) is completed or if you are unable to take any further actions.
\end{enumerate}
\vspace{0.5em}
\textbf{Output Response Format:}

\textbf{Analysis:} describe where you could find the objects of interest and what actions you need to execute to get there.

\textbf{Reasoning:} justify why the next action is important to solve the task.

\textbf{Command:} function call', 'role': 'system', \{'content': 'Scene Graph: \{nodes: [\{pos: (8.383142272116066, 0.3976500981658697, 0), pos\_map: (131.756, 365.228), room\_id: 0, node\_type: <NODETYPE.ROOM: 2>, frontier\_points: set(), id: room-0\}, \{bbox: array([0.32000007, 0.26699999, 0.24999995]), semantic\_class\_name: obstacle\_object, pos: (7.150000083170853, 0.9799999103916819), orn: -2.7124770079457416, name: object\_1, states: not\_move\_able, viewpoint: array([6.802738406863281, -0.4402219209680013, 0.0, 0.0, 0.0,
       0.5381722605035929, 0.8425332814648375, 0.0, 0.0, 0.0],
      dtype=object), room\_id: 0, pos\_map: (143, 340, 0), closest\_vor\_node: (134.0, 330.0), node\_type: <NODETYPE.OBJECT: 3>, pos\_in\_world: (7.12174225331974, 0.9598501065433025), closest\_vor\_node\_in\_world: (6.621742245869159, 0.50985009983778), id: object\_1\}, \{bbox: array([0.26000006, 0.14, 0.31999994]), semantic\_class\_name: updating\_object(reach\_obstacle\_index), pos: (6.550000178538284, 1.0499998436344797), orn: 3.1415926469292366, name: object\_2, states: not\_move\_able, viewpoint: array([6.811979331502813, -0.4331814501971404, 0.0, 0.0, 0.0,
       -0.2516452028034676, 0.9676042215118779, 0.0, 0.0, 0.0],
      dtype=object), room\_id: 0, pos\_map: (144, 328, 0), closest\_vor\_node: (133.5, 339.5), node\_type: <NODETYPE.OBJECT: 3>, pos\_in\_world: (6.521742244379043, 1.0098501072883606), closest\_vor\_node\_in\_world: (7.096742252947211, 0.4848500994652509), id: object\_2\}, \{bbox: array([0.09000008, 0.12, 0.11999994]), semantic\_class\_name: updating\_object(reach\_obstacle\_index), pos: (5.900000083170853, 0.8599999056233103), orn: -0.48065383129846045, name: object\_3, states: move\_able, viewpoint: array([4.561832653037297, 0.8680199317907509, 0.0, 0.0, 0.0,
       -0.9903361548528572, 0.13868778025208775, 0.0, 0.0, 0.0],
      dtype=object), room\_id: 0, pos\_map: (141, 315, 0), closest\_vor\_node: (141.0, 310.0), node\_type: <NODETYPE.OBJECT: 3>, pos\_in\_world: (5.871742234693288, 0.8598501050531864), closest\_vor\_node\_in\_world: (5.621742230967998, 0.8598501050531864), id: object\_3\}, \{bbox: array([0.13000007, 0.07, 0.14999995]), semantic\_class\_name: target\_object, pos: (7.0999998924359895, 0.6499998674763376), orn: -3.0692694745833773, name: object\_4, states: move\_able, viewpoint: array([6.802738406863281, -0.4402219209680013, 0.0, 0.0, 0.0,
       0.5381722605035929, 0.8425332814648375, 0.0, 0.0, 0.0],
      dtype=object), room\_id: 0, pos\_map: (136, 339, 0), closest\_vor\_node: (133.5, 339.5), node\_type: <NODETYPE.OBJECT: 3>, pos\_in\_world: (7.071742252574682, 0.6098501013278961), closest\_vor\_node\_in\_world: (7.096742252947211, 0.4848500994652509), id: object\_4\}, \{location: room-0, id: robot\}], links: [room-0 - object\_2, room-0 - object\_4, room-0 - object\_1, room-0 - object\_3, room-0 - robot]\},
\end{tcolorbox}
\end{figure*}

\begin{figure*}[t]
\centering
\begin{tcolorbox}[
  enhanced,
  sharp corners,
  colback=white,
  colframe=gray!50,
  drop shadow southeast,
  fontupper=\ttfamily\small,
  boxrule=0.5pt,
  arc=2mm,
  left=5mm,
  right=5mm,
  top=3mm,
  bottom=3mm,
]
\textbf{History:}
 explore(room-0) - failure, claim\_to\_target(room-0, object\_4) - success, navigate(room-0, object\_4) - success, claim\_to\_target(room-0, object\_4) - failure, explore(room-0) - failure.\vspace{0.5em}

What is the best next action to complete the task as efficiently as possible? I you don't think that the object can be found in a known room, prioritize opening doors over exploring a room.

\textbf{Remember:}
\begin{enumerate}\itemsep 0pt
  \item Respond with a function call
  \item You can only use the objects and rooms that you have already found. Object names have to match the description exactly.
  \item You can only explore the found rooms that contain unexplored areas. The explore action allows the robot to adjust its configuration, change its field of view, and make an initial observation of the object. This first observation is essential, as it provides the basis for updating the object and reaching its index.
  \item To observe objects from different directions first to get the belief update until reach obstacle or target index. If the updating object reaches the obstacle or target index, call the claim action first. If it is claimed as a target object (greatly welcome), directly call the remove action and then be done. Task completed and you do not need to navigate to it.
  \item If an object is likely to be the target object, it is good to navigate to this object and observe it from different directions for later claim action.
  \item The claim of both the target and obstacle objects is encouraged. After an object is claimed as the obstacle object, it is legal to remove it to free some field of view. However, the removal action often fails in real-world settings, so obstacle removal should be used cautiously and only attempted occasionally.
  \item If some actions are repeated several times, they may not be possible and you need to observe that object from more directions using explore action.
  \item Each object needs to be claimed as the target or obstacle object first and then it is legal to remove it. Directly remove is illegal.
  \item You can’t explore rooms that haven’t been found yet.
\end{enumerate}
\end{tcolorbox}
\end{figure*}

\begin{figure*}[t]
\centering
\begin{tcolorbox}[
  enhanced,
  sharp corners,
  colback=white,
  colframe=gray!50,
  drop shadow southeast,
  fontupper=\ttfamily\small,
  boxrule=0.5pt,
  arc=2mm,
  left=5mm,
  right=5mm,
  top=3mm,
  bottom=3mm,
  title={\textbf{LLM prompts example for low-level navigation action}},
  coltitle=black
]

\{\textbf{Content:} This is the response result in this step:

\vspace{0.5em}

**Analysis:** The target object "object\_3" has been successfully identified in room-0. The next best action would be to navigate to the updating object "object\_4" and observe it from different directions to update its belief until reaching the obstacle index.

\vspace{0.5em}

**Reasoning:** By navigating to the updating object "object\_4" and observing it from different directions, we can update its belief until reaching the obstacle index. This will allow us to determine if "object\_4" is the obstacle object that needs to be removed to complete the task efficiently.

\vspace{0.5em}

**Command:** navigate(room-0, object\_4)
\vspace{0.5em}
The object coordinate is (5.900000083170853, 0.8599999056233103).
The last robot configuration (pose, pan joint, tilt joint, lift joint) viewed this object is [4.991183852927918 1.2915176222424134] (coordinate), -0.3063040772599799 (orientation), 0.0 (pan angle: radian), 0.0 (tilt angle: radian), 0.0 (lift height: meter).
You know the move ranges of your pan (radian), tilt (radian), and lift (meter) joints are: (-0.261666667, 0.261666667), (0.0, 0.5), (0.0, 0.4)
', 'role': 'system'\}

\end{tcolorbox}
\end{figure*}

\begin{figure*}[t]
\centering
\begin{tcolorbox}[
  enhanced,
  sharp corners,
  colback=white,
  colframe=gray!50,
  drop shadow southeast,
  fontupper=\ttfamily\small,
  boxrule=0.5pt,
  arc=2mm,
  left=5mm,
  right=5mm,
  top=3mm,
  bottom=3mm,
  coltitle=black
]

\{\textbf{'Content':} 'Can you pick out a position from following coordinates and determine the orientation, pan joint, tilt joint, lift joint within given ranges for next pose to help robot to view the other side of the object?
(4.555492215079635, 0.34620009739920476), (5.6804422318426955, 2.4674501290082924), (4.621742216066837, 0.8598501050531864), (6.34674224177134, 2.4848501292675733), (4.999192221691279, 1.3299501120582216), (6.046742237300991, 2.507350129602849), (7.096742252947211, -0.01514990798532967), (7.221742254809856, -0.18179991046860788), (6.59674224549663, 2.384850127777457), (7.302942256019831, 2.459550128890574), (6.59674224549663, 2.449150128735602), (5.146742223889946, 0.584850100955367), (5.996742236555933, 2.5098501296401023), (8.39674227231872, 0.7348501031905412), (8.39674227231872, 1.0348501076608896), (8.39674227231872, 1.334850112131238), (4.830942219184161, 1.0848501084059476), (7.696742261887907, 2.37805012767613), (7.396742257417559, 2.453050128793717), (6.9967422514570945, 2.4848501292675733), (6.696742246986746, 2.4705501290544865), 

\vspace{1em}
\textbf{Remember:}
\begin{enumerate}\itemsep 0pt
  \item Respond with a function call.
  \item Make sure that the output vector only includes 6 numbers corresponding to the position (x axis), the position (y axis), orientation, pan joint, tilt joint, lift joint and divided by ;.
\end{enumerate}\}

\end{tcolorbox}
\end{figure*}

\section{LLM prompts example for MoMa-LLM}\label{appendix_g}
The high-level LLM prompts example for MoMa-LLM is shown in this section.
\begin{figure*}[t]
\centering
\begin{tcolorbox}[
  enhanced,
  sharp corners,
  colback=white,
  colframe=gray!50,
  drop shadow southeast,
  fontupper=\ttfamily\small,
  boxrule=0.5pt,
  arc=2mm,
  left=5mm,
  right=5mm,
  top=3mm,
  bottom=3mm,
  title={\textbf{LLM prompts example for MoMa-LLM}},
  coltitle=black
]

\textbf{Content:} You are a Fetch robot in an unexplored house. Your task is to find and remove a blue snack box, its corresponding property is labelled as target\_object, based on your head camera.

You have the coordinates of the four vertices of each workspace and all objects are limited in these known workspaces:

Workspace 1: (5.7, 0.61), (5.66, 1.29), (7.5, 1.29), (7.5, 0.61).

You know the move ranges of your pan (radian), tilt (radian), and lift (meter) joints are: (-0.261666667, 0.261666667), (0.0, 0.5), (0.0, 0.4)

\vspace{1em}
\textbf{You have the following actions available that you can use to achieve this task:}
\begin{enumerate}\itemsep 0pt
  \item \texttt{navigate(room\_name, object\_name)}: navigate to a configuration (pose and joint angles) that can observe this object in this room.
  \item \texttt{claim\_to\_obstacle(room\_name, object\_name)}: claim this object is the obstacle object, when the object has a high possibility to be the obstacle object.
  \item \texttt{claim\_to\_target(room\_name, object\_name)}: claim this object is the target object, when the object has a high possibility to be the obstacle object.
  \item \texttt{remove(room\_name, object\_name)}: remove this target or obstacle object to complete task or free some field of view for better exploration.
  \item \texttt{explore(room\_name)}: randomly select a frontier to explore the unknown space near one of the rooms that is not fully explored yet.
  \item \texttt{done()}: call when the task (target object is removed) is completed or if you are unable to take any further actions.
\end{enumerate}
\vspace{0.5em}
\textbf{Output Response Format:}

\textbf{Analysis:} describe where you could find the objects of interest and what actions you need to execute to get there.

\textbf{Reasoning:} justify why the next action is important to solve the task.
\end{tcolorbox}
\end{figure*}

\begin{figure*}[t]
\centering
\begin{tcolorbox}[
  enhanced,
  sharp corners,
  colback=white,
  colframe=gray!50,
  drop shadow southeast,
  fontupper=\ttfamily\small,
  boxrule=0.5pt,
  arc=2mm,
  left=5mm,
  right=5mm,
  top=3mm,
  bottom=3mm,
]
\textbf{Command:} function call

'role': 'user', \{'content': 'You are currently in the room-0. You are standing next to the following objects: [object\_2, object\_3, object\_4, object\_5, object\_6, object\_7]. Furthermore, you have found the following rooms, objects, and the corresponding properties of the objects in the house so far:

- The names of the objects in room-0: [object\_1, object\_2, object\_3, object\_4, object\_5, object\_6, object\_7].

- The corresponding properties of the objects in room-0: [obstacle\_object, obstacle\_object, obstacle\_object, obstacle\_object, updating\_object(reach\_obstacle\_index), obstacle\_object, updating\_object(reach\_target\_index)].

The only found rooms are [room-0] No other rooms are found.

Your 5 previous actions were: navigate(room-0, object\_5) - success, navigate(room-0, object\_7) - success, claim\_to\_target(room-0, object\_7) - failure, explore(room-0) - failure, navigate(room-0, object\_7) - success.

What is the best next action to complete the task as efficiently as possible? 
\textbf{Remember:}
\begin{enumerate}\itemsep 0pt
  \item Respond with a function call.

\item  You can only use the objects and rooms that you have already found. Object names have to match the description exactly.

\item You can only explore the found rooms that contain unexplored areas. The explore action allows the robot to adjust its configuration, change its field of view, and make an initial observation of the object. This first observation is essential, as it provides the basis for updating the object and reaching its index.

\item To observe objects from different directions first to get the belief update until reach obstacle or target index. If the updating object reaches the obstacle or target index, call the claim action first. If it is claimed as a target object (greatly welcome), directly call remove action and then done. Task completed and You do not need to navigate to it.

\item If an object is likely to be target object, it is good to navigate to this object and observe it from different directions for later claim action.

\item The claim of both the target and obstacle objects is encouraged. After an object is claimed as the obstacle object, it is legal to remove it to free some field of view. However, the removal action often fails in real-world settings, so obstacle removal should be used cautiously and only attempted occasionally.

\item If some actions are repeated several times, they may not be possible and you need to observe that object from more directions using the explore action.

\item Each object needs to be claimed as the target or obstacle object first and then it is legal to remove it. Directly removing is illegal.

\item You can’t explore rooms that haven’t been found yet.'\},
\end{enumerate}
\end{tcolorbox}
\end{figure*}

The high-level LLM prompts example for MoMa-LLM with more geometry properties is shown in this part.
\begin{figure*}[t]
\centering
\begin{tcolorbox}[
  enhanced,
  sharp corners,
  colback=white,
  colframe=gray!50,
  drop shadow southeast,
  fontupper=\ttfamily\small,
  boxrule=0.5pt,
  arc=2mm,
  left=5mm,
  right=5mm,
  top=3mm,
  bottom=3mm,
  title={\textbf{LLM prompts example for MoMa-LLM with more geometry properties}},
  coltitle=black
]
\textbf{Content:} You are a Fetch robot in an unexplored house. Your task is to find and remove a blue snack box, its corresponding property is labeled as target\_object based on your head camera.

You have the coordinates of the four vertices of each workspace and all objects are limited to these known workspaces:

Workspace 1: (5.7, 0.61), ((5.66, 1.29)), ((7.5, 1.29)), ((7.5, 0.61)).

You know the move ranges of your pan (radian), tilt (radian), and lift (meter) joints are: (-0.261666667, 0.261666667), (0.0, 0.5), (0.0, 0.4)

\textbf{You have the following actions available that you can use to achieve this task:}
\end{tcolorbox}
\end{figure*}

\begin{figure*}[t]
\centering
\begin{tcolorbox}[
  enhanced,
  sharp corners,
  colback=white,
  colframe=gray!50,
  drop shadow southeast,
  fontupper=\ttfamily\small,
  boxrule=0.5pt,
  arc=2mm,
  left=5mm,
  right=5mm,
  top=3mm,
  bottom=3mm,
]
\vspace{1em}
\begin{enumerate}\itemsep 0pt
  \item \texttt{navigate(room\_name, object\_name)}: navigate to a configuration (pose and joint angles) that can observe this object in this room.
  \item \texttt{claim\_to\_obstacle(room\_name, object\_name)}: claim this object is the obstacle object, when the object has a high possibility to be the obstacle object.
  \item \texttt{claim\_to\_target(room\_name, object\_name)}: claim this object is the target object, when the object has a high possibility to be the obstacle object.
  \item \texttt{remove(room\_name, object\_name)}: remove this target or obstacle object to complete task or free some field of view for better exploration.
  \item \texttt{explore(room\_name)}: randomly select a frontier to explore the unknown space near one of the rooms that is not fully explored yet.
  \item \texttt{done()}: call when the task (target object is removed) is completed or if you are unable to take any further actions.
\end{enumerate}
\vspace{0.5em}

\textbf{Output Response Format:}

\textbf{Analysis:} Describe where you could find the objects of interest and what actions you need to execute to get there.

\textbf{Reasoning:} justify why the next action is important for solving the task.

\textbf{Content:} 
'You are currently in the room-0. You are standing next to the following objects: [object\_2]. Furthermore, you have found the following rooms, objects, and the corresponding properties of the objects in the house so far:

- The names of the objects in room-0: [object\_2].

- The corresponding properties of the objects in room-0: [updating\_object

(likely\_to\_be\_obstacle\_object)].

- The corresponding poses of the objects (coordinate and quaternion, 7 elements each object) in room-0: [(5.900000083170853, 0.8599999056233103, 0.8800000636606228, 0.0, 0.0, -0.2380201103608644, 0.9712600145230841)].

- The corresponding sizes of the objects (longest horizontal axis is x axis, shorter horizontal axis is y axis, and vertical one is z axis, 3 elements each object) in room-0: [(0.09000007805983846, 0.12000000067840325, 0.11999994465407036)].

The only found rooms are [room-0] No other rooms are found.

Your 5 previous actions were: explore(room-0) - success, claim\_to\_obstacle(room-0, object\_2) - failure, explore(room-0) - success, explore(room-0) - success, explore(room-0) - success.\\

What is the best next action to complete the task as efficiently as possible? 

\vspace{0.5em}
\textbf{Remember:}
\begin{enumerate}\itemsep 0pt
  \item Respond with a function call.

\item  You can only use the objects and rooms that you have already found. Object names have to match the description exactly.

\item You can only explore the found rooms that contain unexplored areas. The explore action allows the robot to adjust its configuration, change its field of view, and make an initial observation of the object. This first observation is essential, as it provides the basis for updating the object and reaching its index.

\item To observe objects from different directions first to get the belief update until reach obstacle or target index. If the updating object reaches the obstacle or target index, call the claim action first. If it is claimed as a target object (greatly welcome), directly call remove action and then done. Task completed and You do not need to navigate to it.

\item If an object is likely to be target object, it is good to navigate to this object and observe it from different directions for later claim action.

\item The claim of both the target and obstacle objects is encouraged. After an object is claimed as the obstacle object, it is legal to remove it to free some field of view. However, the removal action often fails in real-world settings, so obstacle removal should be used cautiously and only attempted occasionally.

\item If some actions are repeated several times, they may not be possible and you need to observe that object from more directions using explore action.

\item Each object needs to be claimed as the target or obstacle object first and then it is legal to remove it. Directly remove is illegal.

\item You can’t explore rooms that haven’t been found yet.',
\end{enumerate}
\end{tcolorbox}
\end{figure*}

\begin{acks}
This research is supported by the ANU Futures Scheme. 
\end{acks}


\begin{thebibliography}{99}
\bibitem[(Auer et al., 2002)]{auer2002finite}
 Auer P, Cesa-Bianchi N, and Fischer P (2002) Finite-time analysis of the multiarmed bandit problem. \textit{Machine learning}. 47: 235–256.

\bibitem[(Aydemir et al., 2013)]{aydemir2013active}
Aydemir A, Pronobis A, Göbelbecker M, and Jensfelt P (2013) Active visual object search in unknown environments using uncertain semantics. \textit{IEEE Transactions on Robotics}, 29(4): 986-1002.

\bibitem[(Burks et al., 2023)]{burks2023harps}
Burks L, Ray HM, McGinley J, Vunnam S, and Ahmed N (2023) HARPS: An Online POMDP Framework for Human-Assisted Robotic Planning and Sensing. \textit{IEEE Transactions on Robotics}, 39(4): 3024-3042.

\bibitem[(Chen et al., 2020)]{Yongbo2020Active}
Chen Y, Huang S, and Fitch R. (2020) Active SLAM for mobile robots with area coverage and obstacle avoidance. \textit{IEEE/ASME Transactions on Mechatronics}, 25(3): 1182-1192.

\bibitem[(Chen et al., 2022)]{chen2022meta}
Chen R, Gao N, Vien NA, Ziesche H, and Neumann G (2022) Meta-learning regrasping strategies for physical-agnostic objects. Scaling robot learning workshop in ICRA, Philadelphia, PA, USA: IEEE.

\bibitem[(Chen et al., 2024a)]{chen2024pomdp}
Chen Y and Kurniawati H (2024a) POMDP planning for object search in partially unknown environment. In Advances in Neural Information Processing Systems, New Orleans, Louisiana, USA, 36.


\bibitem[(Chen et al., 2024b)]{anonymous2024npfkct}
Chen Y and Kurniawati H (2024b) NPF-$k$CT: A $k$-center clustering solver with neural process filter for continuous {POMDP}-based object search. \url{https://openreview.net/forum?id=WYsNjw02DE}.

\bibitem[(Chitta et al., 2017)]{Chitta2017ros}
Chitta S, Marder-Eppstein E, Meeussen W, Pradeep V, Tsouroukdissian AR, Bohren J, Coleman D, Magyar B, Raiola G, Lüdtke M, and Perdomo EF (2017) ros\_control: A generic and simple control framework for ROS. \textit{Journal of Open Source Software}, 2(20): 456-456.

\bibitem[(Danielczuk et al., 2003)]{R1}
Danielczuk M, Kurenkov A, Balakrishna A, Matl M, Wang D, Martín-Martín R, Garg A, Savarese S, and Goldberg K. (2019) Mechanical search: Multi-step retrieval of a target object occluded by clutter. In 2019 International Conference on Robotics and Automation (ICRA). Montreal, Canada: IEEE, pp. 1614-1621.

\bibitem[(Garrett et al., 2020)]{garrett2020online}
Garrett CR, Paxton C, Lozano-Pérez T, Kaelbling LP, and Fox D (2020). Online replanning in belief space for partially observable task and motion problems. In 2020 IEEE International Conference on Robotics and Automation (ICRA), Paris, France: IEEE, pp. 5678-5684.

\bibitem[(Garnelo et al., 2018)]{Garnelo2018Neural}
Garnelo M, Schwarz J, Rosenbaum D, Viola F, Rezende DJ, Eslami SM, and Teh YW (2018) Neural processes. arXiv preprint arXiv:1807.01622.


\bibitem[(Ge et al., 2024)]{ge2024commonsense}
Ge W, Tang C, and Zhang H (2024) Commonsense scene graph-based target localization for object search. arXiv preprint arXiv:2404.00343.

\bibitem[(Giuliari et al., 2023)]{giuliari2023leveraging}
Giuliari F, Skenderi G, Cristani M, Del Bue A, and Wang Y. (2023) Leveraging commonsense for object localisation in partial scenes. \textit{IEEE Transactions on Pattern Analysis and Machine Intelligence}, 45(10): 12038-12049.


\bibitem[(Hausknecht et al., 2015)]{hausknecht2015deep}
Hausknecht M and Stone P (2015) Deep recurrent q-learning for partially observable mdps. In 2015 AAAI Conference on Artificial Intelligence (AAAI) fall symposium series, Austin, Texas, USA.

\bibitem[(He et al., 2017)]{He2016Deep}
He K, Zhang X, Ren S, and Sun J (2016) Deep residual learning for image recognition. In Proceedings of the IEEE conference on computer vision and pattern recognition (CVPR), Las Vegas, NV, USA: IEEE, pp. 770-778.

\bibitem[(Hoerger et al., 2022)]{hoerger2022adaptive}
Hoerger M, Kurniawati H, Kroese D, and Ye N. (2022) Adaptive discretization using Voronoi trees for continuous-action POMDPs. In International Workshop on the Algorithmic Foundations of Robotics, University of Maryland, College Park, USA: Springer, pp. 170-187.

\bibitem[(Honerkamp et al., 2024)]{HonerkampHonerkamp}
Honerkamp D, Büchner M, Despinoy F, Welschehold T, Valada A. Language-grounded dynamic scene graphs for interactive object search with mobile manipulation. \textit{IEEE Robotics and Automation Letters}. 2024 Aug 9.



\bibitem[(Hou et al., 2021)]{Hou_belief}
Hou M, Lin TX, Zhou H, Zhang W, Edwards CR, and Zhang F (2021) Belief space partitioning for symbolic motion planning. In 2021 IEEE International Conference on Robotics and Automation (ICRA). Montreal, Canada: IEEE, pp. 8245-8251.

\bibitem[(Huang et al., 2021)]{huang2021mechanical}
Huang H, Dominguez-Kuhne M, Satish V, Danielczuk M, Sanders K, Ichnowski J, Lee A, Angelova A, Vanhoucke V, and Goldberg K (2021) Mechanical search on shelves using lateral access x-ray. In 2021 IEEE/RSJ International Conference on Intelligent Robots and Systems (IROS). Prague, Czech Republic: IEEE, pp. 2045-2052.

\bibitem[(Huang et al., 2022)]{huang2022mechanical}
Huang H, Fu L, Danielczuk M, Kim CM, Tam Z, Ichnowski J, Angelova A, Ichter B, and Goldberg K (2022) Mechanical search on shelves with efficient stacking and destacking of objects. \textit{The International Symposium of Robotics Research}: 205-221.


\bibitem[(Huang et al., 2024)]{huang2024out}
Huang Y, Yuan J, Kim C, Pradhan P, Chen B, Fuxin L, and Hermans T (2024) Out of sight, still in mind: Reasoning and planning about unobserved objects with video tracking enabled memory models. In 2024 IEEE International Conference on Robotics and Automation (ICRA), Yokohama, Japan: IEEE, pp. 3108-3115.


\bibitem[(Jocher et al., 2020)]{Jocher2020ultralytics}
Jocher G, Stoken A, Borovec J, Changyu L, Hogan A, Diaconu L, Ingham F, Poznanski J, Fang J, Yu L, and Wang M (2020) ultralytics/yolov5: v3. 1-bug fixes and performance improvements. Zenodo.

\bibitem[(Jain et al., 2025)]{Jain2025Learning}
Jain, A., Shaw, S. and Roy, N., 2025. Learning Attentive Neural Processes for Planning with Pushing Actions. arXiv preprint arXiv:2504.17924.

\bibitem[(Kim et al., 2021)]{KimJ}
Kim J, Lee E S, Lee M., Zhang D, and Kim Y M. Sgolam: Simultaneous goal localization and mapping for multi-object goal navigation. arXiv preprint arXiv:2110.07171, 2021.

\bibitem[(Kim et al., 2019)]{Kim2019Attentive}
Kim H, Mnih A, Schwarz J, Garnelo M, Eslami A, Rosenbaum D, Vinyals O and Teh YW (2019) Attentive neural processes. arXiv preprint arXiv:1901.05761.

\bibitem[(Kocsis et al., 2006)]{Kocsis2006Bandit}
Kocsis L and Szepesvári C (2006) Bandit based monte-carlo planning. In European conference on machine learning, Berlin, Heidelberg, Germany: Springer, pp. 282-293.

\bibitem[(Kurenkov et al., 2020)]{kurenkov2020visuomotor}
Kurenkov  A, Taglic J, Kulkarni R, Dominguez-Kuhne M, Garg A, Martín-Martín R, and Savarese S (2020) Visuomotor mechanical search: Learning to retrieve target objects in clutter. In 2020 IEEE/RSJ International Conference on Intelligent Robots and Systems (IROS). Las Vegas, USA: IEEE, pp. 8408-8414.

\bibitem[(Kurniawati et al., 2009)]{Kurniawati_2008} 
Kurniawati H, Hsu D, Lee WS (2009) SARSOP: Efficient point-based POMDP planning by approximating optimally reachable belief spaces. In Robotics: Science and Systems (RSS) IV, MIT Press, pp.65-72. 


\bibitem[(Kurniawati et al., 2016)]{kurniawati2016online}
Kurniawati H and Yadav V (2016) An online POMDP solver for uncertainty planning in dynamic environment. In Robotics Research: The 16th International Symposium (ISRR). Singapore: Springer, pp. 611-629.


\bibitem[(Labbé et al., 2019)]{labbe2019rtab}
Labbé M and Michaud F (2019) RTAB‐Map as an open‐source lidar and visual simultaneous localization and mapping library for large‐scale and long‐term online operation. \textit{Journal of Field Robotics}, 36(2): 416-446.

\bibitem[(Le et al., 2018)]{Le2018Empirical}
Le TA, Kim H, Garnelo M, Rosenbaum D, Schwarz J, and Teh YW (2018) Empirical evaluation of neural process objectives. In NeurIPS workshop on Bayesian Deep Learning, Montréal, Canada.

\bibitem[(Li et al., 2022)]{Li2022Category}
Li Y, Gao N, Ziesche H, and Neumann G
(2022) Category-agnostic 6d pose estimation
with conditional neural processes. arXiv preprint
arXiv:2206.07162.


\bibitem[(Lim et al., 2021)]{Michael2021Voronoi}
Lim MH, Tomlin CJ, and Sunberg ZN (2021) Voronoi progressive widening: efficient online solvers for continuous state, action, and observation POMDPs. In 2021 60th IEEE Conference on Decision and Control (CDC), Austin, Texas, USA: IEEE, pp. 4493-4500.




\bibitem[(Manceron et al., 2022)]{Manceron2017IKPy}
Manceron P, IKPy, Zenodo, 2022, doi: https://doi.org/10.5281/zenodo.6551158.

\bibitem[(Mern et al., 2021)]{mern2021bayesian}
Mern J, Yildiz A, Sunberg Z, Mukerji T, and Kochenderfer MJ (2021) Bayesian optimized Monte Carlo planning. In Proceedings of the AAAI Conference on Artificial Intelligence (AAAI). virtual, pp. 11880-11887.



\bibitem[(OpenAI, 2023)]{OpenAI_GPT}
OpenAI. "GPT-4 technical report." arXiv preprint arXiv:2303.08774, 2023.

\bibitem[(Pajarinen et al., 2019)]{pajarinen2022pomdp}
Pajarinen J, Lundell J, and Kyrki V (2022) POMDP planning under object composition uncertainty: Application to robotic manipulation. \textit{IEEE Transactions on Robotics}, 39(1): 41-56.

\bibitem[(Papadimitriou et al., 1987)]{papadimitriou1987complexity}
Papadimitriou CH and Tsitsiklis JN (1987) The complexity of Markov decision processes. \textit{Mathematics of operations research}, 12(3): 441-450.

\bibitem[(Qi et al., 2017)]{Qi2017Pointnet}
Qi CR, Su H, Mo K, and Guibas LJ (2017) Pointnet: Deep learning on point sets for 3d classification and segmentation. In Proceedings of the IEEE Conference on Computer Vision and Pattern Recognition (CVPR), Honolulu, Hawaii, USA: IEEE, pp. 652-660.

\bibitem[(Quigley et al., 2009)]{Quigley2009ROS}
Quigley M, Conley K, Gerkey B, Faust J, Foote T, Leibs J, Berger E,
Wheeler R, and Ng A (2009) ROS navigation stack. http://wiki.ros.org/navigation. 

\bibitem[(Qureshi et al., 2020)]{R5}
Qureshi AH, Dong J, Choe A, and Yip MC (2020) Neural manipulation planning on constraint manifolds. \textit{IEEE Robotics and Automation Letters}, 5(4): 6089-6096.


\bibitem[(Rana et al., 2023)]{Rana1953Sayplan}
Rana K, Haviland J, Garg S, Abou-Chakra J, Reid I, and Suenderhauf N, Sayplan: Grounding large language models using 3d scene graphs for scalable robot task planning. arXiv preprint arXiv:2307.06135, 2023.

\bibitem[(Rusu et al., 2011)]{Rusu20113d}
Rusu RB and Cousins S (2011) 3d is here: Point cloud library (pcl). In 2011 IEEE International Conference on Robotics and Automation (ICRA), Shanghai, China: IEEE, pp. 1-4.


\bibitem[(Schmalstieg et al., 2023)]{Schmalstieg}
Schmalstieg F, Honerkamp D, Welschehold T, Valada A. Learning hierarchical interactive multi-object search for mobile manipulation. \textit{IEEE Robotics and Automation Letters}. 2023 Nov 2;8(12):8549-56.


\bibitem[(Shaban et al., 2017)]{R3}
Shaban A, Bansal S, Liu Z, Essa I, and Boots B, 2017. One-shot learning for semantic segmentation. arXiv preprint arXiv:1709.03410.

\bibitem[(Shani et al., 2013)]{Shani_2013} Shani G, Pineau J, and Kaplow R (2013) A survey of point-based POMDP solvers. \textit{Autonomous Agents and Multi-Agent Systems}, 27: 1-51.

\bibitem[(Silver et al., 2010)]{silver2010monte} Silver D and Veness J (2010) Monte-Carlo planning in large POMDPs. In Advances in Neural Information Processing Systems (Neurips). Vancouver, British Columbia, Canada, 23.

\bibitem[(Sóti et al., 2023)]{R4}
Sóti G, Hein B, and Wurll C. (2023) Gradient-based grasp pose optimization on a nerf that approximates grasp success. In International Conference on Intelligent Autonomous Systems (IAS). Switzerland: Springer Nature, pp. 303-318.



\bibitem[(Somani et al., 2013)]{somani2013despot}
Somani A, Ye N, Hsu D, and Lee WS (2013) DESPOT: Online POMDP planning with regularization. In Advances in Neural Information Processing Systems (Neurips). Lake Tahoe, Nevada, United States, 26.

\bibitem[(Sucan et al., 2012)]{Sucan2012The}
Sucan IA, Moll M, and Kavraki LE (2012) The open motion planning library. \textit{IEEE Robotics \& Automation Magazine}, 19(4): 72-82.

\bibitem[(Sunberg et al., 2018)]{sunberg2018online}
Sunberg Z and Kochenderfer M (2018) Online algorithms for POMDPs with continuous state, action, and observation spaces. In Proceedings of the International Conference on Automated Planning and Scheduling (ICAPS). Delft, The Netherlands, pp. 259-263.

\bibitem[(Tang et al., 2025)]{Tang2025Task}
Tang, W., He, X., Huang, Y., Xiao, Y., Lu, C. and Cai, P., 2025. Tru-POMDP: Task Planning Under Uncertainty via Tree of Hypotheses and Open-Ended POMDPs. arXiv preprint arXiv:2506.02860.


\bibitem[(Ten et al., 2017)]{Pas2017Grasp}
Ten P. A, Gualtieri M, Saenko K, and Platt R (2017) Grasp pose detection in point clouds. \textit{The International Journal of Robotics Research}, 36(13-14): 1455-1473.



\bibitem[(Team et al., 2024)]{Team_Octo}
Team OM, Ghosh D, Walke H, Pertsch K, Black K, Mees O, Dasari S, Hejna J, Kreiman T, Xu C, Luo J. Octo: An open-source generalist robot policy. arXiv preprint arXiv:2405.12213. 2024 May 20.

\bibitem[(Thorndike et al., 1953)]{thorndike1953belongs}
Thorndike RL (1953) Who belongs in the family? Psychometrika, 18(4): 267-276.


\bibitem[(Wandzel et al., 2019)]{wandzel2019multi}
Wandzel A, Oh Y, Fishman M, Kumar N, Wong LL. and Tellex S (2019) Multi-object search using object-oriented pomdps. In 2019 International Conference on Robotics and Automation (ICRA). Montreal, Canada: IEEE, pp. 7194-7200.

\bibitem[(Yang et al., 2020)]{yang2020deep}
Yang Y, Liang H, and Choi C (2020) A deep learning approach to grasping the invisible. \textit{IEEE Robotics and Automation Letters}, 5(2): 2232-2239.


\bibitem[(Yildirim et al., 2022)]{yildirim2022learning}
Yildirim Y and Ugur E (2022) Learning social navigation from demonstrations with conditional neural processes. Interaction Studies, 23(3): 427-468.

\bibitem[(Yupeng et al., 2025)]{Zheng_Yupeng}
Yupeng Zheng, Pengxuan Yang, Zebin Xing, Qichao Zhang, Yuhang Zheng, Yinfeng Gao, Pengfei Li, Teng Zhang, Zhongpu Xia, Peng Jia, Dongbin Zhao: World4Drive: End-to-End autonomous driving via intention-aware physical latent world model. In 2025 International Conference on Computer Vision (ICCV), Honolulu, Hawai'i, USA. 


\bibitem[(Zitkovich et al., 2023)]{Zitkovich_B}
Zitkovich B, Yu T, Xu S, Xu P, Xiao T, Xia F, Wu J, Wohlhart P, Welker S, Wahid A, and Vuong Q, Rt-2: Vision-language-action models transfer web knowledge to robotic control. In 2023 Conference on Robot Learning (CoRL), pp. 2165-2183.

\bibitem[(Zhao et al., 2024a)]{Zirui2024Large}
Zhao Z, Lee W S, and Hsu D.  (2024a) Large language models as commonsense knowledge for large-scale task planning. In Advances in Neural Information Processing Systems, New Orleans, Louisiana, USA. 36.


\bibitem[(Zhao et al., 2024b)]{Zhao2024Occupancy}
Zhao L, Wang Y, and Huang S (2024b) Occupancy-slam: Simultaneously optimizing robot poses and continuous occupancy map. arXiv preprint arXiv:2405.10743.

\bibitem[(Zheng et al., 2021)]{zheng2021multi}
Zheng K, Sung Y, Konidaris G, and Tellex S (2021) Multi-resolution POMDP planning for multi-object search in 3D. In 2021 IEEE/RSJ International Conference on Intelligent Robots and Systems (IROS). Prague, Czech Republic: IEEE, pp. 2022-2029.

\bibitem[(Zheng et al., 2022)]{Zheng2022Towards}
Zheng K, Chitnis R, Sung Y, Konidaris G, and Tellex S (2022) Towards optimal correlational object search. In International Conference on Robotics and Automation (ICRA), Philadelphia, PA, USA: IEEE, pp. 7313-7319.

\bibitem[(Zheng et al., 2023)]{Zheng_system}
Zheng K, Paul A, and Tellex S. (2023). A system for generalized 3d multi-object search. arXiv preprint arXiv:2303.03178.

\bibitem[(Zhou et al., 2023)]{Zhou_Esc}
Zhou K, Zheng K, Pryor C, Shen Y, Jin H, Getoor L, Wang XE. Esc: Exploration with soft commonsense constraints for zero-shot object navigation. In International Conference on Machine Learning 2023 Jul 3 (pp. 42829-42842). PMLR.



\end{thebibliography}
\end{document}